\RequirePackage{fix-cm}
\documentclass[smallextended]{svjour3} 
\journalname{Constructive Approximation}

    \usepackage[table,svgnames]{xcolor} 
    \usepackage[column=0]{cellspace} 
\smartqed  
\usepackage{eucal}
\usepackage{amsfonts}
\usepackage{hyperref}
\usepackage{appendix}
\usepackage{mathtools}
\usepackage{bbm}
\usepackage{booktabs}
\usepackage{caption}
\usepackage{subcaption}
\usepackage{xparse}
\usepackage{sidecap}
\usepackage{float}
\usepackage{enumerate,mdwlist}
\usepackage{epsfig}
\usepackage[bbgreekl]{mathbbol}
\usepackage{amsfonts,amsmath,enumerate}
\usepackage{amscd}
\usepackage{amsmath,mathrsfs}\usepackage{bm}
\usepackage{latexsym}
\usepackage{mleftright}
\usepackage{color}
\usepackage{xifthen}
\usepackage{makeidx}
\usepackage{amsmath}
\usepackage{amssymb}
\usepackage{mathrsfs}
\usepackage{tikz,tikz-cd}
\usepackage{tkz-euclide}
\usetikzlibrary{matrix}
\usepackage{smartdiagram}

\usepackage[english]{babel}
\usepackage[utf8]{inputenc}

\usepackage{amsmath}
\usepackage{amssymb}

    \usepackage[ruled,vlined]{algorithm2e}
    \SetKwInput{KwInput}{Input}
    \SetKwInput{KwOutput}{Output}
    \usepackage{times}
    \usepackage{eucal} 
    \usepackage{mathrsfs} 
    \usepackage{wasysym} 
    \usepackage{footnote}
        \makesavenoteenv{tabular}
        \usepackage{tikz}
        \usepackage{tikz-cd}
            \usetikzlibrary{arrows,matrix}
        \usepackage{float}
        \usepackage{subcaption}
        \usepackage{booktabs}
        \usepackage[letterpaper,left=3.75cm,top=2.54cm,bottom=2.54cm,right=3.75cm]{geometry}
\usepackage{xparse}
    \usepackage{enumerate}

\usepackage{appendix}
\usepackage[numbers]{natbib}
\usepackage{rotating}
\usepackage[graphicx]{realboxes}

    \NewDocumentCommand{\cc}{o}{{\mathbb{C}{\IfValueT{#1}{^{{#1}}}}}}
    \newcommand{\ee}{{\mathbb{E}}}
    \NewDocumentCommand{\nn}{o}{{\mathbb{N}{\IfValueT{#1}{^{{#1}}}}}}
    \NewDocumentCommand{\pp}{o}{{\mathbb{P}{\IfValueT{#1}{^{{#1}}}}}}
    \NewDocumentCommand{\rr}{o}{{\mathbb{R}{\IfValueT{#1}{^{{#1}}}}}}
            \NewDocumentCommand{\rrd}{}{{\rr[d]}}
            \NewDocumentCommand{\rrD}{}{{\rr[D]}}
            
    \NewDocumentCommand{\xx}{o}{{\mathbb{X}{\IfValueT{#1}{_{{#1}}}}}}
    \NewDocumentCommand{\zz}{o}{{\mathbb{Z}{\IfValueT{#1}{^{{#1}}}}}}
    \NewDocumentCommand{\bbb}{}{{\mathcal{B}}}
    \NewDocumentCommand{\ddd}{}{{\mathcal{D}}}
    \NewDocumentCommand{\fff}{}{{\mathcal{F}}}

    \NewDocumentCommand{\www}{o}{{\mathcal{W}{\IfValueT{#1}{_{{#1}}}}}}
    \NewDocumentCommand{\xxx}{o}{{\mathcal{X}{\IfValueT{#1}{^{{#1}}}}}}
    \NewDocumentCommand{\yyy}{o}{{\mathcal{Y}{\IfValueT{#1}{^{{#1}}}}}}

    \NewDocumentCommand{\AW}{o}{{\mathcal{AW}\IfValueT{#1}{{({#1})}}}}
    \NewDocumentCommand{\NN}{oo}{
                                \ensuremath{
                                        \mathcal{NN}\IfValueT{#1}{_{#1}}\IfValueT{#2}{^{#2}}
                                        \IfValueF{#1}{_{d:D}}\IfValueF{#2}{^{\sigma}}
                                    }
                                }
    \NewDocumentCommand{\RNN}{oo}{{\mathcal{RNN}{\IfValueT{#2}{^{#2}}}{\IfValueF{#2}{^{\sigma}}}{\IfValueT{#1}{_{#1}}}{\IfValueF{#1}{_{d,D}}}}}

\NewDocumentCommand{\co}{m}{{\operatorname{hull}({#1})}}

\newtheorem{thrminformal}{Informal Theorem}
\newtheorem{condition}{Condition}

\NewDocumentCommand{\ccc}{oo}{\ensuremath{\mathcal{C}\IfValueT{#1}{_{#1}}\IfValueT{#2}{^{#2}}\IfValueF{#1}{_{d}}\IfValueF{#2}{^{{\alpha}}}}}
\NewDocumentCommand{\Gaus}{oo}{{\mathscr{G}
                                    \IfValueF{#1}{_{2}}
                                    \IfValueT{#1}{_{#1}}
                                    \IfValueF{#2}{^{d}}
                                    \IfValueT{#2}{^{#1}}
                                }}
\NewDocumentCommand{\mmm}{o}{{
                    \mathscr{M}{\IfValueT{#1}{\left({#1}\right)}
                            }}}
\NewDocumentCommand{\ppp}{oo}{{
                \mathscr{P}{\IfValueT{#2}{_{{#2}}}}{\IfValueT{#1}{\left({#1}\right)}}
                            }}
\usepackage{adjustbox}
\newcommand{\eqdef}{\ensuremath{\stackrel{\mbox{\upshape\tiny def.}}{=}}}

\definecolor{darkgreen}{rgb}{0.0, 0.2, 0.13}
\definecolor{IllustrationGreen}{RGB}{61,126,0}
\definecolor{IllustrationRed}{RGB}{163,0,43}

\definecolor{EDITINGCOLOR_COMMENTPink}{RGB}{204,0,204}
\definecolor{EDITINGCOLOR_NewBlue}{RGB}{0,56,153}
\definecolor{EDITINGCOLOR_MovedGreen}{RGB}{0,102,51}

\newcommand{\ra}[1]{\renewcommand{\arraystretch}{#1}}
\begin{document}
\title{Universal Regular Conditional Distributions}
\subtitle{Via Probabilistic Transformers}

\author{Anastasis Kratsios}

\institute{A. Kratsios \at
              Department of Mathematics\\ 
              McMaster University\\ 
              Hamilton Hall\\
              1280 Main Street West \\
              Hamilton, Ontario, Canada\\
              Tel.: (+41) 078 933 99 19\\
              \email{kratsioa@mcmaster.ca}
}

\date{Received: July 5$^{th}$, 2021 / Revised: November $16^{th}$, 2022/ Accepted: February $23$, 2023}


\maketitle

\begin{abstract}
We introduce a deep learning model that can universally approximate regular conditional distributions (RCDs).  The proposed model operates in three phases: first, it linearizes inputs from a given metric space $\mathcal{X}$ to $\mathbb{R}^d$ via a feature map, then a deep feedforward neural network processes these linearized features, and then the network's outputs are then transformed to the $1$-Wasserstein space $\mathcal{P}_1(\mathbb{R}^D)$ via a probabilistic extension of the attention mechanism of Bahdanau et al.\ (2014).  Our model, called the \textit{probabilistic transformer (PT)}, can approximate any continuous function from $\mathbb{R}^d $ to $\mathcal{P}_1(\mathbb{R}^D)$ uniformly on compact sets, quantitatively.  We identify two ways in which the PT avoids the curse of dimensionality when approximating $\mathcal{P}_1(\mathbb{R}^D)$-valued functions.  The first strategy builds functions in $C(\mathbb{R}^d,\mathcal{P}_1(\mathbb{R}^D))$ which can be efficiently approximated by a PT, uniformly on any given compact subset of $\mathbb{R}^d$.  In the second approach, given any function $f$ in $C(\mathbb{R}^d,\mathcal{P}_1(\mathbb{R}^D))$, we build compact subsets of $\mathbb{R}^d$ whereon $f$ can be efficiently approximated by a PT.  
\end{abstract}

\keywords{Regular Conditional Distributions \and Geometric Deep Learning \and Computational Optimal Transport \and Measure-Valued Neural Networks \and Universal Approximation \and Transformers.}
\subclass{MSC (2020): 68T07 \and 28A50 \and 49Q22 \and 54C65}

\section{Introduction}
Conditioning the law of one random vector $X$ on the outcome of another random vector $Y$ is one of the central practices in probability theory with application in various areas of machine learning, ranging from learning stochastic phenomena to uncertainty quantification.  Due to the recent success of deep learning approaches to various applied science problems, one would hope that rigorous deep learning tools would exist which can approximate any regular conditional distribution (RCD).  However, such tools are currently unavailable in the theoretically-driven machine learning literature.  This paper addresses this open problem by introducing a universal deep learning model designed for constructive universal $1$-Wasserstein space-valued functions, particularly RCDs.    

We motivate this principled deep learning model with four open problems illustrating the need for universal RCDs.  Though our results are novel in the Euclidean case, we will later consider more general input and output metric spaces compatible with various machine learning tasks.  Thus, we begin by focusing on the case where the input space $\xxx=\rr^d$ and the probability measures are supported on $\yyy=\rr^D$.  We denote the $1$-Wasserstein space of probability measures with finite mean by $\ppp[\yyy][1]$ (formally defined shortly).
\begin{enumerate}[(i)]
    \item \textbf{Problem 1 - Universal Regular Conditional Distributions:} Suppose that $X$ and $Y$ are random-elements defined on a common probability space $(\Omega,\fff,\pp)$ and let $X$ and $Y$ take values in $\rr^d$ and in $\rr^D$, respectively.  
    We ask: \textit{Can we approximate the regular conditional distribution map $\xxx\ni x\mapsto \pp\left(Y\in \cdot\mid X=x\right)\in\ppp[\yyy][1]$, uniformly in $x$, with high-probability? }
\item 
    \textbf{Problem 2 - Universal Stochastic Processes:} \textit{(partner paper \cite{acciaio2022metric__PartnerPaper})}
    Assuming that Problem $1$ can be solved.  We ask: \textit{Can a recursive solution to Problem $1$ be used to approximate the conditional evolution of any suitable non-Markovian stochastic process?}
\item 
    \textbf{Problem 3 - A Generic Expression of Epistemic Uncertainty:} Given a finitely-parameterized machine learning model; i.e., 
    $\{\hat{f}_{\theta}\}_{\theta \in [-M,M]^p}\subset C(\rrd,\rrD)$, where $M\in \mathbb{N}_+$, it is common to randomize the trainable parameter $\theta$, i.e.\ replace $\theta$ with a specific $[-M,M]^p$-valued random variable $\vartheta$.  This is typically either done to reduce the model's training time (e.g.\ \cite{kingma2014adam}) or improve the model's generalizability (e.g.\ \cite{srivastava2014dropout}).  Therefore, $\{\hat{f}_{\vartheta}(x)\}_{x \in \xxx}$ is a collection of $\rrD$-valued random-vectors continuously depending on the input $x$.  We ask: \textit{is it possible to design a machine learning model which can approximately quantify the epistemic uncertainties: $\{\pp(\hat{f}_{\vartheta}(x)\in \cdot )\}_{x \in \xxx}$?}
\item 
    \textbf{Problem 4 - Universal Approximation Under Non-Convex Constraints:}  \textit{(partner paper \cite{AB_2022__PartnerPaper})}
    Let $\mathcal{Y}$ be a non-empty compact subset of $\mathbb{R}^D$ and let $f:[0,1]^d\rightarrow \mathbb{R}^D$ be a function satisfying the constraint $f([0,1]^d])\subseteq \mathcal{Y}$.  Typical examples of such functions include the solution operator to convex optimization problems \cite{SurvitCOLT2016_GeodesicConvexityOptimization} or the action of the second player in a two-player Stackelberg game
        \footnote{In a two-player Stackelberg game, the second player must optimize their utility subject to constraints imposed by the first player's utility maximization problem.} \cite{holters2018playing,jin2020local,li2021complexity}
    .  Unfortunately, classical universal approximation theorems do not guarantee that neural networks approximating $f$ satisfy the constraint $f([0,1]^d)\subseteq \mathcal{Y}$.  We ask: \textit{Can one design a neural network architecture which can approximate $f$ while exactly taking values in the constraint set $\mathcal{Y}$?}  Since the answer to this question is negative; instead, we ask: \textit{Is this possible if we allow our model's outputs to be randomized?}
\end{enumerate}
We solve these four problems with a new geometric deep learning model called the \textit{probabilistic transformer}.  This model defines a universal class of probability measure-valued maps.  Illustrated in Figure~\ref{fig_ProbalisticTransformer}, the probabilistic transformer works in three steps, an encoding step, a transformation step, and a decoding step.  In the first step, non-Euclidean data on a suitable metric space $\xxx$ is represented in a Euclidean feature space via an appropriate feature map $\varphi$.  Then, a deep feedforward network with a softmax output layer maps the Euclidean features in $\varphi(\xxx)$ to points on a high-dimensional simplex.  In the last step, the \textit{probabilistic attention mechanism} (which we introduce below) transforms the points on that high-dimensional simplex to a convex combination of a set of quantized probability measures supported on a given (non-empty) closed subspace $\yyy$ of $\mathbb{R}^D$.  

\begin{figure}[H]
\centering
\includegraphics[width=0.95\linewidth]{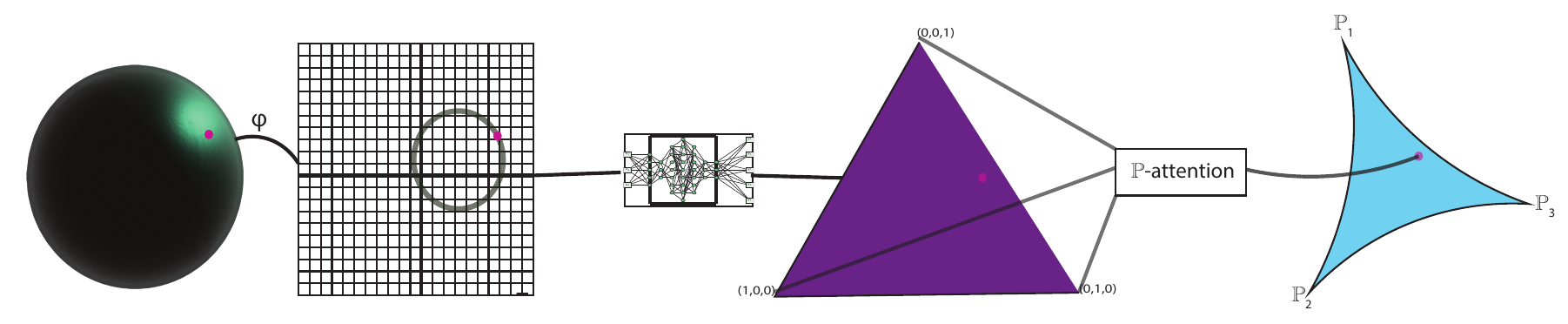}
\caption{Probabilistic Attention Mechanism: Transforming ``simplicial code'' to probability measures.}
\label{fig_ProbalisticTransformer}
\end{figure}
\paragraph{The Probabilistic Attention Mechanism}
Fix a positive integer $N$ and define the $N$-simplex $\Delta_N \eqdef \{ w\in [0,1]^N:\,\sum_{n=1}^N\, w_n=1\}$.  Given probability measures $\pp_1,\dots,\pp_N$ in the $1$-Wasserstein space on $\yyy$, we would like the map $\mathbb{P}\mbox{-}\operatorname{attention}$ in Figure~\ref{fig_ProbalisticTransformer} to produce outputs in $\mathcal{P}_1(\yyy)$ by parameterizing convex combinations of these probability measures, i.e.,\ we would like to implement the map
\begin{equation}
\label{eq:motivation_Pattention}
w\mapsto \sum_{n=1}^N \, w_n\, 
	\pp_n
,
\end{equation}
where the parameter $w$ is in $\Delta_N$.  
However, there are two difficulties with implementing~\eqref{eq:motivation_Pattention}.  First, probability measures in $\mathcal{P}_1(\yyy)$ typically cannot be \textit{exactly} parameterized using finite-dimensional vectors.  Second, the constraint $w\in \Delta_N$ can be troublesome to enforce if $w$ is allowed to vary.  

Our probabilistic attention mechanism bypasses these difficulties by directly implementing convex combinations of the quantizations of the probability measures $\pp_1,\dots,\pp_N$ in~\eqref{eq:motivation_Pattention}.  We make use of the \textit{softmax} function $\operatorname{Softmax}_N:\rr[N]\rightarrow \Delta_N$ defined by $w\mapsto \big(
\frac{e^{w_n}}{\sum_{k=1}^N e^{w_k}}
\big)_{n=1}^N$.  For any positive integers $Q$ and $N$, our probabilistic attention mechanism $\mathbb{P}\mbox{-}\operatorname{attention}_{N,Q}$ maps any $(Y,v)\in \big(\mathcal{Y}^{Q}\times \Delta_Q\big)^N$ and any $w \in \mathbb{R}^N$ to the following probability measure
\begin{equation}
\label{eq_probabilistic_attention}
    \pp\mbox{-}\operatorname{attention}_{N,Q}\big(w,(Y,v)\big)
   		\eqdef 
    \sum_{n=1}^N \, [\operatorname{Softmax}_N(w)]_n\, 
    \left(
    	\sum_{q=1}^Q\,
    	 	v_{n,q}
    	 	\delta_{Y_{n,q}}
    \right)
.
\end{equation}
To keep notation concise, we will not explicitly write the dependence of $\mathbb{P}\mbox{-}\operatorname{attention}_{N, Q}$ on $\yyy$.  

\paragraph{Connection to Classical Attention}
The classical \textit{attention mechanism} of \cite{bahdanau2014neural} is a central component of the state-of-the-art natural language processing (NLP) model of \cite{Vaswanietal__Transformer_2017}, called the \textit{transformer} network.  The classical attention mechanism maps a matrix of ``queries'' $Q$, a matrix of ``keys'' $K$, and a matrix of ``values'' $V$ to the quantity $\operatorname{Softmax}_N(QK^{\top})V$.  As in most theoretical studies of deep learning tools, e.g.\ in \cite{PetersonVoigtlander_2020_EquivalenceCNNFFNN,UniversalDeepConv}, we focus on a simplified version of the classical attention mechanism which maps a $w\in \mathbb{R}^N$ and a set of $Y_{1,1},\dots,Y_{N,1}$ in $\mathbb{R}^D$ to their convex combination $\sum_{n=1}^N \, [\operatorname{Softmax}_N(w)]_n\, 
Y_{n,1}$.  This simplification highlights the connection between probabilistic attention and classical attention mechanisms, namely, if%
\footnote{A similar relation holds when $Q>1$.}
$Q=1$ then
\[
	   \mathbb{E}_{Y\sim \pp\mbox{-}\operatorname{attention}_{N,1}(w,(Y,v))}[Y]
	=
		\sum_{n=1}^N \, [\operatorname{Softmax}_N(w)]_n\, 
		Y_{n,1}.
\]
Thus, the classical attention of \cite{Vaswanietal__Transformer_2017} is implemented in the probabilistic attention mechanism ``on average''.  Therefore, the probabilistic attention mechanism captures more information than the ``deterministic'' attention mechanism of \cite{Vaswanietal__Transformer_2017}.  Namely, the classical attention mechanism cannot encode probabilistic features into its predictions, e.g.\ variance and higher moments.  

Let $\mathcal{P}_1(\yyy)$ denote the set of probability measures on $\yyy$ with a finite mean (defined below).  We quantify the distance between any two probability measures therein using the $1$-\textit{Wasserstein distance} $\mathcal{W}_1$ (defined below).  The metric space $(\mathcal{P}_1(\yyy),\mathcal{W}_1)$ is called the \textit{$1$-Wasserstein} space.  When clear from the context, it will be abbreviated by $\mathcal{P}_1(\yyy)$.  

Synchronizing with the partner paper \cite{AB_2022__PartnerPaper}, we call a map $\hat{T}:\xxx\rightarrow \yyy$ a \textit{probabilistic transformer} (defined precisely in the paper's main body) if it can be represented as in Figure~\ref{fig_ProbalisticTransformer}.  A focal consequence of this paper's first main result is the following probably approximately correct (PAC) guarantee that the probabilistic transformer model can approximate RCDs.  
\begin{thrminformal}[PAC Universal Approximation of {$\ppp[\yyy][1]$}-Borel Functions]
\label{thrm_UAT_Measurable_Variant}
    Let $\xxx$ be a suitable metric space, $\yyy\subseteq \rr^D$ be closed, $\varphi:\xxx\rightarrow \rr^d$ be a regular feature map.  Let $X$ and $Y$ be random variables taking values in $\xxx$ and $\yyy$, respectively.  For every $0<\epsilon\le 1$, there is a probabilistic transformer $\hat{T}:\xxx\rightarrow \mathcal{P}_1(\yyy)$ for which $\mathcal{W}_1(\hat{T}(x),\mathbb{P}(Y\vert X=x))<\epsilon$ with probability at-least $1-\epsilon$.  
\end{thrminformal}
Informal Theorem~\ref{thrm_UAT_Measurable_Variant} is a qualitative worst-case guarantee that addresses one of the approximation theoretical questions on the PT's asymptotic expressiveness.  Namely, ``given enough parameters \textit{can} the probabilistic transformer approximate any continuous RCD?''  However, for this model to be practically feasible, we must address the following converse question:
\[
\mbox{``When can the probabilistic transformer model overcome the curse of dimensionality?''}
\]
There are at least two scenarios when this question admits an affirmative answer.  The first scenario occurs when the target function is regular enough.  The second scenario happens when only considering inputs in a compact subset which is ``near any given (finite) training dataset''.   

To explain the first case, let us briefly consider the problem of approximating a continuous function between Euclidean spaces with ReLU feedforward neural networks uniformly on compact sets.  In \cite{ShenYangZhang__OptimalApproxRatesReLUWidthDepth_2022_JPAMath}, the authors demonstrate an optimal neural network approximation of an arbitrary such continuous function can only be implemented by a neural network whose number of parameters grows exponentially as a function of input space's dimension.  A key point here is that the target function is only assumed to be continuous, and no additional structure is presupposed.  
The typical strategy to circumventing this problem was first proposed in \cite{barron1993universal} and is summarized as \textit{``[the curse of dimensionality] is avoided by considering sets [of target functions] that are more and more constrained as [the input space's dimension] increases"}.  The most common constraint requires smoothness constraints on the target function through moment constraints on its Fourier transform (see \cite{SIEGEL2020313}), Besov-type constraints on the target function, as in \cite{Suzuki2018ReLUBesov,gribonval2019approximation}, but many other options also exist.  Our first result in this direction builds classes of functions in $C(\xxx,\ppp[\yyy][1])$, which can be approximated by our PT model without suffering from the curse of dimensionality.  
These functions are pieced together using classes of functions that classical feedforward neural networks can efficiently approximate. 

\begin{thrminformal}[{$\ppp[\yyy][1]$-Measure Valued Functions that are Efficiently Approximable}]
\label{informal_theorem_f_structure}\hfill\\
We construct a broad class of functions $\mathcal{F}\subseteq C(\xxx,\mathcal{P}_1(\yyy))$ with the property: for any error $\epsilon>0$ and any compact $K\subseteq \xxx$ there is a probabilistic transformer $\hat{T}:\xxx\rightarrow \mathcal{P}_1(\yyy)$ defined by $\mathscr{O}(\epsilon^{-2})$ parameters such that: if $x\in K$ then $1$-Wasserstein distance between $f(x)$ and $\hat{T}(x)$ is at-most $\epsilon$.
\end{thrminformal}

Informally, Theorem~\eqref{informal_theorem_f_structure} exhibits a class of functions which probabilistic transformers can efficiently approximate.  However, in practice, the target RCD need not belong to this class of functions.  Our last main result states that the curse of dimensionality can \textit{always} be avoided by ``localizing'' our set of inputs about a given (finite) training dataset and only approximating the target RCD on this ``localized'' set.  In contrast, most other quantitative universal approximation theorems in the literature provide a single approximation rate valid for all compact subsets of the input space, regardless of any compact set's size and geometry.  These ``localized datasets'' are illustrated pictorially in Figure~\ref{fig_not_localization}.  

\begin{figure}[H]
\centering
\begin{subfigure}{.5\textwidth}
    \centering
    \includegraphics[width=0.40\linewidth]{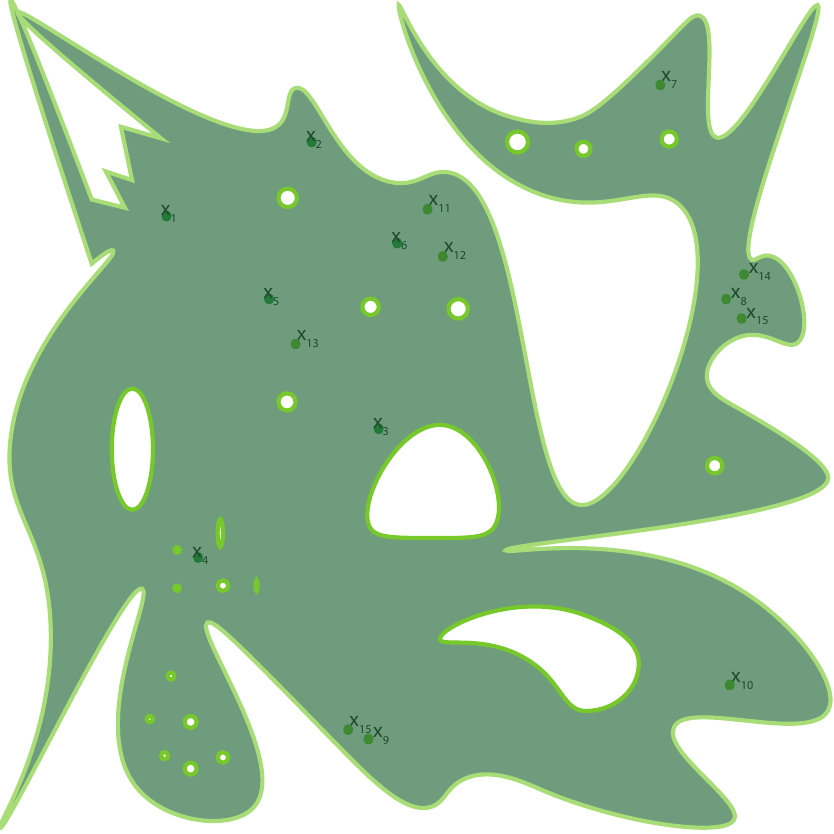}
    \caption{``Large and pathological'' compact set containing data.}
    \label{fig_localization}
\end{subfigure}%
\begin{subfigure}{.5\textwidth}
\centering
    \includegraphics[width=0.40\linewidth]{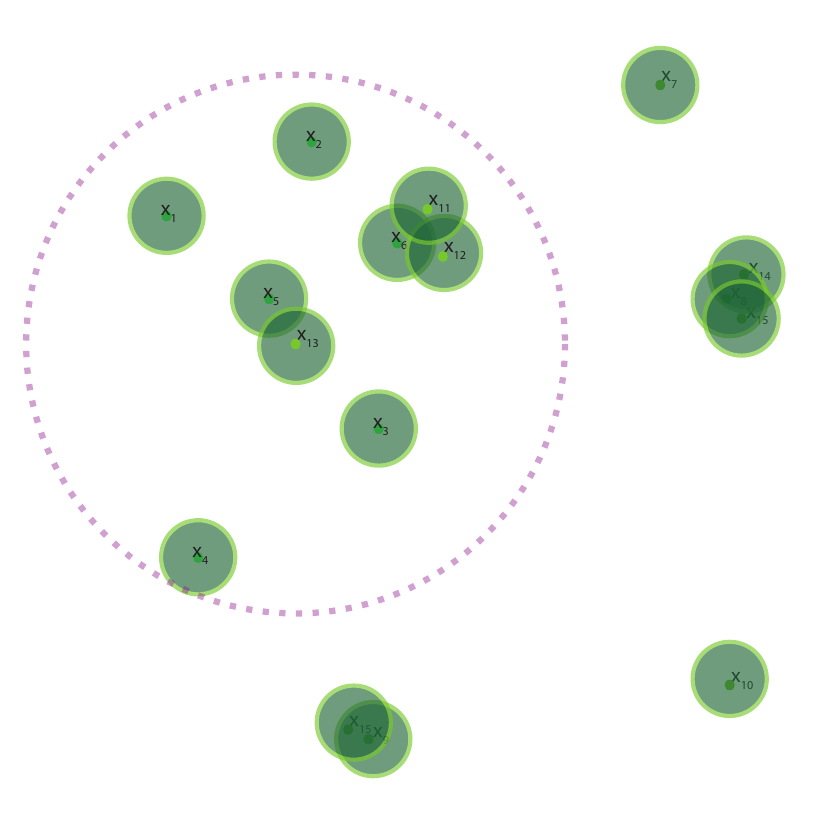}
    \caption{Localization of the dataset about $x_{13}$.}
    \label{fig_not_localization}
\end{subfigure}
\caption{The localization of a dataset is ``well-structured'' while the arbitrary compacta containing the data $x_1$, $\dots$, $x_{18}$ can be ``too large'' have a ``pathological'' geometry.}
\label{fig_general_vs_localized}
\end{figure}

As illustrated in Figure~\ref{fig_general_vs_localized}, the issue is that general uniform approximation guarantees require a model to accommodate arbitrarily large patches of the input space containing a given dataset.  However, expecting an approximation to only hold near a given training dataset is more realistic.  Figure~\ref{fig_localization} illustrates what is meant by ``localization near the dataset $x_1,\dots,x_{18}$'' (about the datum $x_{13}$).  The localized dataset is illustrated by the union of the {\color{IllustrationGreen}{green bubbles}} lying within the {\color{violet}{purple circle}}.  We can guarantee that the target function can be approximated with a probabilistic transformer determined by a small number of parameters.  The diameter of the {\color{violet}{purple circle}} is proportional to the regularity of the target function.  In particular, the curse of dimensionality can always be avoided by appropriately shrinking the {\color{IllustrationGreen}{green bubbles}} and {\color{violet}{purple circle}}.

\begin{thrminformal}[Localization Avoids the Curse of Dimensionality]
\label{informaltheorem_curseofdim}
Let $f:\mathbb{R}^d\rightarrow \mathcal{P}_1(\mathbb{R}^D)$ be continuous and suppose that each of the measures $f(x)$ is supported a low-dimensional manifold.  
Let $\{x_i\}_{i=1}^n$ be a finite subset of $\mathbb{R}^d$.  For every $\epsilon>0$ we construct a compact subset $K$ of $\mathbb{R}^d$, of positive Lebesgue measure, which contains $\{x_i\}_{i=1}^n$ and such that there is a probabilistic transformer $\hat{T}:\mathbb{R}^d \rightarrow \mathcal{P}_1(\mathbb{R}^D)$ whose number of parameters is a polynomial function of $\epsilon^{-1}$, $d$, and $D$ such that: for every $x\in K$ the $1$-Wasserstein distance between $f(x)$ and $\hat{T}(x)$ is at-most $\epsilon$.
\end{thrminformal}

\subsection*{The Probabilistic Transformer}
\label{ss_Models}
We require the following condition of our continuous feature map $\varphi\in C(\xxx,\rrd)$.  A feature map $\varphi$ is said to be \textit{UAP-invariant} if it preserves the universal approximation property (UAP) of any model class upon pre-composition; i.e., $\mathcal{F}\subseteq C(\rrd,\yyy)$ is dense\footnote{Density in $C(\rrd,\yyy)$ and in $C(\xxx,\yyy)$ is for the uniform convergence on compacts topology.} if $\{\hat{f}\circ \varphi\}_{\hat{f}\in \mathcal{F}}$ is dense in $C(\xxx,\yyy)$.  Conveniently, \citep[Theorem 3.4]{kratsios2020non} shows that UAP-invariance is equivalent to the injectivity of $\varphi$.
 
For simplicity, we will instead assume a stronger structural condition on the feature map $\varphi$.  This condition assumes that $\varphi$ and its inverse (on $\varphi(\xxx)$) are not too irregular, more precisely we require that $\varphi$ is invertible on its image and that $\varphi$ and its inverse are both H\"{o}lder.  Indeed most Euclidean optimal embeddings of well-behaved metric spaces have this property (see \citep[Chapter 12]{heinonen2001lectures} for details and  \cite{NaorNeiman_QuantitativeAssouad_2012,DavidSnipes_NonProbabilisticAssouad_2013} for a quantitative refinement of that result).  
\begin{condition}[{Bi-H\"{o}lder Feature Map}]
\label{cond_feature}
There are constants $0< C_{\varphi:L} \le C_{\varphi:U}$ and an $\alpha \in (0,1]$ such that: for every $x,\tilde{x}\in \xxx$ the following holds
\[
    C_{\varphi:L}\,
    d^{\alpha}(x,\tilde{x})
        \leq 
    \|
        \varphi(x)
            -
        \varphi(\tilde{x})
    \|
        \leq
    C_{\varphi:U}\,
    d^{\alpha}(x,\tilde{x})
    .
\]
\end{condition}
\begin{example}[Finite-Dimensional Affine Subspaces of The Hilbert Space]
Let $(\mathcal{H},\langle \cdot,\cdot\rangle)$ be a separable infinite-dimensional Hilbert space, $d<\dim(\mathcal{H})$, $x_0\in \mathcal{H}$, $\{x_i\}_{i=1}^d$ be an orthonormal subset of $\mathcal{H}$ not containing $\{x_0\}$, and set $\xxx
\eqdef \operatorname{span}\{x_1,\dots,x_d\}+x_0$.  
The map $\varphi:\xxx\rightarrow \rr^d$ defined by
\[
    \varphi(x) 
        =
    \big(
        \langle x-x_0, x_i\rangle
    \big)_{i=1}^d
,
\]
is a UAP-invariant feature map satisfying~\eqref{cond_feature} with $\alpha=1$. 
\end{example}
The following example is due to Gromov \cite{Gromov_FillingRiemannianManifolds_1983_Systol_JDiffGeo}.
\begin{example}[Bi-Lipschitz Feature Maps for Compact Connected Riemannian Manifolds]
Let $\mathcal{X}$ be a compact connected Riemannian manifold with geodesic distance $d$.  By \citep[Theorem 1.1]{Katz2_BiLipschitzImbeddingsEuclidean_2011_GeoDedic} for any sufficiently small\footnote{Here, sufficiently small means $\delta \in \big(
    0
        ,
    \frac{\operatorname{Sys}(\xxx)}{10}
\big)$ where $\operatorname{Sys}(\xxx)$ is the shortest length of a non-contractible loop $\gamma:[0,1]\rightarrow \xxx$ (i.e.\ $\gamma(0)=\gamma(1)$ for which there is no homotopy to a point in $\xxx$).  For further details on $\operatorname{Sys}(\xxx)$ see \cite{Guth_MetaphorsinSystolicGeometry_2010}.}
$\delta>0$, and any maximal $\delta$-separated subset%
\footnote{A subset $\{x_i\}_{i \in I}$ is $\delta$-\textit{separated}, for a fixed $\delta>0$, if $\min_{i,j\in I;\,i\neq j}\,d(x_i,x_j)\geq \delta$.  This subset is a \textit{maximal $\delta$-separated subset} if there is no larger subset of $\xxx$ which is $\delta$-separated.}
$\{x_i\}_{i\in I}$ of $\mathcal{X}$, the map $\varphi:\xxx\rightarrow \rr^I$ defined for any $x\in \xxx$ by
\[
    \varphi(x) 
        \eqdef 
    (d(x,x_i))_{i\in I},
\]
satisfies Condition~\ref{cond_feature} with $\alpha=1$.  In particular, $\varphi$ is UAP-invariant.  
\end{example}

We now define the main model class studied in this paper.  

\begin{definition}[Probabilistic Transformer]
\label{defn_proabilitic_transformer}
Let $\varphi:\xxx\rightarrow \rr^d$ be a feature map, $\sigma\in C(\rr)$, and $\yyy\subseteq \rrD$ be non-empty.  A map $\hat{T}:\xxx\rightarrow \mathcal{P}_1(\yyy)$ is called a \textit{probabilistic transformer} if it admits the iterative representation
\begin{equation}
\label{eq__defn_proabilitic_transformer___representation}
\begin{aligned}
\hat{T} (x) & \eqdef  
	 \pp\mbox{-}\operatorname{attention}_{N,Q}\big(
		x^{J+1}
	,
		(Y,v)
	\big)
\\
x^{(j+1)} & \eqdef A^{(j)} \sigma \bullet x^{(j)} + b^{(j)} \qquad j=0,\dots,J\\
x^{(0)} & \eqdef x
,
\end{aligned}
\end{equation}
where $\bullet$ denotes component-wise composition, for each $j=0,\dots,J+1$ $A^{(j)}$ is a $d_{j+1}\times d_j$-matrix and $b^{(j)}\in \rr^{d_{j+1}}$ with $d_{j+1}=\rr^N$ and $d_0=d$, and where $(Y,v)\in \big(\mathcal{Y}^{Q}\times \Delta_Q\big)^N$, $w \in \Delta_N$, for some positive integers $N$ and $Q$.  
The integer $J$ is called $\hat{T}$'s \textit{depth}, $\max_{j=0,\dots,J+1}\,d_j$ is called $\hat{T}$'s \textit{width}, and the \textit{number of parameters} is
\[
		\operatorname{Param}(\hat{T})
	\eqdef 
		\underbrace{
		    NQ(D+1)
		}_{\pp\text{-Attention Parameters}}
		+
		\underbrace{
		    \sum_{k=0}^J\,
		        d_{j+1}(d_j+1)
		}_{\text{Feedforward Network Parameters}}
	.
\]
The set of all $\hat{T}$ with representation~\eqref{eq__defn_proabilitic_transformer___representation} is denoted by $\NN[\varphi,\yyy][\sigma:\star]$.
\end{definition}

\begin{definition}[Empirical Measures in General Position]
\label{definition_em_general_position}
A set of empirical probability measures $\big\{\frac1{Q}\sum_{q=1}^{Q_n} \delta_{y,q}\big\}_{n=1}^N$ is said to be in \textit{general position} if the $y_{1,1},\dots,y_{N,Q_N} \in \yyy$ are all distinct.
\end{definition}
\begin{example}
Consider the probability measures $\mu_1\eqdef \delta_1,\mu_2\eqdef \frac1{2}\delta_1 + \frac1{2}\delta_0, \mu_3\eqdef \delta_0$ on $\rr$.  Together, $\mu_1$ and $\mu_3$ are in general position while $\mu_1$ and $\mu_2$ are not.  
\end{example}
\subsection{Notation}
\label{ss_intro_Notation}
Given any uniformly continuous function $f:\xxx\rightarrow \yyy$, such as any continuous function defined on a compact domain, we use $\omega_f$ to denote a modulus of continuity of $f$.  
If $\tilde{\omega}_f(t)\mapsto 
\sup\{d_Y(f(x),f(x')):x\in \xxx,x'\in X,d_{\xxx}(x,x')\leq t \}$ is right-continuous then we set $\omega_f\eqdef \tilde{\omega}_f$ to be this function\footnote{This particular modulus of continuity is often called the optimal modulus of continuity of $f$.}.  
If $\tilde{\omega}_f$ is not right-continuous then, we set $\omega_f(t)\eqdef \frac{1}{t}\int_t^{2t}\tilde{\omega}_f(s)ds$, which is necessarily continuous.  When restricting any continuous function to a compact set, we always let $\omega_f$ be the modulus of continuity just described; thus, we omit the dependence of $\omega_f$ on the particular compact.
Following \cite{EmbrechtsHofert}, we define the \textit{generalized inverse} of $\omega_f$ by $\omega_f^{-1}(t)\eqdef \sup\{s:\omega_f(s)\leq t\}$. Note that if $\omega_f$ is invertible, then its generalized inverse and its usual inverse coincide.

Fix $1\le q<\infty$.  The set $\ppp[\yyy][q]\subseteq \ppp[\yyy][1]$ consist of of all $\mu \in \ppp[\yyy][1]$ for which the value of $\int_{y \in \yyy} d_{\yyy}^q(y_0,y)d\mu(y)$ is finite for some fixed $y_0\in \yyy$.  We equip $\ppp[\yyy][q]$ with the \textit{subspace topology} induced by inclusion in $\ppp[\yyy][1]$.  We emphasize that when $1<q<\infty$, $\ppp[\yyy][q]$ is not equipped with the Wasserstein-$q$ metric but rather with the Wasserstein-$1$ metric.  

The \textit{diameter} of any subset $A\subseteq \xxx$ is defined to be $\operatorname{diam}(A)\eqdef \sup_{x_1,x_2\in A}\, d_{\xxx}(x_1,x_2)$.  We denote the \textit{cardinality} of any $A\subseteq \xxx$ by $\# A$.  The rectified linear unit (ReLU) activation function is defined by $\operatorname{ReLU}(x) \eqdef \max\{0,x\}$, for every $x\in \mathbb{R}$.  

\section{Main Universal Approximation Results}\label{s_main_result}
We now present the paper's main results.  
\subsection{Universal Approximation: The General Case}
\label{ss_MAIN_Unstructured_Case}
Our first main result concerns the universal approximation capabilities of the probabilistic transformer model.  Two errors appear in the theorem: a \textit{quantization error} $(\epsilon_Q>0)$ and an \textit{approximation error} $(\epsilon_A>0)$.  These errors respectively quantify the error made by reducing the infinite-dimensional problem to a finite-dimensional problem and the error incurred by approximately solving that finite-dimensional problem.  

The following result is approximately optimal, in the sense that it guarantees the existence of a PT which uniformly approximates the target function $f:\xxx\rightarrow\mathcal{P}_1(\yyy)$, it implements an $\frac{\epsilon_Q}{2}$-metric projection onto the hull of a finite number of probability measures supported on $Q$ points, and together these probability measures approximately optimally ``discretize $\xxx$'s image under $f$''.  For any \textit{quantization error} $\epsilon_Q>0$, $N$ is the smallest number of probability measures $\mu_1,\dots,\mu_N$ in $f(\xxx)$ which form $\frac{\epsilon_Q}{2}$-net of $f(\xxx)$.  The integer $Q$ is the smallest number of points in $\yyy$ required to approximate each of these probability measures $\mu_1,\dots,\mu_N$ by finitely supported probability measures to an error of $\frac{\epsilon_Q}{2}$, in $1$-Wasserstein distance.  
That is,\ for each $n$, 
$
    \mathcal{W}_1\big(
        \mu
    ,
        \hat{\mu}_n
    \big) 
        <
    \frac{\epsilon_Q}{2}
$ where $\hat{\mu}_n\eqdef \sum_{q=1}^Q\, v_{n,q}\delta_{Y_{n,q}}$ for some $v_n\in \Delta_Q$ and $Y_1,\dots,Y_Q$ are points in $\yyy$.  

Once the image of $f(\xxx)$ is discretized, we may reduce the problem of approximating $f$ by approximately implementing a map sending $x\in \xxx$ to an element of the $\epsilon_A$-optimality set
\[
	\Big\{
		\nu 
			=
		\sum_{n=1}^N\, w_n \hat{\mu}_n:
		\,
		\mathcal{W}_1(\nu,f(x)) 
		\leq 
		 \inf_{\tilde{w}\in \Delta_N}\,
		 \mathcal{W}_1\big(
		 	f(x)
		 	,
		 	\sum_{n=1}^N\, \tilde{w}_n \hat{\mu}_n
		 \big) 
		 +
		 \epsilon_A/2
	\Big\}
.
\]
Since probability measures formed by convex combinations of $\{\hat{\mu}_n\}_{n=1}^N$ in general position are in $1$-$1$ correspondence with parameters in the $N$-simplex, then we may effectively transcribe our approximation problem to a problem of approximating a continuous map into the $N$-simplex.  The results of \cite{paponkratsios2021quantitative} imply that this function can be implemented to $\frac{\epsilon_A}{2}$-precision while satisfying the constraint that all outputs lie in the $N$-simplex by using a deep feedforward network with softmax output layer.  

\begin{theorem}[Quantitative Universal Approximation: General Case]
\label{theorem_MAIN__GeneralCase}
Let $d,D\in \nn_+$, $\mathcal{X}$ be a compact metric space, $\mathcal{Y}\subseteq \rr^D$ be non-empty, closed, and path-connected, $f:\xxx\rightarrow \mathcal{P}_1(\yyy)$ be uniformly continuous with modulus of continuity $\omega_f$, $\varphi:\xxx\rightarrow \rr^d$ be continuous and injective and let $\sigma \in C(\rr)$ satisfy condition~\ref{condi_KL}.  For every ``quantization error'' $\epsilon_Q>0$ there exists a minimal $N\in \nn_+$ such that there exist ${\mu}_1,\dots,{\mu}_N\in f(\xxx)$ satisfying the covering condition
\[
        \max_{x\in \xxx}\,
        \min_{n\leq N}\,
        \mathcal{W}_1(f(x),{\mu}_n)
    <
        \epsilon_Q/2.
\]
For every ``approximation error'' $\epsilon_A \in (0,1)$ there exists $Q\in \nn_+$ and a probabilistic transformer network $\hat{T}:\xxx\rightarrow \mathcal{P}_1(\yyy)$ satisfying:
\begin{enumerate}
    \item[(i)] \textbf{Universal Approximation:}
    \[
        \max_{x\in \xxx}\,
            \mathcal{W}_1\big(f(x),\hat{T}(x)\big)
        <
        \epsilon_Q + \epsilon_A
    \]
    \item[(ii)] \textbf{Approximate Metric Projection:}
    \[
            \max_{x\in \xxx}\,
            \mathcal{W}_1\big(f(x),\hat{T}(x)\big)
        \leq 
            \max_{x\in \xxx}
            \min_{n\leq N}
            \,
            \mathcal{W}_1(f(x),{\mu}_n) + \epsilon_A + \epsilon_Q/2
        .
    \]
\end{enumerate}
Moreover, the depth, width, and the number of attention parameters defining $\hat{T}$ are recorded in Table~\ref{tab_general_approximation_rates}.  
\end{theorem}

Before discussing the depth, width, and attention complexity of the approximating PT constructed in Theorem~\ref{theorem_MAIN__GeneralCase}, as recorded in Table~\ref{tab_general_approximation_rates}, we showcase some implications of our universal approximation theorem.
Any PT induces a map from $\bar{T}:\xxx\rightarrow \mathbb{R}^{D}$ which sends any $x\in \xxx$ to the vector defined via 
\begin{equation}
\label{eq_induced_map__probtransformer}
        \bar{T}(x)
    \eqdef 
    	\mathbb{E}_{\hat{Y}^x\sim \hat{T}(x)}[Y^x]
    .
\end{equation}
Together Theorem~\ref{theorem_MAIN__GeneralCase} and a result of \cite{BruHeinicheLootgieter1993} imply that if $\hat{T}$ is a probabilistic transformer approximating a uniformly continuous $\mathcal{P}_1(\rr^D)$-valued function $f$ then, $\bar{T}$ uniformly approximates the map sending any $x\in \xxx$ to the mean of the probability measure $f(x)$.  
Similarly to~\eqref{eq_induced_map__probtransformer}, we denote map by $\bar{f}:x\mapsto  \mathbb{E}_{Y\sim f(x)}[Y]$.  

\begin{corollary}[Uniform Approximation of Mean]
\label{cor_mean_prediction}
Consider the setting of Theorem~\ref{theorem_MAIN__GeneralCase}.  It holds that
\[
    \max_{x\in \xxx}\,
    \Big\|
    	\bar{f}(x)
    	-
		\bar{T}(x)
    \Big\|
        <
    \epsilon_Q  + \epsilon_A.
\]
\end{corollary}
The following corollary shows that once we have approximated a probability measure-valued map $f$ by a probabilistic transformer, we can directly approximately evaluate integrals of Lipschitz integrands integrated against our approximations $\hat{T}(x)$ to each of the probability measures $f(x)$.  

\begin{corollary}[Uniform Quadrature for Lipschitz Test Functions]
\label{cor_MC_Lipschitz_test_functions}
Consider the setting of Theorem~\ref{theorem_MAIN__GeneralCase} 
and fix a Lipschitz map $g:\yyy\rightarrow \rr$.  It holds that
\[
    \max_{x\in \xxx}\,
        \left|
            \sum_{n=1}^N\,
            \sum_{q=1}^Q
                \operatorname{Softmax}_N(x^{J+1})
                    \,
                    v_{n,q}
                    g(Y_{n,q})
        -
            \mathbb{E}_{\tilde{Y}\sim f(x)}\big[
            g(\tilde{Y})
            \big]
        \right|
    <
        \operatorname{Lip}(g)(\epsilon_Q+\epsilon_A)
;
\]
where, $\{v_{n,q}\}_{n=1,\dots,N,\,q=1,\dots,Q}$ and $\{Y_{n,q}\}_{n=1,\dots,N,\,q=1,\dots,Q}$ are as in~\eqref{eq__defn_proabilitic_transformer___representation} and $x^{J+1}$ denotes the output of the penultimate layer of $\hat{T}$ given the input $x$.
\end{corollary}

Theorem~\ref{theorem_MAIN__GeneralCase} implies the following probably approximately correct (PAC) approximation guarantee for Borel-measurable functions from $\xxx$ to $\mathcal{P}_1(\rr^D)$.  The result states that any Borel measurable map from a locally compact $\xxx$ to $\mathcal{P}_1(\rr^D)$ can be approximated by a PT.

\begin{corollary}[PAC-Universal Approximation: General Case]
\label{cor_UAT_Qualitative}
Let $d,D\in \nn_+$, $\mathcal{X}$ be a locally-compact metric space, $\mathcal{Y}\subseteq \rr^D$ be closed and path-connected, $\varphi:\xxx\rightarrow \rr^d$ satisfy condition~\ref{cond_feature}, let $\sigma \in C(\rr)$ satisfies condition~\ref{condi_KL}, and let $\mu$ be a Borel probability measure on $\mathcal{X}$. 
For every Borel measurable function $f:\mathcal{X}\rightarrow \mathcal{P}_1(\mathcal{Y})$, and every $\epsilon\in (0,1]$ there is probabilistic transformer $\hat{T}:\xxx\rightarrow \mathcal{P}_1(\yyy)$ satisfying:
\[
    \mathbb{P}\biggl(
        \max_{x\in \xxx}\,
            \mathcal{W}_1\big(f(x),\hat{T}(x)\big)
        <
        \epsilon
        \biggr)
    \geq 
    1-\epsilon
    .
\]
\end{corollary}
We return to the discussion of our main quantitative universal approximation theorem.  
\begin{table}[H]
    \centering
    \ra{1.3}
    \resizebox{\columnwidth}{!}{%
    \begin{tabular}{@{}llll@{}}
    \toprule
    $\boldsymbol{\sigma}$ \textbf{Regularity} 
    & $\boldsymbol{\mathbb{P}}$\textbf{-Attention Parameter} $\boldsymbol{N}$ & \textbf{Width} & \textbf{Depth} 
    \\
    \midrule
    $\sigma = \operatorname{ReLU}$ 
    &
        $\left\lceil
            \frac{C_{\varphi:U}^{2d}
            }{C_{\varphi:L}} 
                \left(
                    \frac{
                        d
                        2^{\frac{5}{2}}
                        \big|
                            \xxx
                        \big|^{\alpha 2d}
                    }{
                        (d+1)^{\frac{1}{2}}
                        \big(
                            \omega_f^{-1}(\frac{4}{3}\epsilon_Q)
                        \big)^{\alpha}
                    }
                \right)^d
        \right\rceil
        $
    &
    $ d(N-1)+3^{d+3}\max\{d,3\}$
    &
    $ N\left(
            C_d^{(1)}
        +
            11
            \left\lceil
                    \frac{
                        {C}_d^{(2)} 
                        |\xxx|^{\alpha d /2}
                    }{
                    \Big(
                        \omega_f^{-1}\Big(
                            {C}_d^{(3)} 
                            \,
                            \frac{\epsilon_A^2}{
                            |2-\epsilon_A|}
                            \,
                            \frac{
                                N - \sqrt{N}- 1
                            }{
                                N^{3/2}
                            }
                            \Big)
                    \Big)^{\alpha d/2}
                    }
            \right\rceil
        \right)
    $
    \\
    Smooth and Not Polynomial
     & 
     $\left\lceil
            \frac{C_{\varphi:U}^{2d}
            }{C_{\varphi:L}} 
                \left(
                    \frac{
                        d
                        2^{\frac{5}{2}}
                        \big|
                            \xxx
                        \big|^{\alpha 2d}
                    }{
                        (d+1)^{\frac{1}{2}}
                        \big(
                            \omega_f^{-1}(\frac{4}{3}\epsilon_Q)
                        \big)^{\alpha}
                    }
                \right)^d
        \right\rceil
    $
    &
    $2 + d + N$
    & 
    $
    \mathscr{O}\left(
        \frac{
            C_{\varphi,d}\,
            N 
	        \big(
	            |\xxx|^{2d/\alpha}
	        \big)
        }{
            \omega_f^{-1}\Big(
                C_{d,\hat{\mu}_{\cdot}}
                \frac{
                    \epsilon_A
                }{
                    |2-\epsilon_A|
                }
                \frac{
                    N - \sqrt{N}- 1
                }{
                    N^{1+2d}
                }
                \,
                    \big(
                		\epsilon_A'
                	\big)^{2d}
                \Big)
            \Big)^{\alpha}
        }
    \right)
    ,
    $
    \\
    Polynomial Degree at-least $2$
     & 
     $\left\lceil
            \frac{C_{\varphi:U}^{2d}
            }{C_{\varphi:L}} 
                \left(
                    \frac{
                        d
                        2^{\frac{5}{2}}
                        \big|
                            \xxx
                        \big|^{\alpha 2d}
                    }{
                        (d+1)^{\frac{1}{2}}
                        \big(
                            \omega_f^{-1}(\frac{4}{3}\epsilon_Q)
                        \big)^{\alpha}
                    }
                \right)^d
        \right\rceil
    $
    &
    $2 + d + N$
    & 
    $
    \mathscr{O}\left(
        \frac{
            C_{\varphi,d}\,
            N 
	        \big(
	            |\xxx|^{(4d+2)/\alpha}
	        \big)
        }{
            \omega_f^{-1}\Big(
                C_{d,\hat{\mu}_{\cdot}}
                \frac{
                    \epsilon_A
                }{
                    |2-\epsilon_A|
                }
                \frac{
                    N - \sqrt{N}- 1
                }{
                    N^{3+4d}
                }
                \,
                \big(
            		\epsilon_A'
                	\big)^{4d+2}
                \Big)
            \Big)^{\alpha}
        }
    \right)
        $
    \\
    Continuous and Not Affine 
     & 
     $\left\lceil
            \frac{C_{\varphi:U}^{2d}
            }{C_{\varphi:L}} 
                \left(
                    \frac{
                        d
                        2^{\frac{5}{2}}
                        \big|
                            \xxx
                        \big|^{\alpha 2d}
                    }{
                        (d+1)^{\frac{1}{2}}
                        \big(
                            \omega_f^{-1}(\frac{4}{3}\epsilon_Q)
                        \big)^{\alpha}
                    }
                \right)^d
        \right\rceil
    $
    &
    $2 + d + N$
    &
    -
    \\
    \bottomrule
    \end{tabular}
    }
    \caption{General Approximation Rates.}
    \label{tab_general_approximation_rates}
\end{table}
The constant $C_{d,\hat{\mu}_{\cdot}}>0$ in Table~\ref{tab_general_approximation_rates}, depends on the embedding dimension $d$, the choice of the measures $\{\hat{\mu}_n\}_{n=1}^N$ in Theorem~\ref{theorem_MAIN__GeneralCase}, and on $\sigma$.  Furthermore, the constant $C_{\varphi,d}>0$ depends only on $\varphi$, on the embedding dimension $d$, and on $\sigma$.  Both $C_{d,\hat{\mu}_{\cdot}}$ and $C_{\varphi,d}$ are independent of $\epsilon_A$.  Likewise, the constants $C_d^{(1)},\, C_d^{(2)}, \, C_d^{(3)}>0$ depend only on the embedding dimension $d$ and on $\varphi$'s H\"{o}lder exponent; i.e.\ they depend only on $d$ and on $\alpha\in (0,1]$.

The approximation rates in Table~\ref{tab_general_approximation_rates} contain explicit dependencies on every quantity impacting the probabilistic transformer's approximation error.  If we suppress all constants independent of $\epsilon_A$ and assume that the target function is H\"{o}lder, then the rates may be simplified and easily interpreted.  
\begin{example}[H\"{o}lder Continuous Maps from $\xxx$ to $\mathcal{P}_1(\rr^D)$]
\label{ex_HolderExample}
Consider the setting of Theorem~\ref{theorem_MAIN__GeneralCase} with $\sigma=\operatorname{ReLU}$.  
Let $\varphi:\xxx\rightarrow \rr^d$ be bi-Lipschitz (i.e.\ Assumption~\ref{cond_feature} holds with $\alpha=1$) and let $f:\xxx\rightarrow \mathcal{P}_1(\rr^D)$ be $\beta$-H\"{o}lder.  
Then, for every $0 < \epsilon_A<1$, the probabilistic transformer of Theorem~\ref{theorem_MAIN__GeneralCase} $\hat{T}$ has:
\begin{enumerate}
    \item[(i)] \textbf{Width (N):} $\tilde{\mathscr{O}}\Big(
     \epsilon_Q^{-d/\beta} |\xxx|^{2d}
    \Big)$,
    \item[(ii)] \textbf{Depth:} $\mathscr{O}\Big(
        |\xxx|^{d/2}
        (2 \epsilon_A^2 - \epsilon_A)
        N^{5 d/\beta}
    \Big)$,
    \item[(iii)] \textbf{Attention Parameter $(Q)$:} $\mathscr{O}(\epsilon_Q^{-1/D})$.
\end{enumerate}
where $\tilde{\mathscr{O}}$ also suppresses all constants independent of $\epsilon_A$. 
\end{example}

Comparing the rates in Example~\ref{ex_HolderExample} with the optimal rates at which ReLU feedforward neural networks can approximate H\"{o}lder functions, derived in \cite{ShenYangZhang__OptimalApproxRatesReLUWidthDepth_2022_JPAMath}, illustrates a few key differences between our probability measure-valued approximation problem and the classical Euclidean-valued problem.  The first difference is that the network's width diverges as the quantization error $\epsilon_Q$ approaches $0$; this effect reflects the infinite dimensionality of $\mathcal{P}_1(\rr^D)$.  The following distinction between probability measure-valued deep neural approximation and Euclidean-valued deep neural approximation is the appearance of a ``third network dimension'', namely, the attention parameter $Q$.  This parameter expresses the error incurred by approximately implementing probability measures in the image space $f(\xxx)$ only using finite-dimensional parameters.  In contrast, when approximating Euclidean targets, every vector is exactly implementable with finite-dimensional parameters.  

The worst-case rates of Table~\ref{tab_general_approximation_rates} motivate the following question:
\[
\mbox{\textit{What structures on $f(\xxx)$ accelerates the approximation rates of Theorem~\ref{tab_general_approximation_rates}?}}
\]
We answer this question by ablating the role of each parameter in the probabilistic transformer.  

\subsection{Probabilistic Attention Mechanisms Defined by Few Parameters}
\label{ss_MAIN_Structured_Case__Improved_Q}

We consider structures on $f(\xxx)$ which decelerate the rate at which $Q$ and $N$ grow as the quantization error $\epsilon_Q$ tends to $0$, as compared to the rates derived in Theorem 1~\ref{theorem_MAIN__GeneralCase}.  

\begin{table}[H]
    \centering
    \ra{1.3}
    \resizebox{\columnwidth}{!}{
    \begin{tabular}{@{}lllll@{}}
    \toprule
    Structure 
    & \textbf{Width} & $\boldsymbol{\mathbb{P}}$\textbf{-Attention Parameter} $\boldsymbol{Q}$ & $\boldsymbol{\mathbb{P}}$\textbf{-Attention Parameter} $\boldsymbol{N}$ & \textbf{Theorem} 
    \\
    \midrule
    $f(\xxx)\subseteq \mathcal{P}_p(\rr^D)$; $p>1$, $d>1$
    &
    $
    \mathcal{O}\Big(
        d 
            +
        \exp\big(
            -\log(\epsilon_Q)
            \frac{2-p}{1-p}\epsilon_Q^{-D/(1-p)}
            \tilde{C}_D
        \big)
    \Big)
    $
    &
    $
    C_{D,p} \epsilon_Q^{-D + D/p}
    $
    &
    $
    \mathcal{O}\Big(
        \exp\big(
            -\log(\epsilon_Q)
            \frac{2-p}{1-p}\epsilon_Q^{-D/(1-p)}
            \tilde{C}_D
        \big)
    \Big)
    $
    &
    Theorem~\ref{theorem_pmoments}
    \\
    $f(\xxx)$ 
    Supported in $\overline{B_{\rr^D}(0,1)}$; $D>1$
    &
    $
    \mathcal{O}\Big(
        d + 
        \exp\big(
                -\log(\epsilon_Q) 
                C_D 
                \epsilon_Q^{-D}
            \big)
    \Big)
     $
     & 
     $\mathcal{O}\left(
     C^{1/D}
        \epsilon_Q^{-1/D}
     \right)$ 
     & 
    $
    \mathcal{O}\Big(
        \exp\big(
            -\log(\epsilon_Q/3) 
            C_D 
            (\epsilon_Q/3)^{-d}
        \big)
    \Big)
    $
    & 
    Theorem~\ref{theorem_unit_ball}
     \\
    $f(\xxx)$ Lower Ahflors $(C,q)$-Regular
    & -
    &
    $
    \big\lceil \epsilon_Q^{-q}\,(C^{-1}5^q) \big\rceil
    $
    & - 
    & 
    Theorem~\ref{theorem_MAIN__AhlforsRegularCase}
    \\
    \bottomrule
    \end{tabular}
    }
    \caption{Efficient Attention Mechanisms}
    \label{tab_QuantizationControl}
The constants $C,c>0$ in Table~\ref{tab_general_approximation_rates} depend only on the quantization error $\epsilon_Q$ and the constants $C_D,\tilde{C}_D>0$ constants depend only on the dimension $D$.
\end{table}

\begin{theorem}[Universal Approximation: Compact Supported Supported Measures Case]
\label{theorem_unit_ball}
Consider the setting of Theorem~\ref{theorem_MAIN__GeneralCase} with $D>1$, and suppose that for every $x\in \xxx$, the probability measure $f(x)$ is supported on the closed $D$-dimensional Euclidean unit ball $\overline{B_{\rr^D}(0,1)}$. Then, the conclusion of Theorem~\ref{theorem_MAIN__GeneralCase} holds with 
\[
    Q
        \eqdef
    \Big\lceil
    \epsilon_Q^{-D}
     \Big(\frac{4D}{D-1}\Big)^D
    \Big\rceil
        \mbox{ and }
    N \mbox{ is of the order }
    \mathcal{O}\Big(
        \exp\big(
            -\log(\epsilon_Q/3) 
            C_D 
            (\epsilon_Q/3)^{-d}
        \big)
    \Big).
\]
\end{theorem}
\begin{theorem}[Universal Approximation: Higher Moments]
\label{theorem_pmoments}
Consider the setting of Theorem~\ref{theorem_MAIN__GeneralCase} with $D>1$, and suppose that for every $x\in \xxx$, $f(x)$ is supported on the closed $D$-dimensional Euclidean unit ball $\overline{B_{\rr^D}(0,1)}$.  Then there is a probabilistic transformer $\hat{T}:\xxx\rightarrow \ppp[\rr^D][1]$ satisfying
\[
    \max_{x\in \xxx}\,
        \mathcal{W}_1\big(
            f(x)
        ,
            \hat{T}(x)
        \big)
            <
        \epsilon_A
            +
        o\big(
           \epsilon_Q
        \big)
    .
\]
Furthermore the depth and width of $\hat{T}$ are as in Table~\ref{tab_general_approximation_rates} but with 
\[
    Q
        = 
    C_{D,p} 
    \epsilon_Q^{-D + D/p}
    \mbox{ and with $N$ of the order }
    \mathcal{O}\Big(
        \exp\big(
            -\log(\epsilon_Q)
            \frac{2-p}{1-p}\epsilon_Q^{-D/(1-p)}
            \tilde{C}_D
        \big)
    \Big)
    ,
\]
and the constant $C_{D,p}>0$ depends only on $D$ and on $p$.
\end{theorem}
Next, we focus on target functions that output probability measures supported on low-dimensional structures in $\mathbb{R}^D$ and such that each output measure does not place too much mass on small subsets of those low-dimensional structures.  Following, \cite{KloecknerQuantizationAhlforsRegular2012,Quasissymmetric_dimension_Ahlfors_Bishopetal_2016,SemmesSurfaces_Fassleretal_2020} we formalize this condition using the notation of Ahlfors lower-regularity of a probability measure.  Let $C,q>0$ and $0<q\leq D$.  We call a compactly supported probability measure $\pp$ on $\mathbb{R}^D$ \textit{Ahlfors lower $(C,q)$-regular} if there is a constant $\tilde{C}\geq C$ such that: for $\mu$-almost every $y\in \operatorname{supp}(\pp)$ and every $|\operatorname{supp}(\mu)|\geq r>0$ it holds that
\[
        \tilde{C}r^q
    \leq 
        \pp\big(
            B_{\rr^D}(y,r)
        \big)
.
\]
The Riemannian volume measure on any compact Riemannian submanifold of $\rr^D$ (see \cite{AlvaroMitrea_2015_HardyAhlforsRegular_2015}) and self-similar measures on fractals (see \citep[Theorem 4.7]{Triebel_FractalsSpectra_2011}) are Ahlfors lower-regular probability measures.  The set of Ahlfors lower $(C,q)$-regular probability measures on $\mathbb{R}^D$ is denoted by $\mathcal{AP}^{C,q}(\rrD)$.  

\begin{theorem}[Universal Approximation: Ahlfors Lower Regular Case]
\label{theorem_MAIN__AhlforsRegularCase}
Consider the setting of Theorem~\ref{theorem_MAIN__GeneralCase}.  
Let $C,q>0$ and $0<q\leq D$ and suppose moreover that $f(\xxx)\subseteq \mathcal{AP}^{C,q}(\rrD)$.  Then, the conclusion of Theorem~\ref{theorem_MAIN__GeneralCase} holds with 
\[
    Q
        \eqdef 
    \big\lceil
        \epsilon_Q^{-q}\,(5^{q}C^{-1})
    \big\rceil.
\]
\end{theorem}
To illustrate the role of each constant in the lower $(C,q)$-Ahlfors regularity condition, we consider the particular case where the probability measures in $f(\xxx)$ are supported on a latent $q$-dimensional affine subspace of $\rr^D$; here, $q$ is a positive integer less than $D$.
\begin{corollary}[Uniformly-Lower Bounded Densities on Affine Subspaces]
\label{cor_lowerboundeddensity}
Let $H\subseteq \rr^D$, $q\le D$ be a positive integer, dimensional affine subspace, and let $\mu$ denote the $q$-dimensional Lebesgue measure on $H$.  
Consider the setting of Theorem~\ref{theorem_MAIN__GeneralCase} and suppose that for every $x\in \xxx$, the measure $f(x)$ is absolutely continuous with respect to $\mu$.  Suppose also that $f(x)$ is supported on $[0,1]^D\cap H$
\[
    0<
        C \eqdef 
    \inf_{x\in \xxx,u\in [0,1]^D\cap H}\,
    \Big\|
        \frac{df(x)}{d\mu}(u)
    \Big\|
    .
\]
Then, the conclusion of Theorem~\ref{theorem_MAIN__GeneralCase} holds with 
$
    Q
        \eqdef 
    \Big\lceil
        \epsilon_Q^{-q}\,
        \left(
            \frac{5^q
                \Gamma(\frac{q}{2}+1)
            }{
                \pi^{q/2}
                C
            }
        \right)
    \Big\rceil.
$
\end{corollary}

\subsection{Narrower Transformers Width - For Geometric Priors}
\label{ss_MAIN_Structured_Case}

Theorem~\ref{theorem_MAIN__AhlforsRegularCase} and Corollary~\ref{cor_lowerboundeddensity} describe a situation where the probability measures in $f(\xxx)$ are all supported on the same low-dimensional structure in $\rr^D$.  This is one possible interpretation of the ``manifold" hypothesis in machine learning \cite{VincentPascaletalManifoldDenoising2008}.  An alternative interpretation of the low-dimensional ``manifold'' hypothesis, within this paper's context, is that there is a low-dimensional structure in the $\mathcal{P}_1(\rr^D)$ containing $f(\xxx)$.  The following few results show that this alternative interpretation slows the rate at which $N$ grows as a function of the quantization error $\epsilon_Q$.

\begin{table}[H]
    \centering
    \ra{1.3}
    \resizebox{\columnwidth}{!}{
    \begin{tabular}{@{}lllll@{}}
    \toprule
    Structure 
    & Width & $\mathbb{P}$-Attention Parameter $Q$ & $\mathbb{P}$-Attention Parameter $N$ & Theorem 
    \\
    \midrule
    $f(\xxx)\subseteq \{\pp_{\theta}\}_{\theta \in [0,1]^{d_{\mathcal{M}}}}$ 
    & 
    $
    \mathcal{O}\Big(
            d 
        + 
                \epsilon_Q^{-
                    d_{\mathcal{M}}
                }
    \Big)
    $
    &
    -
    &
        $
        \mathcal{O}\big(
                \epsilon_Q^{-
                    d_{\mathcal{M}}
                }
        \big)
        $
    &
    Theorem~\ref{theorem_MAIN__LipschitzManifoldCase}
    \\
    $f(\xxx)\subseteq \mathcal{M}$,  $\mathcal{M}$: $d_{\mathcal{M}}$-dim. Manifold
    & 
    $
    \mathcal{O}\big(
            d
        +
            \epsilon_Q^{-d_{M}}
            \omega_f(|\xxx|)^{d_{M}}
    \big)
    $
    & 
    -
    & 
    $
    \mathcal{O}\big(
        \epsilon_Q^{-d_{M}}
        \omega_f(|\xxx|)^{d_{M}}
    \big)
    $
    &
    Corollary~\ref{cor_embedded_manifold}
    \\
    $\dim(f(\xxx))\leq d_{\mathcal{M}}$ 
    & 
    $
    \mathcal{O}\big(
    d + 
    \lceil \epsilon_Q^{-d_{\mathcal{M}}}
    \omega_f(|\xxx|)
    ^{d_{\mathcal{M}}}\rceil
    \big)
    $
    &
    -
    &
    $\mathcal{O}\big(
    \epsilon_Q^{-d_{\mathcal{M}}}
    \omega_f(|\xxx|)
    ^{d_{\mathcal{M}}}
    \big)$
    &
    Theorem~\ref{defn_Assouad_dimension}
    \\
    \bottomrule
    \end{tabular}
    }
    \caption{Narrower Probabilistic Transformers via Geometric Priors.}
    \label{tab_WidthControl}
\end{table}

We formalize low-dimensionality using the fractal dimension of \cite{Assouad_OG_Thesis_1979}, which is well-suited to general metric spaces.  Let $\mathcal{M}\subseteq \mathcal{P}_1(\rr^d)$ be compact and suppose that there exists constants $C_{\mathcal{M}},d_{\mathcal{M}}>0$ such that for every $\pp\in \mathcal{M}$ and every $0<r\leq |\mathcal{M}|$ the number of balls in $\mathcal{W}_1(\rrd)$ of radius at-most $r$ required to cover set $\mathcal{M}\cap B_{\mathcal{W}_1(\rr^d)}(\pp,r)$ is at-most 
\begin{equation}
\label{eq_homogenity}
    C_{\mathcal{M}}
    \big(
    r^{-1}
    |\mathcal{M}|
    \big)^{d_{\mathcal{M}}}
.
\end{equation}
\begin{definition}
[Assouad Dimension]
\label{defn_Assouad_dimension}
   The smallest $d_{\mathcal{M}}$ for which~\eqref{eq_homogenity} holds is $\mathcal{M}$'s Assouad dimension.  
\end{definition}

\begin{theorem}[Universal Approximation - Image of Finite Assouad Dimension]
\label{theorem_finite_Assouad_Dimension}
    Consider the setting of Theorem~\ref{theorem_MAIN__GeneralCase}.  Let $\mathcal{M}\subseteq \mathcal{P}_1(\rrD)$ have Assouad dimension $d_{\xxx}>0$ and let $C_{\mathcal{M}}$ be such that~\eqref{eq_homogenity} holds.  
    Then the conclusion of Theorem~\ref{theorem_MAIN__GeneralCase} holds but with $
        N
    \leq 
        C_{\mmm}\epsilon_Q^{-d_{\mathcal{M}}}|f(\xxx)|^{d_{\mathcal{M}}}
    $.  
\end{theorem}

\begin{example}[Embedding of a Compact Manifold into $\mathcal{P}_1(\rr)$]
Suppose that $M$ is a compact manifold of dimension $d_{M}\in \nn_+$.  The Whitney Embedding Theorem guarantees that there is a smooth injective map $F_1:M\rightarrow \rr^{2d_{M}+1}$ with a smooth inverse on its image.  In particular, $F_1(M)$ is contained in some $2d_M+1$-dimensional closed Euclidean ball $\overline{B(0,r)}$, for a sufficiently large radius $r>0$.  Consequentially, the quantity $C\eqdef \max_{x\in \xxx}\,\max\{\|\nabla F_1(x)\|,\|\nabla (F_1)^{-1}\circ F(x)\|\}$ is finite; whence, $F_1$ is a $C$-bi-Lipschitz map.  The construction in \citep[Example 5.5]{BenoitBertrand_2012_JTopAnal} implies that there is some $b\in \rr^{2d_{M}+1}$ such that the map $F_2:\overline{B(0,r)}\mapsto \mathcal{P}_1(\rr)$ given by
\[
    x \mapsto \frac1{2d_{M}+1} \sum_{i=1}^{2d_{M}+1}\,
    \delta_{\sqrt{d}x_i + b_i}
\]
is an isometric embedding.  Therefore, the composing map $F\eqdef F_2\circ F_1$ given by
\begin{equation}
\label{eq_Klockner_biLipschtiz_embedding}
    F(x)
        \eqdef 
    \frac1{2d_{M}+1} \sum_{i=1}^{2d_{M}+1}\,
        \delta_{\tilde{F}_1(x)}
\end{equation}
where $\tilde{F}_1(x)\eqdef \sqrt{d_{M}} F_1(x) + b$, 
is a $C$-bilipschitz embedding of $M$ into $\mathcal{P}_1(\rrd)$.  In Particular, \citep[Lemma 9.3]{Robinson_2009} that $\mathcal{M}\eqdef F(M)$ has Assouad dimension at-most $d_M$, whence Theorem~\ref{theorem_finite_Assouad_Dimension} applies.
\end{example}
\begin{corollary}[Universal Approximation - Image In An Embedded Compact Smooth Manifold]
\label{cor_embedded_manifold}
    Consider the setting of Theorem~\ref{theorem_MAIN__GeneralCase}, let $\mathcal{M}$ be a compact manifold, let $F:M\rightarrow \mathcal{P}_1(\rr^D)$ be the bi-Lipschitz embedding defined in~\eqref{eq_Klockner_biLipschtiz_embedding}, set $\mathcal{M}\eqdef F(M)$, and suppose that $f:\xxx\rightarrow \mathcal{P}_1(\rr^D)$ is such that $f(\xxx)\subseteq \mathcal{M}$.  Then, the conclusion of Theorem~\ref{theorem_MAIN__GeneralCase} holds but with $N$ of the order $ \mathcal{O}\big(
    \epsilon_Q^{-d_{M}}|f(\xxx)|^{d_{M}}
    \big)
    $.  
\end{corollary}

Another example of a low-dimensional structure containing $f(\xxx)$ occurs when $f$ maps $\xxx$ into some (possibly unknown) regular parametric family of probability measures $\mathcal{M}\eqdef \{\pp_{\theta}\}_{\theta \in [0,1]^P}$ contained in $\mathcal{P}_1(\rr^D)$; where $P\in \nn_+$.  Here, regularity means that the map $\theta \mapsto \pp_{\theta}$, from the parameter space $[0,1]^P$ to the mean of each member of the parametric family $\mathcal{M}$, is bi-Lipschitz.

\begin{theorem}[Universal Approximation: Lipschitz Manifold - via Barycentric Injectivity]
\label{theorem_MAIN__LipschitzManifoldCase}
Consider the setting of Theorem~\ref{theorem_MAIN__GeneralCase}.  
Let $D,P\in \nn_+$, $p:\rr^{D}\times \rr^P\rightarrow \rr^{d}$ be $L$-Lipschitz, $\pp\in \mathcal{P}_{2}(\rrD)$, and define $\mathcal{M}\eqdef \{\pp_{\theta}\eqdef {p_{\theta}}_{\#} \pp\}_{\theta \in [0,1]^P}$.  If $f(\xxx)\subseteq \mmm$ and if there is a $C>0$ such that the following ``bi-Lipschitz barycenter condition'' holds: for every $\theta_1,\theta_2 \in [0,1]^p$ it holds that
\begin{equation}
\label{eq_biLipshitz_condition}
        \|\theta_1-\theta_2\|
    \leq 
        C
        \big\|
            \mathbb{E}_{Y_1\sim \pp_{\theta_1}}[Y_1]
                -
            \mathbb{E}_{Y_2\sim \pp_{\theta_2}}[Y_2]
        \big\|
\end{equation}
then, the conclusion of Theorem~\ref{theorem_MAIN__GeneralCase} holds with
\[
N \eqdef 
\big\lceil
    \epsilon_Q^{-p}
    \,
    \big(
        2^p\max\{L,C\}^{2p} L^p\sqrt{p}^{p}
    \big)
\big\rceil
.
\]
Furthermore, the map $\theta \mapsto \pp_{\theta}$ is a Lipschitz homeomorphism onto its image with Lipschitz inverse.  
\end{theorem}

\section{Circumventing the Curse of Dimensionality}
\label{ss_no_curses}
One reason classical neural approaches have attained such popularity is that they bypass the curse of dimensionality on several learning tasks.  We illustrate two cases in which probabilistic transformers circumvent dimensionality's curse.  

Our first approach builds functions in $C(\xxx,\ppp[\yyy][1])$ that are efficiently approximable by an extension of the PT model.  This extension of the PT model leverages multiple probabilistic attention mechanisms.  The constructed class of functions is built by piecing together maps that feedforward neural networks can efficiently approximate (see \cite{barron1993universal,Suzuki2018ReLUBesov,chen2019efficient,SchmidtHieberAnnStat2020,lu2020deep,FloR2021}).  The second method constructs subsets of $\xxx$ on which the curse of dimensionality can always be broken when approximating any given target function $f\in C(\xxx,\ppp[\yyy][1])$.  

\subsection{Avoiding the Curse of Dimensionality by Lifting DNN Approximation Classes}
\label{sss_Functorial_efficiency}

Paralleling \cite{gribonval2019approximation}, we consider the following class of functions that deep feedforward networks can efficiently approximate but which also have uniformly bounded outputs.

\begin{definition}[{DNN Approximation Classes: $\mathcal{A}^{\sigma,r,S}(\rr^d,\rr^N)$}]
\label{defn_approximation_class_DNN}
Let $\sigma\in C(\rr)$, $r>0$ be rate parameter, $S>0$ a growth parameter, and $d,N\in \nn_+$.  
A function $f:\rr^d\rightarrow \rr^N$ belongs to the DNN approximation class $\mathcal{A}^{\sigma,r}(\rr^d,\rr^N)$ if there is a constant $C>0$ such that for every $x\in \rr^d$, every $R>0$, and every $\epsilon>0$ there is a DNN $\hat{f}\in \NN[d,N][\sigma]$ satisfying:
\begin{enumerate}
    \item \textbf{Approximable:} $
    \max_{u\in \rr^d,\, \|u-x\|\leq R}\,
    \big\|
        f(u)
            -
        \hat{f}(u)
    \big\|
        <
    \epsilon$,
    \item \textbf{$S$-Boundedness}: $\max_{x\in \rr^d}\, \|f(x)\|\leq S$,
    \item \textbf{Efficient Approximation Rate}: $\hat{f}$ has a realization of depth and width at-most $C R^r \epsilon^{-r}$.
\end{enumerate}
\end{definition}
Our construction is motivated by the following observation.  Fix $d,N\in \nn_+$, $\sigma \in C(\rr)$ and $r>0$.  The map 
$y\mapsto \delta_y$ is an isometric embedding of $\rr^D$ in $\mathcal{P}_1(\rr^D)$.  Suppose that $f$ belongs to the DNN approximation class $\mathcal{A}^{\sigma,r,S}(\rr^d,\rr^N)$ and consider the probability measure-valued map
\begin{equation}
    \label{eq_Markov_TowerMotivation}
    f(x) = \delta_{g(x)}
    .
\end{equation}
Since $f$ belongs to the DNN approximation class $\mathcal{A}^{\sigma,r,S}(\rr^d,\rr^N)$, then it can be uniformly approximated to arbitrary precision by a map of the form $x\mapsto \delta_{\hat{f}(x)}$ where $\hat{f}\in \NN[d,N][\sigma]$ with a realization of depth and width at-most $C R^r \epsilon^{-r}$.  Iterating this construction together with convex combinations yields the following class of functions, reminiscent of \cite{FloR2021}.   

\begin{definition}[Tree-Class]
\label{defn_tree_class}
Fix a valency parameter $V\in \nn_+$ with $V\geq 2$, a rate $r>0$, a growth parameter $S>0$, an activation function $\sigma \in C(\rr)$, a feature map $\varphi:\xxx\rightarrow \rr^d$ satisfying condition~\ref{cond_feature}, and a dimension $D\in \nn_+$.  
A function $f:\xxx\rightarrow \mathcal{P}_1(\rr^D)$ belongs to the tree-class of height $0$ and valency $V$, denoted by $\mathcal{T}_{0:V}^{\varphi,r,\sigma,S}(\rr^d,\mathcal{P}_1(\rr^D))$, if it is of the form
\[
    f(x) 
        \eqdef 
    \delta_{g\circ \varphi(x)}
    ,
\]
where $g\in \mathcal{A}^{\sigma,r,S}(\rr^d,\rr^D)$.  
Let $T\in \nn_+$.  A function $f:\xxx\rightarrow \mathcal{P}_1(\rr^D)$ belongs to the tree-class of height $T$ and valency $V$, denoted by $\mathcal{T}_{T:V}^{\varphi,r,\sigma,S}(\rr^d,\mathcal{P}_1(\rr^D))$, if it is of the form
\[
    f
        \eqdef 
    \sum_{v=1}^V\,
        [\operatorname{Softmax}_V\circ w\circ \varphi(\cdot)]_v
        f_v(\cdot)
    ,
\]
where $
f_1,\dots,f_V
    \in
\mathcal{T}_{T-1:V}^{\varphi,r,\sigma,S}(\rr^d,\mathcal{P}_1(\rr^D))
$ and $w\in \mathcal{A}^{r,\sigma,S}(\rr^d,\rr^V)$.
\end{definition}

Mimicking the structure of tree-class functions, we may extend the probabilistic transformer, using not one but several probabilistic attention mechanisms.  The resulting \textit{``tall probabilistic transformer''} models is illustrated in  Figure~\ref{fig_ProbalisticTransformer__DeepAttention}. 

\begin{figure}[H]
\centering
\includegraphics[width=0.6\linewidth]{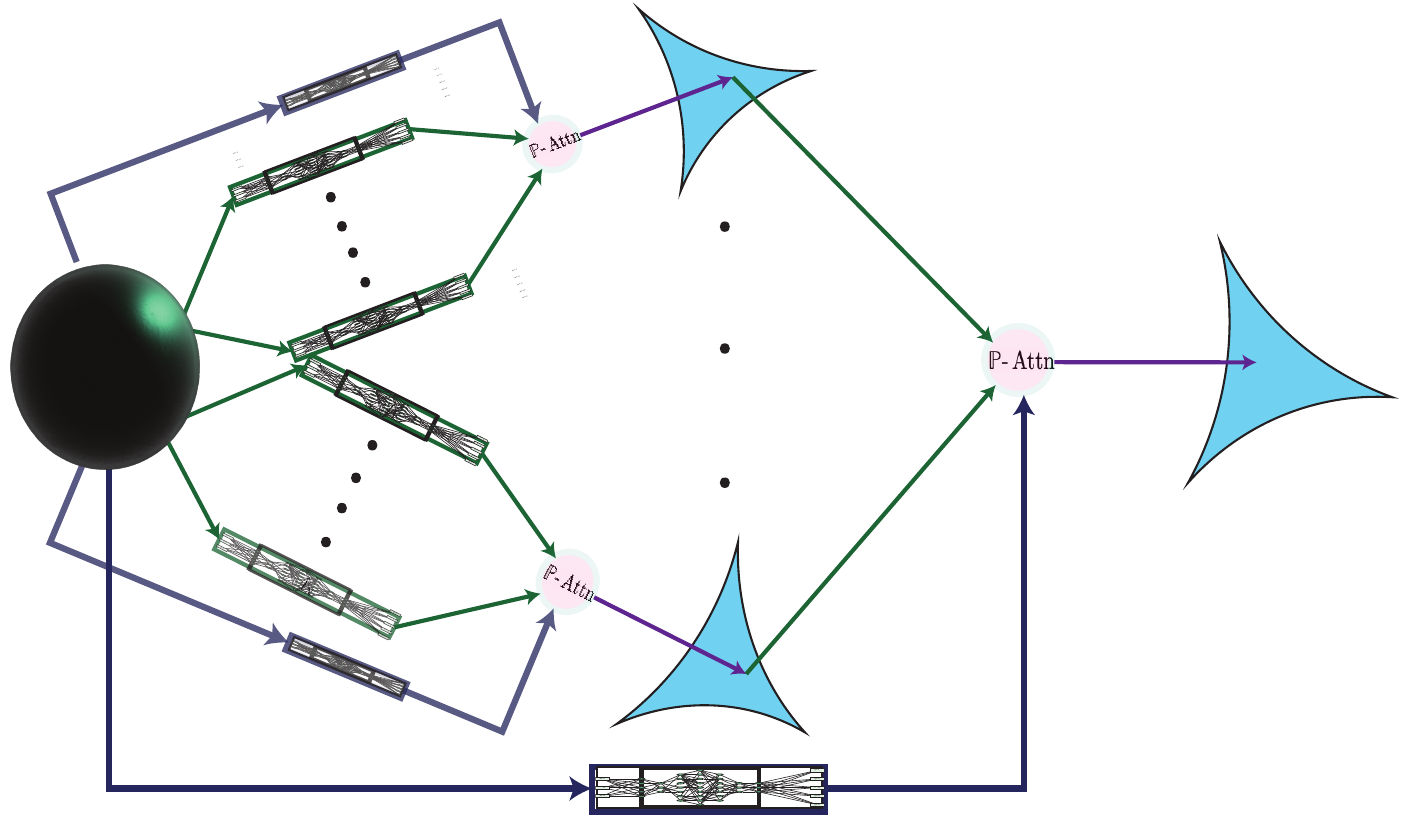}
\caption{Tall Probabilistic Transformer with height $2$.}
\label{fig_ProbalisticTransformer__DeepAttention}
\end{figure}
\begin{definition}[Tall Probabilistic Transformer]
\label{defn_tall_probabilistic_transformer}
Let $\varphi:\xxx\rightarrow \rr^d$ be a feature map satisfying condition~\ref{cond_feature}, fix an activation function $\sigma \in C(\rr)$, a dimension $D\in \nn_+$, a valency parameter $V\in \nn_+$.  
Define $\NN[\varphi:\yyy:0,V][\sigma:\star]\eqdef \NN[\varphi:\yyy][\sigma:\star]$.  
For every $T\in \nn_+$, a probabilistic transformers of height $T$ and valency $V$ is a map $f:\rr^d\rightarrow \mathcal{P}_1(\yyy)$ with representation
\[
    \hat{T}
        \eqdef
    \sum_{v=1}^V\,
        [\operatorname{Softmax}_V\circ w\circ \varphi(\cdot)]_v\,
        \hat{T}_v
\]
where $
    \hat{T}_1,
    \dots,
    \hat{T}_V
    \in
\NN[\varphi:\yyy:T-1,V][\sigma:\star]
$ and $w\in \NN[d,N][\sigma]$.  
We denote the set of all probabilistic transformers of height $T$ and valency $V$ by $\NN[\varphi:\yyy:T,V][\sigma:\star]$.  
\end{definition}
Tall probabilistic transformers can efficiently approximate tree-class functions.  

\begin{theorem}[Tall Probabilistic Transforms Efficiently Approximate Tree-Class Functions]
\label{theorem_efficient_approximation}\hfill\\
Let $\xxx$ be a compact metric space and $d,D\in \nn_+$.
Fix a valency parameter $V\in \nn_+$, $V\geq 2$, a rate $r>0$, a growth parameter $S>0$, a height parameter $T$, an activation function $\sigma \in C(\rr)$, and a feature map $\varphi:\xxx\rightarrow \rr^d$ satisfying Condition~\ref{cond_feature}.  
For every $\epsilon>0$ and each $f\in \mathcal{T}_{T:V}^{\varphi,r,\sigma,S}$
there exists a tall probabilistic transformer $\hat{T}\in \NN[\varphi:\yyy:T,V][\sigma:\star]$ 
satisfying
\begin{enumerate}
    \item \textbf{Approximation:} 
    $
    \max_{x\in \xxx}\, 
    \mathcal{W}_1\big(
        f(x)
            ,
        \hat{T}(x)
    \big)
    <\epsilon,
    $
    \item \textbf{Efficient Parameter Count}: The number of parameters determining $\hat{T}$ is at-most
    \[
    \mathscr{O}\left(
        \epsilon^{-2r}
        V^{T(1-2r)}
        \big(
            2^{-1}C_{\varphi:U}^{2} 
            |\xxx|^{2\alpha } 
        \big)^r
    \right),
    \]
where, $\mathscr{O}$ hides a positive constant independent of $\epsilon$, $d$, $D$, $|\xxx|$, $\alpha$, and of $r$.  
\end{enumerate}
\end{theorem}

\subsection{Breaking the Curse of Dimensionality by Localization}
\label{ss_MAIN_Localization}
The following result reflects the machine learning folklore that neural network models predict well \textit{``near a training dataset"}.  We make the convention that, for any $x\in \xxx$, $B_{\xxx}(x,\infty)\eqdef \xxx$.  
\begin{definition}[Localization about a Training-Set]\label{defn_localization}
Let $\xx\subseteq \xxx$ be a non-empty finite set.  
Given $\bar{x}\in \xx$ and $0\leq \delta<\infty$, $0<R\leq \infty$, define the \textit{$(R,\delta)$-localization about $\xx$ relative to $\bar{x}\in \xx$} as
\begin{equation}
    \mathbb{X}_x^{R,\delta}
        \eqdef 
    \left\{
x \in \xxx:\, d_{\xxx}\left(x,\xx\cap \overline{B_{\xxx}\left(
\bar{x},\delta
\right)}\right)\leq R
\right\}
\label{eq_localization}
.
\end{equation}
\end{definition}

Our approach reinterprets the recently proposed notion of \textit{``controlled universal approximation}", introduced in \cite{paponkratsios2021quantitative}, in which the user restricts the maximum size of the given compact subset of inputs in $\xxx$ for which a neural network approximates a given target function.  In the context of \cite{paponkratsios2021quantitative}, the size restriction on that compact allows one to bypass topological obstructions present in approximating continuous functions between Riemannian manifolds.  Instead, in this paper's context, the size and structure restrictions on the compact subset of inputs allow us to bypass the curse of dimensionality.   

\begin{corollary}[Localization Avoids the Curse of Dimensionality]
\label{cor_localization}
Consider the setting of Theorem~\ref{theorem_MAIN__GeneralCase}, suppose additionally suppose that $f:\xxx\rightarrow \mathcal{P}_1(\yyy)$ is $\beta$-H\"{o}lder continuous for some $\beta \in (0,1]$, and that for each $x\in \xxx$, $f(x)$ is lower $(C,r)$-Ahlfors regular for some $C>0$.  
Let $\mathbb{X}\subseteq \xxx$ be any non-empty finite set and fix some $x\in \mathbb{X}$.  
There is a constant $C_f>0$, depending only on $f$, such that for any $r>0$ if we define
\[
        \delta 
            \eqdef 
        C_f\,
        \epsilon^{1/\beta}
    \mbox{ and }
        R
            \eqdef 
        \min\left\{
            \epsilon^{2/\beta - 2r/ad}
        ,
            \sup\{r>0:\, 
                \# \mathbb{X} \cap B_{\xxx}(x,r) \leq \epsilon^{-r/2}
            \}
        \right\}
,
\]
then, there exist an $\hat{T}\in \NN[\varphi,D][\operatorname{ReLU}:\star]$ satisfying the uniform estimate holds on $\mathbb{X}^{R,\delta}_x$
\[
    \sup_{x\in \mathbb{X}^{R,\delta}_x}\,
        \mathcal{W}_1\big(
            f(x)
        ,   
            \hat{T}(x)
        \big)
    <
        \epsilon
    ,
\]
and such that:
\begin{enumerate}[(i)]
    \item $Q\in \mathscr{O}(\epsilon^{-r})$,
    \item $N\leq \epsilon^{-r}$,
    \item $\hat{T}$ has depth and width of the order
    $
        \mathscr{O}\big(
             \epsilon^{-r}
        \big)
    .
    $
\end{enumerate}
\end{corollary}
\section{Applications}\label{s_Applications}
Applications of the paper's main theoretical results are now considered.  The focus is on problems at the interface of approximation theory and applied probability.  A numerical illustration accompanies each application. 

We begin by describing the benchmarks and evaluation metrics used in each of our implementations and the heuristic used to train the PT model.  Implementation details are relegated to Appendix~\ref{a_NumericalScheme}.   

\paragraph{Description of Supporting Numerical Illustrations}  

In these experiments, we aim to learn a fixed ``target function" $f\in C(\xxx,\ppp[\yyy][1])$ from a set of training data.  In each experiment, we are given a (non-empty and finite) set of training inputs $\xx\subseteq \rrd$ and training outputs $\{\mu_x\}_{x \in \xx}$.  Here, $\mu_x\eqdef \frac{1}{S}\sum_{s=1}^S \delta_{X_x^s}$ where, $X_x^1,\dots,X_x^S$ are i.i.d.\ random samples drawn from some the probability measure $f(x)$, where $S=10^3$.  The ``testing" dataset is used solely for evaluation.  It consists of the $100$ pairs $(x^{\text{test}},\mu_x^{\text{test}})$ where each $x$ is sampled uniformly from $\cup_{x \in \xx} \overline{\operatorname{Ball}(x,10^{-1})}\subset \rrd$.  

\paragraph{Benchmark:} The performance of the proposed universal model $\NN[1_{\rrd},\ddd][\operatorname{ReLU},\star]$ is benchmarked against four types of learning models.  The first type consists of ``classical learning models," which are designed to predict $\rrD$-valued outputs.  These models are trained on the empirical means $\frac{1}{S}\sum_{s=1}^S{X_x^s}$.  The benchmark models include the Elastic-Net regressor of \cite{zou2005regularization} (ENET), kernel ridge regressor (kRidge), gradient boosted (see \cite{friedman2002stochastic}) random forest regressor (GBRF), and a DNN.  Each of these models maps $\rrd$ to $\rrD$.  

The second type of model maps inputs in $\rrd$ to Gaussian probability measures on $\mathbb{R}^D$.  Two such models are considered.  The first is a Gaussian process regressor (GPR).  The second model is an instance of the geometric deep network (GDN) model of \citep[Section 4.4.2]{paponkratsios2021quantitative}, which in this context, sends each input in $\rrd$ to a pair of a ``mean" vector in $\rrD$ and a $D\times D$-symmetric positive-definite ``covariance" matrix.  This transformation is implemented by deep feedforward networks in $\NN[d,d+2^{-1}d(d-1)]$ post-composed by the map $\rr[d+2^{-1}d(d-1)]\ni (x_1,x_2)\to (x_1,x_2x_2^{\top})\in \rrd\times P_{d,d}$; where $P_{d,d}$ is the set of $d\times d$-symmetric positive definite matrices.  

Analogously to \cite{lakshminarayanan2016simple}, the GDN is trained as follows.  First, for each $\mu_x$ in the training dataset, we identify a pair $(\mu_x^{MLE},\sigma_x^{MLE})$ via maximum likelihood estimation.  Then the GDN model is trained to learn the map $x\mapsto (\mu_x^{MLE},\sigma_x^{MLE})$.  

The third type of model is the mixture density network (MDN) of \cite{bishop1994mixture}.  These are the first neural network models introduced with the intent to learn conditional expectations, but no approximation theoretic foundation is available for MDNs.  These are trained similarly to the GDNs.  First, for each $x$ in the training set, the parameters of a Gaussian mixture model $\hat{\nu}_x$ are inferred using the EM algorithm.  The MDN is trained to learn a map sending each $\xx\ni x \mapsto \hat{\nu}_x \in \ppp[\yyy][2]$.  

The last benchmark is an \textit{``Oracle"} (MC-Oracle).  Here, MC-Oracle sends any $x$ in the input space to the Monte Carlo estimate of $f(x)$ obtained by independently sampling $S$ times from $f(x)$.  MC-Oracle represents an inaccessible gold standard since it requires knowledge of the unknown measure-valued function $f$.  We emphasize that such samples are unavailable to the user outside the training set.  We highlight that the predictive accuracy of MC-Oracle cannot generally be beaten but only matched.  However, the MC-Oracle model cannot be deployed in practice since it requires complete knowledge of the target function $f$.  Nevertheless, the proposed models often need less time than the MC-Oracle to generate predictions.  

All hyperparameters of the involved models are cross-validated on a large grid, and the parameters are optimized using the ADAM stochastic gradient-descent method of \cite{kingma2014adam}.  The code for all implementations and further implementation details is available at \cite{AK2021Git}. 

We use two quality metrics in each experiment to evaluate and compare each machine learning model's prediction precision and complexity/parsimony.  These quality metrics are defined as follows.  

\paragraph{Evaluation Metrics:} The prediction quality metrics compare the model's predicted Wasserstein distance to the MC-Oracle model (W1) and the difference in the predicted mean to the mean predicted by the MC-Oracle model (M).  These error metrics are first averaged over each data point in the training and testing sets.  We then report the worst-case approximation quality between these two evaluation sets.  Confidence intervals for the error distributions of W1 and M are constructed using the bias-corrected and accelerated bootstrapped confidence intervals of \cite{EfronConfidenceIntervals}.  The respective lower and upper confidence intervals about (W1) are denoted by (W1-95L) and (W1-95), and those about (M) are similarly denoted by (M-95L) and (M-95R).  All values less than $10^{-20}$ are set to $0$.  

\paragraph{Parsimony Metrics:} These later quantities are computed directly using the (practically inaccessible) MC-Oracle method.  The model complexity metrics are the number of parameters (N\_Par), the total training time (Train\_Time), and the ratio of the time required to generate predictions for every instance in the testing set over the entire time the MC-Oracle model needs to do the same ($\frac{\text{Test-Time}}{\text{MC-Oracle Test-Time}}$).  
            
\subsection{Problem 1: Generic Regular Conditional Distributions}\label{s_applications}
We first show that our architecture can approximate regular conditional distributions to arbitrary precision with arbitrarily high probability under mild integrability conditions.  Let $(\Omega,\fff,\pp)$ be a probability space, let $X:\Omega\rightarrow \xxx$ be an $\fff$-measurable random variable, and let $Y:\Omega\rightarrow \xxx$.  Then, by \citep[Theorem 6.3]{Kallenberg2002Foundations} there exits a $\pp$-a.s. uniquely determined \textit{probability kernel}
$\pp_{\cdot}^X:\xxx\times \bbb(\yyy)\rightarrow [0,\infty]$ satisfying:
\begin{equation}
    \pp_{x}^Y(B)
        =
    \pp\left(
        Y \in B
            \mid
        X = x
    \right),
    \label{eq_quick_overview_disintegration}
\end{equation}
for every Borel subset $B\subseteq \yyy$.  Thus, $x\mapsto \pp(Y\in \cdot|X=x)$ is a Borel-measurable function from $(\xxx,\bbb(\xxx))$ to $\left(\ppp(\yyy),\bbb(\ppp(\yyy))\right)$, which is called the \textit{regular conditional distribution} of $Y$ given $X$.  We point the interested reader to \citep[Pages 106-107]{Kallenberg2002Foundations} for further details).  We denote the push-forward measure of $\pp$ by $X$ using $X_{\#}\pp$.  The following result is the rigorous version of Informal Theorem~\ref{thrm_UAT_Measurable_Variant}.  

\begin{corollary}[PAC-Universal Regular Conditional Distributions]
\label{cor_Universal_Regular_Conditional_Distributions_RCP}
Consider the setup of~\eqref{eq_quick_overview_disintegration} and suppose that for every $x\in \xxx$ it holds that $\pp_x^Y\in \mathcal{P}_p(\rr^D)$ for some $p>1$.  
For every $0<\epsilon<1$ there exists a probabilistic transformer $\hat{T}:\xxx\rightarrow \ppp[\rr^D][1]$ satisfying:
$$
X_{\#}\pp\Big(
    W_1\left(
        \pp^Y_x
            ,
        \hat{T}(x)
    \right)
        \leq 
    \epsilon
\Big)
    \geq 1- \epsilon
.
$$
\end{corollary}
Once the probabilistic transformer has approximated the RCD of $Y$ given $X$, we can directly approximate any the conditional expectation $\ee[f(Y, X)|X=x]$. 
\begin{corollary}[Generic Conditional Expectations]\label{cor_Universal_Disintegration}
Consider the setting of Corollary~\ref{cor_Universal_Regular_Conditional_Distributions_RCP}.  For every Borel-measurable function $f:\yyy\times \xxx \rightarrow \rr$ which is uniformly Lipschitz
\footnote{
By uniformly Lipschitz, we mean that there is some $L>0$ satisfying
$
\left|
    f(y_1,x)
        -
    f(y_2,x)
\right|\leq L\, d_{\yyy}\left(y_1,y_2\right)
,
$
for every $x \in \xxx$ and every $y_1,y_2 \in \yyy$.  
} 
in its first argument and uniformly bounded
\footnote{
By uniformly bounded, we mean that there is some $m>0$ and some $(y^{\star},x^{\star})\in \xxx$ such that:
$\sup_{(t,x)\in \yyy\times  \xxx}|f(y,x)|\leq m |f(y^{\star},x^{\star})|<\infty$.  Note that if both $\xxx$ and $\yyy$ are compact, then if $f$ is jointly continuous, then it is necessarily uniformly bounded.  
} it holds that:
\begin{equation*}
\left|
    \ee\left[
        f(Y,X)
        \mid X = x
    \right]
    -
    \int_{y \in \yyy} 
        f(y,x)
        \big(\hat{T}(x)\big)(dy)
\right|\leq M\epsilon
,
\end{equation*}
$\pp$-a.s. for every $x \in K_{\epsilon}$; where, $M>0$ is a constant independent of $\epsilon$, $K_{\epsilon}$, $N$, and of $\{\hat{\mu}_n\}_{n=1}^{N}$.  
\end{corollary}
\subsubsection{Experimental Ablation of Corollary~\ref{cor_Universal_Regular_Conditional_Distributions_RCP} and of~Theorem~\ref{theorem_MAIN__GeneralCase}: Conditional Evolution of an SDE}
Fix $r,d\in \nn_+$ and consider a $r$-dimensional Brownian motion $(B_t)_{t\geq 0}$ on a probability space $(\Omega,\fff,\pp)$ equipped with the natural filtration $(\sigma\{B_s\}_{0\leq s\leq t})_t$ generated by $(B_t)_{t\geq 0}$.  For each $x \in \rrd$, let $(X_t^x)_t$ denote the unique continuous strong solution to the stochastic differential equation:
\begin{equation}
    \label{eq_SDE_Definition}
    X_t^x = x + \int_0^t \alpha(s,X_s^x)ds + \int_0^t \beta(s,X_s^x)dB_s
    ;
\end{equation}
where $x \in \rrD$, under the assumption that $\alpha\in 
C\left([0,T]\times \rrd,\rrd\right)$, $\beta\in C\left(
[0,T]\times \rrd,\rr[r\times d]
\right)$ satisfy the usual conditions, (i) and (ii) listed below, where $\rr[r\times d]$ is identified with the space of $r\times d$-matrices normed by the Fr\"{o}benius norm $\|\cdot\|_F$; which exists (for example by \citep[Theorem 8.7]{DaPratoIntroStochasticCalculusandMalliavin2014ThirdEdidtion}) under the following conditions:
there exists some $M>0$ such that for every $x_1,x_2\in \rrd$ and every $t \in [0,T]$ the following holds
\begin{enumerate}[(i)]
    \item $\left(\|\alpha(t,x_1)-\alpha(t,x_2)\|^2 + \|\beta(t,x_1)- \beta(t,x_2)\|^2_F\right)^{\frac1{2}} \leq M\|x_1-x_2\|$,
    \item $\left(
    \|\alpha(t,x_1)\|^2
        +
    \|\beta(t,x_1)\|^2_F
    \right)^{\frac1{2}}\leq M \left(1+\|x_1\|\right)^{\frac1{2}}$.
\end{enumerate}
Furthermore, conditions (i) and (ii) together with \citep[Propositions 8.15 and 8.16]{DaPratoIntroStochasticCalculusandMalliavin2014ThirdEdidtion}
and \citep[Theorem 6.9]{VillaniOptTrans} imply that the map from $[0,\infty)\times \rr^d$ to the subset $\ppp[\rrd][2]$ of the $1$-Wasserstein space $(\mathcal{P}_1(\rrd),\mathcal{W}_1)$ defined by
\begin{equation}
\label{eq_conditional_expectation_approximation}
    (t,x)
        \mapsto 
    \pp\left(X_t^x\in \cdot \right)
\end{equation}
is a continuous function.  Thus, our results guarantee that the map of~\eqref{eq_conditional_expectation_approximation} can be approximated uniformly by the probabilistic transformer.  
Let us examine the implications of the rates in Table~\ref{tab_general_approximation_rates} using a numerical implementation of the probabilistic transformer model, trained to approximate~\eqref{eq_conditional_expectation_approximation}.  Set $d=1,\xxx=[0,1]\times [-1,1]$, and $\varphi=1_{\rr}$.  

We first uniformly sample a finite set of training inputs $\mathbb{X}_{\operatorname{train}}$ in $\xxx$.  For each $(t,x)\in \mathbb{X}_{\operatorname{train}}$ we simulate the random variable $X_t^x$ by first discretizing the stochastic differential equation~\eqref{eq_SDE_Definition} using the Euler–Maruyama method and then approximating the $\mathbb{E}[X_t^x]$ and $\mathbb{P}(X_t^x \in \cdot)$ using standard Monte Carlo sampling.  A probabilistic transformer of a fixed depth and width is then trained to minimize the objective function
\[
    \sum_{(t,x)\in \mathbb{X}_{\operatorname{train}}}
    \,
    \mathcal{W}_1\Big(
        \hat{T}(t,x)
            ,
        \hat{\mathbb{P}}(X_t^x \in \cdot)
    \Big)
,
\]
where $\hat{\mathbb{P}}(X_t^x \in \cdot)$ denotes our approximation of $\mathbb{P}(X_t^x \in \cdot)$, as outlined above.  We then uniformly sample a distinct finite subset $\mathbb{X}_{\operatorname{test}} \subseteq \xxx$, all of whose elements do not belong to $\mathbb{X}_{\operatorname{train}}$.  For each $(t,x)\in \mathbb{X}_{\operatorname{test}}$ we then generate approximate $\mathbb{E}[X_t^x]$ and $\mathbb{P}(X_t^x \in \cdot)$ using the combined Euler–Maruyama method and Monte Carlo sampling approach above.  

The performance of the probabilistic transformer, which does not know the drift $(\alpha)$ and diffusion $(\beta)$ coefficients defining the stochastic differential equation~\eqref{eq_SDE_Definition}, is compared to the Euler–Maruyama method and Monte Carlo sampling ``Oracle method'' described above, on the testing dataset $\{\big((t,x),\hat{\mathbb{P}}(X_t^x \in \cdot)\big):(t,x)\in \mathbb{X}_{\operatorname{test}}\}$ using the following performance metrics
\[\resizebox{1\hsize}{!}{$
        \boldsymbol{W}_1^{\operatorname{test}}
            \eqdef 
        \sum_{(t,x)\in \mathbb{X}_{\operatorname{test}}}
            \,
            \mathcal{W}_1\Big(
                \hat{T}(t,x)
                    ,
                \hat{\mathbb{P}(X_t^x \in \cdot)}
            \Big)
    \mbox{ and }
        \boldsymbol{M}^{\operatorname{test}}
            \eqdef 
        \sum_{(t,x)\in \mathbb{X}_{\operatorname{test}}}
            \,
        \Big(
            \mathbb{E}_{Y\sim \hat{T}(t,x)}[Y]
                -
            \mathbb{E}_{\tilde{Y}\sim \hat{\mathbb{P}(X_t^x \in \cdot)}}[\tilde{Y}]
        \Big)^2
$}\]

The performance of the trained probabilistic transformer model for different specifications of the drift parameter ($\alpha$) and the diffusion parameter ($\beta$) are reported in Table~\ref{tab_Approximation_SDE} below.  The time required to train the probabilistic transformer model $T_{\operatorname{tr}}$ on the training set $\mathbb{X}_{\operatorname{train}}$ and the time required to evaluate its predictions on the test set $\mathbb{X}_{\operatorname{test}}$ are also reported.  

The rates in Table~\ref{tab_general_approximation_rates} are reflected by Table~\ref{tab_Approximation_SDE} since the probabilistic transformer produces a larger approximation error (relative to the performance metric $\boldsymbol{W}_1^{\operatorname{test}}$) for more complicated drift and diffusion specifications.  Less regular target functions require deeper and wider probabilistic transformers to be approximated with comparable accuracy to more regular target functions  
\begin{table}[H]
    \centering
    \caption{Test-Set Prediction of~\eqref{eq_SDE_Definition}}
	\ra{1.3}
    \resizebox{\columnwidth}{!}{
    \begin{tabular}{@{}lllrrr@{}}
			\toprule
			Drift $(\alpha)$ & Volatility $(\beta)$ & 
			$
			\boldsymbol{W}_1^{\operatorname{test}}
			$ &  $
			\boldsymbol{M}^{\operatorname{test}}
			$ &  $\text{T}_{\text{tr}}$ & $\text{T}_{\text{eval}}$
			\\
			\midrule
			$.1$ &  $.01$ 
			& 
			2.17E-05 &    1.00E-03 &              
			4.00E+01 &          6.42E-03 
			\\
			$.1x$ & $.1x$ & 
			1.54E-03 &    8.82E-04 &                    5.42E+01 &          4.33E-03 
			\\
			$\sin(t+x)-xe^{-t}$ & $(|\cos(t+x)|+e^{-t})x$ 
		& 2.34E-01 &  5.12E-03 &           
		5.50E+01 &        4.92E-03
		\\
			\bottomrule
		\end{tabular}
    }
    \label{tab_Approximation_SDE}
\end{table}
Next, we investigate the effect of dimensionality described in Table~\ref{tab_general_approximation_rates} on the probabilistic transformer's performance.  For this, we repeat the above experiment but replacing~\eqref{eq_SDE_Definition} with
\begin{equation}
    \label{eq_SDE_Definition__fractional}
    X_t^x = x + \int_0^t \alpha(s,X_s^x)ds + \int_0^t \beta(s,X_s^x)dB_s^H
    ,
\end{equation}
where $(B_t^H)_{t\geq 0}$ is a fractional Brownian motion with Hurst parameter $H \in (1/2,1)$.  As shown in \cite{CarmonaCoutin_1998_fBM_nonMarkov,DuncanDuncan__2000_StochCalfBM}, this means that the increments of the fractional Brownian motion $(B_t^H)_{t\geq 0}$ are auto-correlated, which implies that both $(B_t^H)_{t\geq 0}$ and $X_t^x$ are non-Markovian stochastic processes.  Whence $\mathbb{P}(X_t^x\in \cdot)$ depends on the entire (infinite) path of the process $X_t^x$ and not only on $(t,x)$.  The findings of the analogous numerical experiment, mutatis mutandis, are reported in Table~\ref{tab_Approximation_fSDE}.  
\begin{table}[H]
    \centering
    \caption{Test-Set Prediction of~\eqref{eq_SDE_Definition__fractional}}
	\ra{1.3}
    \resizebox{\columnwidth}{!}{
    \begin{tabular}{@{}lllrrr@{}}
			\toprule
			Drift $(\alpha)$ & Volatility $(\beta)$ & 
			$
			\boldsymbol{W}_1^{\operatorname{test}}
			$ &  $
			\boldsymbol{M}^{\operatorname{test}}
			$ &  $\text{T}_{\text{tr}}$ & $\text{T}_{\text{eval}}$
			\\
			\midrule
			$.1$ &  $.01$ &
			5.21E-02 &  1.56E-03 &          
			5.72E+01 &        3.75E-03
			\\
			$.1x$ &  $.1x$ 
			& 5.27E-02 &  1.47E-03 &          
			5.03E+01 &        4.42E-03
			\\
			$\sin(t+x)-xe^{-t}$ & $(|\cos(t+x)|+e^{-t})x$ &
			1.25E-01 &  4.32E-03 &           
			4.80E+01 &        4.28E-03 
			\\\bottomrule
		\end{tabular}
	}
    \label{tab_Approximation_fSDE}
\end{table}
In this experiment, the process $(X_t^x)_{t\ge 0}$ is non-Markovian; thus, the distribution $\mathbb{P}(X_t^x\in \cdot )$ depends on the entire path realized by the process $(X_s^x)_{0\le s<t}$.  Therefore, the drop in predictive accuracy from Table~\ref{tab_Approximation_SDE} to~\ref{tab_Approximation_fSDE} is likely due to the fact that the trained PT model only used on the realized value of $X_0$ and on the terminal time $t$ as inputs.  

\subsection{Problem 3: ``A Generic Expression of Epistemic Uncertainty"}
\label{ss_Problem_2}

Consider a family of learning models $\big\{\hat{f}_{\theta}:\, \theta \in \Theta\big\}\subseteq C(\xxx,\yyy)$ parameterized by a compact metric space $\Theta$ with metric $d_{\Theta}$; $\Theta$ is called the learning model's \textit{parameter space}.  Randomized algorithms are often used to select a learning model's parameters.  Examples include stochastic optimization algorithms \cite{kingma2014adam,Rosasco2020StochasticProximalDescent}, the randomization of a reservoir computer's hidden layers \cite{LukasJPLyudmilaRiskboundsReservoir2020JMLR}, randomized training of neural ODE-based models \cite{TeichmannCuchieroLarsson2020,herrera2021optimal}, or dropout algorithms used to improve a deep learner's generalizability \cite{srivastava2014dropout}.  The randomized algorithm defines a $\Theta$-valued random variable $\vartheta$ on some auxiliary probability space $(\Omega,\fff,\pp)$, and a learning model is chosen by sampling $\vartheta$.  

In what follows, we metrize the product space $\xxx\times \Theta$ via the metric $d_{\xxx\times \Theta}((x_1,\theta_1),(x_2,\theta_2))\eqdef d_{\xxx}(x_1,x_2)+d_{\Theta}(\theta_1,\theta_2)$.  The following regularity conditions are required.  
\begin{condition}\label{regularity_condition_randomization}
Assume that $\hat{f}$ and $\vartheta$ satisfy the following:
\begin{enumerate}[(i)]
    \item $\hat{f}:\xxx\times \Theta \ni (x,\theta)\mapsto \hat{f}_{\theta}(x)\in \yyy$ is uniformly continuous with modulus of continuity $\omega_{\hat{f}}$,
    \item There are $y_0\in \yyy$ and $1<q<\infty$ such that, for every $x \in \xxx$: $\sup_{\theta \in \Theta} \ee\left[d_{\yyy}(\hat{f}_{\theta}(x),y_0)^{q}\right]<\infty$.
\end{enumerate}
\end{condition}
\begin{proposition}[{Continuity of the Map $x\mapsto \pp (\hat{f}_{\vartheta}(x)\in \cdot )$}]
\label{prop_MOTIVATION_ex_MOTIVATION_CNTFUNCTIONS_2_Randomized_Parameters}
Suppose that condition~\ref{regularity_condition_randomization} holds.  The measure-valued function $\xxx\ni x \mapsto \pp\big(\hat{f}_{\vartheta}(x)\in \cdot \big)$ has image in $\ppp[\yyy][q]$ and it is uniformly continuous with modulus of continuity $\omega_{\hat{f}}$; i.e.\ for any $x_1,x_2 \in \xxx$ it holds that
\[
    W_1 \left(
    \pp(\hat{f}_{\vartheta}(x_1)\in\cdot)
    ,
    \pp(\hat{f}_{\vartheta}(x_2)\in\cdot)
    \right)
    \leq 
    \omega_{\hat{f}}\left(
    d_{\xxx}\left(x_1,x_2\right)
    \right)
.
\] 
\end{proposition}
\begin{example}
\label{ex_motivational_example_covered}
If $\xxx$ and $\Theta$ are compact, then\footnote{This follows from \citep[Tychonoff Product Theorem (Theorem 37.3)]{munkres2014topology} and by the fact that $d_{\xxx\times \Theta}$ metrizes the product topology on $\xxx\times \Theta$.} so is $\xxx\times \Theta$; moreover, in this case the Heine-Cantor Theorem \cite[Theorem 27.6]{munkres2014topology} implies that every $\hat{f}\in C(\xxx\times \Theta ,\yyy)$ is uniformly continuous.  
\end{example}
\begin{corollary}[Universal Epistemic Uncertainty Models]\label{cor_learnability}
Assume the setting of Proposition~\ref{prop_MOTIVATION_ex_MOTIVATION_CNTFUNCTIONS_2_Randomized_Parameters}, that $\xxx$ and $\yyy$ are compact, and that conditions~\ref{condi_KL} and~\ref{cond_feature} are met by $\sigma$ and $\varphi$, respectively.  Then, there is a probabilistic transformer $\hat{T}:\xxx \rightarrow \mathcal{P}_1(\yyy)$ satisfying:
\[
    \max_{x\in \xxx}\,
        \mathcal{W}_1\left( \pp\left(\hat{f}_{\vartheta}(x)\in \cdot \right),\hat{T}(x)\right)<\epsilon
.
\]
\end{corollary}
\subsubsection{Numerical Illustration: Learning Stochasticity from MC-Dropout}\label{ss_predicting_dropout}
We consider DNNs with parameters randomized by the \textit{MC-dropout} regularization algorithm of \cite{srivastava2014dropout}.  
In this experiment, the user is provided with an exogenous DNN $\tilde{f}\in C(\rrd,\rr)$ defined by:
\[
    \tilde{f}(x)=A_J f^{(J-1)}_x + b_J , \qquad 
    f^{(j)}_x\eqdef A_{j-1} f^{(j-1)}_x + b_{j},
    \qquad f^{(0)}_x \eqdef x.
\]
The weights and biases of the model are drawn randomly from a standard normal distribution, then fixed throughout the experiment.  

The MC-dropout algorithm randomly sets some of $\tilde{f}$'s parameters to $0$.  Formally, the user is given random matrices $B_1,\dots,B_J$ (of the same respective dimension as the $A_1,\dots,A_J$), where each $B_i$ is populated with i.i.d.\  Bernoulli entries.  For each input $x \in \rrd$ we then define
$$
Y_x = 
(B_J\odot A_J)f^{(J-1)}_x + b_J , \qquad 
f^{(j)}_x\eqdef (B_{j-1}\odot A_{j-1}) f^{(j-1)}_x + b_{j},
\qquad f^{(0)}_x \eqdef x;
$$
where $\odot$ denotes the Hadamard product defined on matrices of the same dimensions by $(B\odot A)_{i,j}\eqdef B_{i,j}A_{i,j}$.  In this experiment, each learning model's goal is to learn the probability measure-valued function $x \mapsto \pp(Y_x\in \cdot)\in \ppp[\rr][1]$. 
\begin{table}[H]
\centering
\caption{Quality Metrics; d:500, D:1, Depth:1, Width:5, Dropout rate: 0.1.}
\label{tabl_MC2}
\ra{1.3}
\resizebox{\columnwidth}{!}{
\begin{tabular}{@{}lrrllllrrr@{}}
\toprule
{} &      PT &  MC-Oracle &     ENET &   KRidge &     GBRF &      DNN &      GPR &      GDN &      MDN \\
\midrule
W1-95L                        & 1.03e-06 &          0 &        - &        - &        - &        - & 0.0001 &    0.948 &   0.039 \\
W1                            & 2.09e-06 &          0 &        - &        - &        - &        - &  0.001 &    0.976 &   0.043 \\
W1-95R                        &  3.9e-06 &          0 &        - &        - &        - &        - &  0.004 &     1.050 &   0.046 \\
M-95L                         & 8.33e-07 &   0 &   0.001 & 2.54e-08 & 0.0003 &  0.002 &        0 &   0.014 &   0.039 \\
M                             & 2.29e-06 &   0 &  0.002 & 0.0003 & 0.001 &  0.003 &   0.001 &   0.0194 &   0.047 \\
M-95R                         & 4.79e-06 &   0 &   0.003 & 0.001 &  0.002 &  0.004 &  0.004 &    0.024 &   0.056 \\
N\_Par                         & 1.55e+05 &          0 &      200 &        200 & 3.98e+06 & 1.41e+05 &        - & 1.41e+05 & 4.66e+05 \\
Train\_Time                    &     15.1 &       1.02 & 1.62e+09 &     1.12 &     14.2 &     11.6 &    0.889 &     12.5 &    276 \\
$\frac{\text{Test-Time}}{\text{MC-Oracle Test-Time}}$ &    0.26 &          1 &  0.001 &    0.14 &  0.01 &    0.23 &  0.003 &    0.20 &   0.13 \\
\bottomrule
\end{tabular}
}
\end{table}

Table and~\ref{tabl_MC2} show that the probabilistic transformer reflects the information in the training data more efficiently than the classical learning models since it achieves a good lower value of M.  Moreover, it does so while simultaneously achieving lower a lower W1 value.  Thus, the PT architecture is well-suited for predicting the uncertainty arising from MC-dropout.

\subsection{Learning Heteroscedastic Noise}
\label{ss_Numerical_Hetero_Noise}
The next example concerns heteroscedastic regression.  Heteroscedasticity is a common phenomenon, and it is particularly central to econometrics (e.g., \cite{engle1982autoregressive} and \cite{mcculloch1985miscellanea}).  
\begin{example}[Approximating Heteroscedastic Noise]\label{ex_heteroscedastic_regression}
Typical non-linear regression problems are interested in inferring an unknown function $f\in C(\rrd,\rrD)$ from a set of noisy observations:
\begin{equation*}
    \{Y_x = f(x) + \epsilon_x\}_{x\in\xx}
    ,
\end{equation*}
where $\xx\subset \rrd$ is a (finite non-empty) dataset and $\{\epsilon_x\}_{x \in \xx}$ are independent and integrable $\rrD$-valued random vectors with mean $0$.  A key point here is that the random variables $\{\epsilon_x\}_{x\in \xx}$ are not required to be identically distributed.  Since each $\epsilon_x$ has mean $0$, such non-linear regression tasks are interested in learning the map:
\begin{equation}
    \xx\ni x \mapsto \mathbb{E}\left[Y_x\right]=f(x) \in \rrD
\label{eq_vanilla_regression}
.
\end{equation}
However, approximation of~\eqref{eq_vanilla_regression} need not provide any information about the noise $\epsilon_x$.  Therefore, instead of approximating~\eqref{eq_vanilla_regression}, it is more informative to instead approximate:
\begin{equation}
    \xx\ni x \mapsto \pp\left(f(x)+\epsilon_x\in \cdot\right) \in \ppp[\rrD][1]
    .
    \label{eq_learning_noise_nonboring_regression}
\end{equation}
NB, the approximation in $1$-Wasserstein distance of~\eqref{eq_learning_noise_nonboring_regression}, uniformly for every $x \in \xx$, implies the approximation of~\eqref{eq_vanilla_regression}, uniformly for every $x \in \xx$.  
\end{example}

Consider the setting of Example~\ref{ex_heteroscedastic_regression}  and assume that the map $[0,1]^d\ni x \mapsto \pp\left(f(x)+\epsilon_x\ni \cdot \right)$ belongs to $\ppp[\rr][1]$; where each $\epsilon_x$ is distributed as a multivariate Laplace random-variable with mean $0$ and variance $\|x\|$ and where $f$ is an exogenously given DNN in $\NN[d,1]$.  In this experiment, $f$'s weights and biases are generated randomly by sampling independently and uniformly from $[-2^{-1},2^{-1}]$, and then they are left fixed.  

Unlike in the previous experiments, in this experiment, we can explicitly compute the law of $f(x)+\epsilon_x$ in closed form.  Therefore, we can explicitly compare the predictive capabilities of MC-Oracle to the probabilistic transformer model.  We emphasize that the law of $f(x)+\epsilon_x$ is not directly available to the user.  As reflected in Table~\ref{tab_hetero_skedestic_learning_d_100}, the inaccessible MC-Oracle method offers a better prediction of the law.  Nevertheless, the PT model offers the best overall performance with respect to the (W1) and (M) performance metrics. 

\begin{table}[H]
    \centering
    \caption{Quality Metrics: $d=100$, $Depth=1$, $Width=100$.}
    \label{tab_hetero_skedestic_learning_d_100}
\ra{1.3}
\resizebox{\columnwidth}{!}{
\begin{tabular}{@{}lrrllllrrr@{}}
    \toprule
    {} &      PT &  MC-Oracle &     ENET & KRidge &     GBRF &      DNN &    GPR &      GDN &      MDN \\
    \midrule
    W1-95L                        &    0.907 &          0 &        - &      - &        - &        - &    196 &      128 &     8.87 \\
    W1                            &     1.03 &          0 &        - &      - &        - &        - &    197 &      129 &      9.1 \\
    W1-95R                        &     1.19 &          0 &        - &      - &        - &        - &    198 &      129 &     9.78 \\
    M-95L                         &     23.1 &       23.5 &     22.6 &   22.4 &     21.9 &     22.2 &   22.2 &     21.7 &     23.3 \\
    M                             &     26.1 &       26.2 &     25.1 &   25.3 &     25.1 &     25.3 &   26.1 &     25.2 &     26.1 \\
    M-95R                         &     30.5 &         30 &     29.6 &   29.3 &     27.9 &     29.1 &   29.1 &     29.2 &     29.8 \\
    N\_Par                         & 2.11e+05 &          0 &      200 &      0 & 1.45e+05 & 6.06e+04 &      0 & 6.06e+04 & 6.33e+05 \\
    Train\_Time                    & 5.44e+03 &         14 & 1.62e+09 &   2.31 &     2.09 &     59.7 &   8.08 &     59.6 &    0.166 \\
    $\frac{\text{Test-Time}}{\text{MC-Oracle Test-Time}}$ &      0.4 &          1 & 0.000879 &  0.026 &  0.00224 &    0.246 & 0.0779 &    0.242 & 2.41e+05 \\
    \bottomrule
    \end{tabular}    
    }
\end{table}
 
Table~\ref{tab_hetero_skedestic_learning_d_100} shows that the PT offers similar accuracy to the other machine learning models when predicting $f$.  However, it also shows that the PT does so while simultaneously showing the best prediction of the probability measure-valued map $x\mapsto \mathbb{P}(f(x)+\epsilon_x\in \cdot)$.  
It also shows that unlike the other $\ppp[\rr][1]$-valued models, the PT's performance does not degrade in high-dimensional settings.  Thus, the PT model provides a viable means of learning probability measure-valued functions even in high-dimensional settings.  

\subsubsection{An Example: Learning the Law of Extreme Learning Machines}
\label{sss_Extreme_Learning_Machine}
We now use Corollary~\ref{cor_learnability} to show that the probabilistic transformer model can learn the behaviour of a standard class of randomized deep neural networks.  These randomized deep neural network models are not trained using conventional stochastic gradient descent schemes.  Instead, all but the last neural network layer's weights and biases are generated randomly, and only the network's final layer's parameters are trained.  This final layer is trained using the \textit{ridge regression} method of \cite{hoerl1970ridge}.  This is an instance of an \textit{extreme learning machine} introduced by \cite{ELMSHuangBabri2004}, which has found extensive theoretical study since then (see  \cite{ELMsHuangQuinYuMao2006} and \cite{LukasJPLyudmilaRiskboundsReservoir2020JMLR}).  The model is now formalized.  

Fix some $J,W,M\in \nn_+$ and let $\Theta\eqdef [-M,M]^{W[(d+1)+J(W+1)]}$.  Identify each $\theta \in \Theta$ with $(A_J,b_J,\dots,A_1,b_1)$ where if $j>1$ then $A_j$ is a $d_j\times d_{j-1}$-matrix with coefficients in $[-M,M]$, $b_j \in \rr[W]$, $d_j\eqdef W$ if $j>0$ and $d_0\eqdef d$.  Fix an activation function $\sigma \in C(\rr)$.   
Each $\theta\in  \Theta$ defines a \textit{feature map} $\hat{\varphi}_{\theta}\in C(\rrd,\rr[W])$ via:
\begin{equation*}
    \hat{\varphi}_{\theta}(x)\eqdef x^{(J)},
\qquad
x^{(j)}\eqdef \sigma\bullet (A_{j-1} x^{(j-1)} + b_{j}),
\qquad x^{(0)} \eqdef x
.
\end{equation*}

Fix a training dataset $\xx\subset \rrd$ (i.e.\ a non-empty finite set) and a hyperparameter $\lambda>0$, let $Y^{\xx}$ be the $\#\xx\times D$-matrix whose rows are $Y^{\xx}_{\tilde{x}}\eqdef f(\tilde{x})$, indexed via $\tilde{x}\in \xx$, and let $X^{\xx:\theta}$ be the $\# \xx\times W$-matrix whose rows are $X^{\xx:\theta}_{\tilde{x}}\eqdef \hat{\varphi}(\tilde{x})$, for $\tilde{x}\in \xx$.  Define the learning-model class $\{\hat{f}_{\theta}:\theta \in \Theta\}$ through the associated ridge-regression solution operator as:
\begin{equation}
    \hat{f}_{\theta}(x) \eqdef 
\hat{\varphi}_{\theta}(x)
\left(
    (X^{\xx:\theta})^{\top}X^{\xx:\theta} + \lambda I_{W}
\right)^{-1}
(X^{\xx:\theta})^{\top}
Y^{\xx}
\label{eq_extreme_learning_machines}
.
\end{equation}
\begin{example}[The Law of Extreme Learning Machines is Approximable]\label{cor_learnability_of_ExLMs}
In the setting of~\eqref{eq_extreme_learning_machines}, let $\vartheta$ be a uniform random-vector in $\Theta$ defined on some auxiliary probability space $(\Omega,\fff,\pp)$.  Suppose that $\sigma$ satisfies Condition~\ref{condi_KL}.
For every $\epsilon>0$, Corollary~\ref{cor_learnability} guarantees that there exists a probabilistic transformer approximating the map $x\mapsto \pp(\hat{f}_{\vartheta}(x)\in \cdot)$ to arbitrary precision.  
\end{example}

We illustrate our theoretical finding numerically. 
Fix $J,W\in \nn_+$.  The randomized parameter $\vartheta$ is defined to be $\vartheta_i\eqdef (\vartheta_{1,i},\vartheta_{2,i})$, for $i=1,\dots,d$, and where the $\vartheta_{1,i}$ are i.i.d. and uniformly distributed on $[-M,M]$, where $M=1$, and where the $\vartheta_{2,i}$ are i.i.d. and independent of the $\vartheta_{1,i}$; moreover, and $\vartheta_{2,i}$ is a Bernoulli random variable with $.75$ probability of being $0$.  Thus, the extreme learning machine's parameters are highly sparse.

We are given 600 consecutive business days of stock returns from the following tickers: 'IBM,' 'QCOM,' 'MSFT,' 'CSCO,' 'ADI,' 'MU,' 'MCHP,' 'NVR,' 'NVDA,' 'GOOGL,' 'GOOG,' and 'AAPL.'  The objective is a regression task aiming at predicting the returns of 'AAPL' on the following day, given the closing prices of the remaining stocks\footnote{
These tickers are used because they are significant presences in the business sector as 'AAPL' or are vital constituents of its supply chain (see \cite{APPLsc}).  
}.  For each $x\in \rr[11]$ in the dataset, the parameters of the extreme learning machine are generated by sampling from $\vartheta$ and then plugging them into the learning model of~\eqref{eq_extreme_learning_machines}.  The training set consists of the pairs $\left\{(x,\{X_x^s\}_{s=1}^S)\right\}_x$, where $x$ is the closing prices indexed over the first $80\%$ of the data and the testing set consists of the remaining pairs and $X_x^1,\dots, X_x^S$ are the predicted next-day prices predicted by the ELMs with randomly generated internal randomness (here we have $S$ such draws for the ELM's random internal structure; i.e.\ $S$ i.i.d.\ samples of the ELM's randomly generated hidden weights and biases).        
\begin{table}[H]
\centering
\caption{Quality Metrics: Extreme Learning Machine for AAPL Returns; Depth: $2$, Width: $10^3$.}
\ra{1.3}
\resizebox{\columnwidth}{!}{
\begin{tabular}{@{}lrrllllrrr@{}}
    \toprule
    {} &      PT &  MC-Oracle &     ENET &   KRidge &     GBRF &      DNN &      GPR &      GDN &      MDN \\
    \midrule
    W1-95L                        & 0.000304 &          0 &        - &        - &        - &        - & 0.000258 &    0.999 & 0.000264 \\
    W1                            & 0.000869 &          0 &        - &        - &        - &        - & 0.000904 &        1 &  0.00087 \\
    W1-95R                        &   0.0016 &          0 &        - &        - &        - &        - &  0.00196 &     1.01 &   0.0016 \\
    M-95L                         & 0.000103 &          0 & 0.000525 &  0.00431 & 0.000203 & 0.000366 & 0.000578 &   0.0471 & 0.000422 \\
    M                             & 0.000233 &          0 & 0.000626 &  0.00597 & 0.000317 & 0.000445 &  0.00132 &   0.0475 & 0.000546 \\
    M-95R                         & 0.000399 &          0 & 0.000768 &  0.00783 & 0.000422 & 0.000545 &  0.00196 &    0.048 & 0.000759 \\
    N\_Par                         & 1.15e+05 &          0 &       22 &        22 &      550 & 4.28e+04 &        - & 4.28e+04 & 3.44e+05 \\
    Train\_Time                    &      101 &   1.62e+09 & 1.62e+09 &     0.65 &    0.238 &     14.7 &     10.8 &     14.9 &    0.133 \\
    $\frac{\text{Test-Time}}{\text{MC-Oracle Test-Time}}$ & 0.000133 &          1 & 1.94e-07 & 4.24e-05 & 5.72e-07 &  0.00014 &  4.9e-05 & 0.000124 &     0.96 \\
    \bottomrule
\end{tabular}
}
\label{tab_experiment_AAPL_easy}
\end{table}

Table~\ref{tab_experiment_AAPL_easy} shows that the PT can learn the law of each $\pp(\tilde{f}_{\vartheta}(x)\in \cdot )$ while also offering a competitive approximation of the map $x\mapsto \ee[\tilde{f}_{\vartheta}(x)]$.  The experiment also shows that the approximation is stable independently of how sophisticated the map $\tilde{f}$ is; i.e.\ regardless if $\tilde{f}$ is extremely deep or wide and of the randomization $\vartheta$; i.e.\ when its components have a high probability of being sparse. 

The theoretical contributions made in this article are now summarized.  
\section{Conclusion}\label{s_conclusion}
This paper introduces the first principled deep learning model, which can provably approximate regular conditional distributions (RCD).  This result is a consequence of Theorem~\ref{theorem_MAIN__GeneralCase}, which is a universal approximation result guaranteeing that the probabilistic transformer model can approximate any uniformly continuous function mapping inputs in a suitable metric space $\xxx$ to outputs in the $1$-Wasserstein space over some Euclidean space.  
We demonstrate that there are situations in which the PT model can implement these approximations while avoiding the curse of dimensionality.  These conditions are met if the target function is sufficiently regular or if the approximation is only required to hold uniformly on compact subsets of $\xxx$ which are ``close to a given (finite) training dataset''.    

\subsection*{Acknowledgements}
The author would especially like to thank Beatrice Acciaio and Gudmund Pammer for their many helpful and encouraging discussions.  The author would like to thank Hanna Wutte for her extensive help and feedback in the finalization stages of this paper.  The author would like to thank Behnoosh Zamanlooy for her input and helpful discussions.   The author would also like to thank Josef Teichmann for his valuable mentorship and direction.  The author would also like to thank the entire working group at the ETH for their helpful feedback and encouragement throughout this project's development.  The author would like to thank Andrew Allan for his insightful discussions on SDEs driven by fBMs with Hurst exponent $H\in (1/2,1)$.  The author would also like to thank Thorsten Schmidt for his interest and our stimulating discussions on this research project's many future implications in the context of stochastic filtering.

            %
            %
\begin{appendix}
\section{Background}
\label{s_Background}
First, the relevant geometric deep learning tools, developed in \cite{kratsios2020non}, \cite{paponkratsios2021quantitative}, and in \cite{kratsios2021NEU} are presented.  This will be used to produce universal approximators between non-linear finite-dimensional input and output spaces.  Then, relevant tools from the theory of $\epsilon$-metric projections are surveyed; these will be used to almost optimally reduce the infinite-dimensional output space $\ppp[\yyy][1]$ to a certain finite-dimensional topological manifold.  The proposed model is then formalized.  

\subsection{Geometric Deep Learning}\label{ss_GDL_Background}
Introduced in \cite{mcculloch1943logical}, the class of feedforward neural networks (DNN) from $\rrd$ to $\rrD$ with \textit{activation function} $\sigma \in C(\rr)$, denoted by $\NN$, is the set of continuous functions $\hat{f}:\rrd\rightarrow \rrD$ admitting a recursive representation:
\begin{equation}
    \begin{aligned}
 &\hat{f} = W_{J+1}\circ f^{(J)},\qquad
&f^{(j+1)}=\sigma \bullet W_j\circ f^{(j)},\qquad
& f^{(1)} = W_1
;
\end{aligned}
\label{eq_representation_DNN_intro}
\end{equation}  
where 
$j=1,\dots,J$, $J \in \nn_+$, $W_j:\rr[d_j]\rightarrow \rr[d_{j+1}]$ are affine functions, $d_1=d$, $d_{J+1}=D$, and $\bullet$ denotes component-wise application.  
The width of (the representation of~\eqref{eq_representation_DNN_intro}) $\hat{f}$ is $\max_{1\leq j\leq J}\{d_j\}$ and its depth is $J$.  
Following, \cite{gribonval2019approximation}, the complexity of a DNN with representation~\eqref{eq_representation_DNN_intro} is the number of its trainable parameters.  Following \cite{paponkratsios2021quantitative}, we estimate the complexity of any DNN's representation by its height and depth; this is because a DNN's total number of parameters can be explicitly upper-bounded as a function of its depth $J$ and height $H$.  

The success of DNNs lies at the intersection of their universal approximation property, their efficient approximation capabilities for large classes of functions, and their easily implementable structure.  Initially proven for the $J=1$ case by \cite{Hornik} and by \cite{Cybenko}, later characterized by \cite{leshno1993multilayer} and \cite{pinkus1999approximation}, and recently for the arbitrary depth case in \cite{kidger2019universal}, the \textit{universal approximation theorem} states that $\NN$ is dense in $C(\rrd,\rrD)$, for the uniform convergence on compacts topology, if the following condition is satisfied by the activation function $\sigma$.  
\begin{condition}[Kidger-Lyons Condition: \cite{kidger2019universal}]\label{condi_KL}
The activation function $\sigma \in C(\rr)$ is not affine, and it is differentiable at least one point and with a non-zero derivative at that point.  
\end{condition}
More recently, it was shown in \cite{kratsios2020non} that DNNs can be extended to accommodate inputs and outputs from general metric spaces $\xxx$ and $\yyy$, respectively.  This can be done by precomposing every $\hat{f}\in \NN[d,D]$ by a \textit{feature map} $\varphi\in C(\xxx,\rrd)$ and post-composing $\hat{f}$ by a function $\rho\in C(\rrD,\yyy)$.  If such maps exist, then the density in $C(\xxx,\yyy)$ of the resulting class of conjugated deep neural models $\left\{\rho\circ \hat{f}\circ \varphi:\, \hat{f}\in \NN[d,D]\right\}$ can be inferred from the density of the class $\NN$ under certain invariance conditions on $\varphi$ and $\rho$.  The condition, recorded in Assumption~\ref{cond_feature}, on $\varphi$ is known to be sharp.  

We will always assume a UAP-Invariant feature map is provided.  We note that when $\xxx$ is embedded in $\rrd$ then any feature map $\varphi:\xxx\rightarrow \rrd$ can be approximated by UAP-Invariant feature maps using the reconfiguration networks of \cite{kratsios2021NEU}.  

Following \cite{Brown1962Collared}, a metric space $\mmm$ is called a (metric) manifold (with boundary)\footnote{
We will refer to any metric topological submanifold $\mmm$ of $\ppp[\yyy][1]$ (possibly with boundary) as a manifold.  
} if every $y\in \mmm$ is contained in a (relatively) open set $U_y\subseteq \mmm$ for which there exists a continuous bijection $\varphi_y$ with continuous inverse from $U_y$ either to $\rrD$ or to the upper-half space $\{z \in \rrD:\, z_n\geq 0, \, n=1,\dots,D\}$; where $D$ is the same for each $y \in \mmm$.  
The set of all points in $y\in \mmm$ where $\varphi_y$ identifies $U_y$ with $\rrD$ is called the \textit{interior} of $\mmm$, denoted by $\operatorname{int}(\mmm)$, and all other points of $\mmm$ belong to its boundary, denoted by $\partial \mmm$.  We focus on the $N$-simplex embedded in $\ppp[\yyy][1]$. 
\begin{example}[Hull of a Finite Set of Probability Measures]\label{ex_hull}
Suppose that we are given a finite set of probability measures $\{\mu_n\}_{n=1}^N\subseteq \ppp[\yyy][1]$ and consider its \textit{hull}, as defined as 
\[
    \co{\{\mu_n\}_{n=1}^N}
        \eqdef 
    \big\{
        \nu \in \mathcal{P}_1(\yyy):\, (\exists w \in \Delta_N)\,
            \nu = \sum_{n=1}^N\, w_n \mu_n 
    \big\}
.
\]
The subspace $\co{\{\mu_n\}_{n=1}^N}$ is a finite-dimensional metric submanifold of $\ppp[\yyy][1]$ whose boundary $\partial \co{\{\mu_n\}_{n=1}^N}$ contains the set $\{\mu_n\}_{n=1}^N$.
\end{example}
Ultimately, we would like to map the output of any DNN in $\NN[d,D]$ to $\mmm$ via a continuous map with a continuous inverse.  However, many manifolds of interest, such as the $N$-simplex, are compact, and therefore such a map cannot exist since $\rrD$ is not compact.  Nevertheless, this topological obstruction can be circumvented as follows.

Building on the ideas of \cite{anderson1967topological}, in \cite{kratsios2020non} the geometry of any such metric topological submanifold $\mmm$ of $\ppp[\yyy][1]$ was exploited, via the \citep[Collar Neighborhood Theorem]{Brown1962Collared}, to construct a function $T\in (0,1)\times C(\mmm,\operatorname{int}(\mmm))$ which continuously deletes the boundary of $\mmm$ while simultaneously asymptotically approximating the identity function:
\begin{enumerate}[(i)]
    \item For every $(t,x)\in (0,1)\times \mmm$, we have that $T_t(x)\in \operatorname{int}(\mmm)$,
    \item $T_t$ converges uniformly to the identity on $\mmm$.
\end{enumerate}
Thus $T^{}$ allows us to approximate continuous functions with values in $\mmm$ using continuous functions in $\operatorname{int}(\mmm)$.  This is key because $\operatorname{int}(\mmm)$ may not be compact. Therefore there can exist continuous surjective functions from $\rrD$ to $\operatorname{int}(\mmm)$ with continuous inverses, which we can use to modify the outputs of the DNNs in $\NN[d,D]$. 
\begin{example}\label{ex_truncation_function_of_the_N_simplex}
For the $N$-simplex, the following is an example of a function satisfying (i) and (ii) above:
\begin{equation}
    T_{\cdot}^{\Delta_N}(\cdot):
(0,1)\times \Delta_N \ni 
(t,\beta) \mapsto 
\begin{cases}
t
(\beta-\bar{\Delta}_N) +\bar{\Delta}_N
&: \, \mbox{ if }
\beta \neq\bar{\Delta}_N\\
\bar{\Delta}_N &: \, \mbox{ if }
\beta =\bar{\Delta}_N
\end{cases}
\in \operatorname{int}\left(\Delta_N\right)
\label{eq_definition_of_T_opterator}
,
\end{equation}
where $\bar{\Delta}_N$ is the $N$-simplex's barycenter $\bar{\Delta}_N\eqdef (\frac1{N},\dots,\frac1{N})\in \text{int}(\Delta_N)$
.  
\end{example}
Next, the relevant theory of approximate metric projections is reviewed.  These results, lying at the junctions of general topology and set-valued analysis, form a key cornerstone of this paper's analysis\footnote{For more details, we recommend \cite{RepovsSemenov1998ContinuousSelectors}.}.    
\subsection{$\epsilon$-Metric Projections}\label{ss_epsilon_metric_projections}
Given a closed $A\subseteq \ppp[\yyy][1]$, we would like to systematically, optimally, and continuously identify the closest measures in $A$ to any given $\mu \in \ppp[\yyy][1]$.  The set-valued function $P_A:\ppp[\yyy][1]\mapsto 2^A$, which maps any $\mu\in\ppp[\yyy][1]$ to the collection of probability measures in $A$ with minimal Wasserstein distance to $\mu$, is called the \textit{metric projection} (or best approximation operator) in $\ppp[\yyy][1]$ and it is defined by:
\begin{equation}
    P_{A}(\mu)
        \eqdef 
    \left\{
    \nu \in A:\,
    W_1 (\nu,\mu)
      =
      \inf_{\tilde{\nu}\in A} W_1 \left(
      \tilde{\nu}
        ,
    \mu
      \right)
    \right\}
    \label{eq_best_approximation_set}
    .
\end{equation}
However, in general, $P_A$ is not single-valued and can even be empty (see \citep[Theorem 6.6]{RepovsSemenov1998ContinuousSelectors} and \cite{motzkin1935quelques} for examples in the context of metric-projections in the simpler Banach-space situations).    

Following \cite{LiskovecQuasiSolutionsApproximation1973}, the problem that~\eqref{eq_best_approximation_set} may be empty for some $\mu\in \ppp[\yyy][1]$ can be overcome by instead considering approximations of the set-function $P_A$, for some error tolerance $\epsilon>0$.  These $\epsilon$-best approximations are always non-empty if they are defined by:
\begin{equation*}
    P_{A}^{\epsilon}(\mu)
        \eqdef 
    \left\{
    \nu \in A:\,
    W_1 (\nu,\mu)
      \leq
      \inf_{\tilde{\nu}\in A} W_1\left(
      \tilde{\nu}
        ,
    \mu
      \right)
      +\epsilon
    \right\}
    .
\end{equation*}
The analogue of the otherwise ill-defined problem of continuously associating any $\mu\in \ppp[\yyy][1]$ to its closest measure in $A$, according to the Wasserstein distance, in the context of~\eqref{eq_best_approximation_set} is, therefore, a \textit{continuous selector} from $\ppp[\yyy][1]$ to $P_A^{\epsilon}$.  That is, an \textit{$\epsilon$-metric projection} (or $\epsilon$-best approximation operator) is a continuous function $\Pi^{\epsilon}_A\in C(\ppp[\yyy][1],A)$ satisfying the $\epsilon$-optimality condition:
$$
\Pi^{\epsilon}_A(\mu) 
    \in
P_A^{\epsilon}(\mu)
,
$$
for every $\mu\in \ppp[\yyy][1]$.  Since we are concerned with approximations of measure-valued functions, then it is both equally natural and convenient to show that the probabilistic transformer model can implement $\Pi^{\epsilon}_A\circ f$, for some $\epsilon$ metric-projection $\Pi^{\epsilon}_A$ with $A$ the hull of a finite number of measures.  In this way, we reduce the problem of learning a function with a potentially infinite-dimensional image to an almost optimal finite-dimensional approximation to $f$.

\section{Proof of Main Result}
Our first few lemmata rely on some additional background from the theory of Lipschitz-free Banach spaces.  We review the relevant background and demonstrate a few folkloric embedding lemmas before delving into our results.  The interested reader is referred to \cite{GeodefroyKaltonRemembering} and \cite{WeaverNice} for a detailed exposition of these spaces, initially introduced by \cite{ArensEells}.  

\subsection{Lipschitz-Free Spaces and Related Embedding Lemmata}
\label{ss_Appendix_Embedding_Lemmas}
We begin by describing our main isometric embedding of $\ppp[\xxx][1]$ into a particular Banach space originally introduced by \cite{ArensEells}, which has since come to be known both as the Arens-Eells space over $\xxx$ (see \cite{WeaverNice}), also termed the Lipschitz-Free space over $\xxx$ in the particular case when $\xxx$ is itself a Banach space (see \cite{godefroy2003lipschitz}).  The space can be constructed in various ways; see \cite{WeaverNice} for references.  Here, we outline the following a combination of the expositions in \cite{WeaverNice} and \cite{CuthisSonice2016StructureofLipschitzFreeSpaces}.  

Let $\xxx$ be a metric space, and pick an arbitrary $x_0\in \xxx$. The tripe $(\xxx,d_{\xxx},x_0)$ is called a pointed metric space.  Let $\operatorname{Lip}_0(\xxx)$ denote the set of Lipschitz functions from $\xxx$ to $\rr$ which map the distinguished point $x_0$ to $0$.  Unlike the space of Lipschitz functions, the semi-norm $\|\cdot\|_{\operatorname{Lip}_0}$ sending any $f\in \operatorname{Lip}_0(\xxx)$ to its minimal Lipschitz constant $\sup_{x_1,x_2\in \xxx} \frac{\|f(x_1)-f(x_2)\|}{d_{\xxx}(x_1,x_2)}$ (with the convention that $\frac{0}{0}=0$) defines a genuine norm on $\operatorname{Lip}_0(\xxx)$.  Moreover, as shown in \citep[Proposition 2.3]{WeaverNice}, $\operatorname{Lip}_0(\xxx)$ is complete under $\|\cdot\|_{\operatorname{Lip}_0}$ and therefore it is a Banach space.  Furthermore, $\operatorname{Lip}_0(\xxx)$ is always a dual space as shown in \citep[Theorem 2.37]{WeaverNice}, and it has (at least one) pre-dual which can be identified with the closure of the linear span of the integral functionals $\left\{e_x:\,f\mapsto \int_{z\in \xxx}f(z) (\delta_x-\delta_{x_0})(z)=f(x)\right\}_{x\in\xxx}$ in the dual space $\operatorname{Lip}_0(\xxx)^{\star}$ respect to its dual norm $
\|F\|_{\star}\eqdef \sup_{\|f\|_{\operatorname{Lip}_0(\xxx)}} F(f)
.
$  We denote this particular predual by $\text{\AE}(\xxx)$.  
As proven on \citep[page 3836]{CuthisSonice2016StructureofLipschitzFreeSpaces}, the topology induced by the dual norm $\|\cdot\|_{\star}$ coincides with the topology induced by the Arens-Eells norm $\|\cdot\|_{\text{\AE}}$, a variant of the Kantorovich-Rubinstein of \cite{KantorovichRubinsteinDualityOriginal1957}, which on any $F \in \operatorname{span}(\{e_x:\,x\in \xxx\})$ equals to:
$
    \left\|F\right\|_{\text{\AE}}
     = \inf \left\{
    \sum_{n=1}^N |k_n| d_{\xxx}(x_{n,1},x_{n,2}):\, 
    F = \sum_{n=1}^N k_n (e_{x_{n,1}}-e_{x_{n,2}})
    \right\}
    .
    $
In particular, for any $N,\tilde{N}\in \nn_+$, and $\beta \in \Delta_N$, $\tilde{\beta}\in \Delta_{\tilde{N}}$, and any $x_1,\dots,x_N,\tilde{x}_1,\dots,\tilde{x}_{\tilde{N}}\in \xxx$ we have the coincidental identity; which is essentially recorded in \citep[Section 3.3]{WeaverNice}:
\begin{equation}
    W_1
    \left(
        \sum_{n=1}^N\beta_n\delta_{x_n}
            ,
        \sum_{n=1}^{\tilde{N}}\tilde{\beta}_n\delta_{\tilde{x}_n}
    \right)
        =
    \left\|
        \sum_{n=1}^N\beta_n e_{x_n}
            -
        \sum_{n=1}^{\tilde{N}}\tilde{\beta}_n e_{\tilde{x}_n}
    \right\|_{\textit{\AE}}
    \label{eq_Arens_Eells_Background_coincidence_with_wasserstein_setup}
    .
\end{equation}
NB, by \citep[Page 2]{Weaver2018Uniquenesspredual}, the construction is independent of our choice of $x_0$; up to isometry.  Therefore, for convenience, we may henceforth omit any explicit dependence on $x_0$.  
We are now able to record the following metric embedding lemma.  We use $e_0$ to denote the ``reference" integral functional 
$e_0:\operatorname{Lip}_{0}(\xxx)\ni 
f\mapsto \int_{z\in \xxx} f(z)\delta_{x_0} = f(x_0)$.  
NB, when it is obvious from the context (e.g.\ as in~\eqref{eq_Arens_Eells_Background_coincidence_with_wasserstein_setup}) we will suppress the explicit dependence of the Arens-Eells space above $\yyy$ on the choice of base-point and simply write $\|\cdot\|_{\text{\AE}}$ in place of $\|\cdot\|_{\text{\AE}(\yyy,y_0)}$.
\begin{lemma}[Embedding Lemma A]\label{lem_isometic_embedding}
Let $\xxx$ be a metric space and fix some $x_0\in \xxx$.  Then the map: 
$$
\Psi:\ppp[\xxx][1]\ni \mu \mapsto 
\left[
    f\mapsto \int_{x \in \xxx} f(x)d\mu(x)
        - 
    f(x_0) 
\right] 
\in 
\text{\AE}(\xxx)
$$ 
is an isometric embedding from $\ppp[\xxx][1]$ to the following closed and convex subset of $\text{\AE}(\xxx)$:
\begin{equation}
D^+(\xxx)\eqdef \overline{\left\{
\sum_{n=1}^N \beta_n e_x:\, N\in \nn_+, \, \beta \in \Delta_N,\, x_1,\dots,X_N \in \xxx
\right\}}
;
\label{lem_isometic_embedding_set_description}
\end{equation}
which sends any $\sum_{n=1}^N \beta_n \delta_{x_n}\in \ppp[\xxx][1]$ to the element
$\sum_{n=1}^N \beta_n e_{x_n}\in \overline{\operatorname{B}_{\text{\AE}(\xxx)}(0,e_0)}\subset
\text{\AE}(\xxx)
$ of~\eqref{lem_isometic_embedding_set_description}
.
\end{lemma}
\begin{proof}
Since the completion of any metric space is uniquely determined up to isometry and since the closed convex hull of $\left\{e_x\right\}_{x \in \xxx}\subset \text{\AE}(\xxx)$, which we denote by $\overline{\operatorname{co}\left(\left\{e_x\right\}_{x \in \xxx}\right)}$, is a complete subspace of $\text{\AE}(\xxx)$ then the completion of $D\eqdef \left\{\sum_{n=1}^N \beta_n\delta_{x_n}\in \ppp[\yyy][1]\,:\, N\in \nn_+, x_1,\dots,x_N \in \xxx,\, \beta \in \Delta_N\right\}$ for the $1$-Wasserstein metric is isometric to $\overline{\operatorname{co}\left(\left\{e_x\right\}_{x \in \xxx}\right)}$.  Moreover, by~\eqref{eq_Arens_Eells_Background_coincidence_with_wasserstein_setup} the isometry is given by the extension of the map $\sum_{n=1}^N \beta_n \delta_{x_n} \mapsto \sum_{n=1}^N \beta_n e_{x_n}$.  
\end{proof}
The next embedding lemma concerns the identification of $\co{\{\hat{\mu}_n\}_{n=1}^N}$ with the standard simplex $\Delta_N$ via the map:
\begin{equation}
    \Phi_{\co{\{\hat{\mu}_n\}_{n=1}^N}}: \Delta_N \ni \beta \mapsto \sum_{n=1}^N \beta_n \hat{\mu}_n \in \ppp[\yyy][1]
\label{eq_lipschitz_identification_simplex_Wasserstein_hull}
.
\end{equation}
We will often combine this identification with Lemma~\ref{lem_isometic_embedding} to obtain an embedding of the standard simplex in $\text{\AE}(\xxx)$.  The next lemma gives us a handle on the regularity of this identification.  
\begin{lemma}[Embedding Lemma B]
\label{lem_Lipschitz_constant_parameterization}
For every $N\in \nn_+$ and every $\hat{\mu}_1,\dots,\hat{\mu}_N\in \ppp[\yyy][1]$, the map of~\eqref{eq_lipschitz_identification_simplex_Wasserstein_hull} is $2\sqrt{N}$-Lipschitz continuous.  Moreover, if each of the $\hat{\mu}_1,\dots,\hat{\mu}_N\in \ppp[\yyy][1]$ are distinct then~\eqref{eq_lipschitz_identification_simplex_Wasserstein_hull} is continuous and injective with Lipschitz inverse.  
\end{lemma}
\begin{proof}%
Let $\beta,\gamma\in \Delta_N$ by Lemma~\ref{lem_isometic_embedding} we have:
\allowdisplaybreaks
\begin{align}
\nonumber
W\left(\sum_{n=1}^N \beta_n \mu_n, \sum_{n=1}^N \gamma_n \mu_n\right) = &
\left\|\sum_{n=1}^N \beta_n \Psi(\mu_n) - \sum_{n=1}^N \gamma_n \Psi(\mu_n)  \right\|_{\text{\AE}} 
\\ 
\nonumber
= &  
\left\|\sum_{n=1}^N (\beta_n-\gamma_n) \Psi(\mu_n)\right\|_{\text{\AE}}
\\
\nonumber
\leq & \sum_{n=1}^N\left\| (\beta_n-\gamma_n) \Psi(\mu_n)\right\|_{\text{\AE}}
\\
\nonumber
\leq & \sum_{n=1}^N|\beta_n-\gamma_n| 
\max_{1\leq n\leq N} \left(
\|\Psi(\mu_n)-e_0\|_{\text{\AE}} + \|e_0\|_{\text{\AE}}
\right)
\\
\nonumber
\leq & 2 \sum_{n=1}^N|\beta_n-\gamma_n|
\\
\nonumber
\leq & 2 \sqrt{N}\sqrt{\sum_{n=1}^N|\beta_n-\gamma_n|^2}
\\
\nonumber
= & 2\sqrt{N}\|\beta-\gamma\|_2.
\end{align}
So the map of~\eqref{eq_lipschitz_identification_simplex_Wasserstein_hull} is Lipschitz constant is at-most $2\sqrt{N}$.  

Next, suppose that each $\hat{\mu}_1,\dots,\hat{\mu}_N$ are distinct. Since $\Psi$ is an isometric embedding, it is injective, and therefore $\{\Psi(\hat{\mu}_n)\}_{n=1}^N$ is a basis of the finite-dimensional normed space $\left(\operatorname{span}\left(\{\Psi(\hat{\mu}_n)\}_{n=1}^N\right),\|\cdot\|_{\text{\AE}}\right)$.  Since all norms are equivalent on a finite-dimensional normed space, by \citep[Proposition 1.5]{conway2013course}, and since $\{\Psi(\hat{\mu}_n)\}_{n=1}^N$ is a basis for the linear space 
$\left(\operatorname{span}\left(\{\Psi(\hat{\mu}_n)\}_{n=1}^N\right),\|\cdot\|_{\text{\AE}}\right)$ then $\Psi\circ \Phi_{\operatorname{hull}(\{\hat{\mu}\}_{n=1}^N)}$ is bi-Lipschitz.  Since $\Psi$ is an isometry then, we conclude that there must be a constant $c>0$ such that:
\begin{equation}
    c\left\|\beta-\gamma\right\|_2 
        \leq 
    W_1\left(\sum_{n=1}^N \beta_n \hat{\mu}_n,\sum_{n=1}^N \gamma_n \hat{\mu}_n\right)
        \leq 
    2\sqrt{N}\left\|\beta-\gamma\right\|_2 
    \label{lem_Lipschitz_constant_parameterization_bi_lsipchitz_redux}
    .
\end{equation}
The left-hand side of~\eqref{lem_Lipschitz_constant_parameterization_bi_lsipchitz_redux} implies that $\Phi_{\co{\{\hat{\mu}_n\}_{n=1}^N}}^{-1}$ is injective since it is bi-Lipschitz.  Let $\mu,\nu \in \ppp[\yyy][1]$ and set $\beta\eqdef \Phi_{\co{\{\hat{\mu}_n\}_{n=1}^N}}^{-1}\left(\mu\right)$, and 
$\Phi_{\co{\{\hat{\mu}_n\}_{n=1}^N}}^{-1}\left(\nu \right)$.  Then by~\eqref{lem_Lipschitz_constant_parameterization_bi_lsipchitz_redux} we compute:
\allowdisplaybreaks
\begin{align}
\nonumber
& c\left\|\Phi_{\co{\{\hat{\mu}_n\}_{n=1}^N}}^{-1}\left(\mu\right) - \Phi_{\co{\{\hat{\mu}_n\}_{n=1}^N}}^{-1}\left(\nu\right)\right\|_2
   \\
   \nonumber
   & \leq 
W_1 \left(
\Phi_{\co{\{\hat{\mu}_n\}_{n=1}^N}}\circ \Phi_{\co{\{\hat{\mu}_n\}_{n=1}^N}}^{-1}\left(\mu\right)
,
\Phi_{\co{\{\hat{\mu}_n\}_{n=1}^N}}\circ \Phi_{\co{\{\hat{\mu}_n\}_{n=1}^N}}^{-1}\left(\nu\right)
\right)
\\ 
\nonumber
&= W_1 \left(\mu,\nu\right)
;
\end{align}
thus, $\Phi_{\co{\{\hat{\mu}_n\}_{n=1}^N}}^{-1}$ is $c^{-1}$-Lipschitz.  
\end{proof}
Combining Lemma~\ref{lem_isometic_embedding} with Lemma~\ref{lem_Lipschitz_constant_parameterization} we have the following.
\begin{lemma}[Embedding Lemma C]\label{lem_isometic_embedding_partII}
Let $\xxx$ be a metric space and fix some $x_0\in \xxx$.  Then the map $\Psi$ of Lemma~\ref{lem_isometic_embedding} satisfies the following:
\begin{enumerate}[(i)]
\item For any $\mu_1,\dots,\mu_N \in \ppp[\yyy][1]$, the set $\Psi(\operatorname{hull}(\{\mu_n\}_{n=1}^N))$ is closed an convex in in $\textit{\AE}(\xxx)$,
\item $\Psi$ preserves convex combinations, in the sense that, for any $\mu_1,\dots,\mu_N\in \ppp[\yyy][1]$ and any $\lambda\in \Delta_N$ the following holds:
$
\Psi\left(
\sum_{n=1}^N \lambda_n\mu_n
\right) = 
\sum_{n=1}^N 
\lambda_n
\Psi\left(
\mu_n
\right)
.
$
\end{enumerate}
\end{lemma}
\begin{proof}
For (i): Since $\Psi$ is an isometry, it is a continuous and injective map; thus, it is a homeomorphism onto its image.  By Lemma~\ref{lem_Lipschitz_constant_parameterization} the map $\Phi|_{\co\{\mu_n\}_{n=1}^N}: \Delta_N\ni \beta \mapsto \sum_{n=1}^N \beta_n 
\mu_n$ is a homeomorphism onto its image.  Hence, their composition $\Psi\circ \Phi_{\co{\{\mu_n\}_{n=1}^N}}$ is a homeomorphism onto its image.  Since $\Delta_N$ is closed and bounded in $\rrD$ then the Heine-Borel Theorem implies that it is compact; whence, $\Psi\circ \Phi_{\co{\{\mu_n\}_{n=1}^N}}(\Delta_N)=\operatorname{hull}(\{\mu_n\}_{n=1}^N)$ is compact in $\text{\AE}(\xxx)$.  It remains to show that the set is convex.  Indeed this is the case since any $F_i\in \Psi(\operatorname{hull}(\{\mu_n\}_{n=1}^N))$, for $i\in \{1,2\}$, are of the form $F_i = \sum_{n=1}^N \beta_{n}^{(i)} \mu_n$ for some $\beta^{(i)}\in \Delta_N$.  Therefore, for any $\lambda \in [0,1]$ we have:
$$
\lambda F_1 + (1-\lambda)F_2 
= 
\sum_{n=1}^N 
\left(\lambda \beta^{(1)} + (1-\lambda)\beta^{(2)}\right) \mu_n.
$$
Since $0\leq \left(\lambda \beta^{(1)} + (1-\lambda)\beta^{(2)}\right)_n\leq 1$, for each $n=1,\dots,N$, and since 
$$
\sum_{n=1}^N \left(\lambda \beta^{(1)} + (1-\lambda)\beta^{(2)}\right) = 
\lambda \sum_{n=1}^N \beta_n^{1} + 
(1-\lambda) \sum_{n=1}^N \beta_n^{2} = \lambda + (1-\lambda)
=1;
$$
then $\lambda F_1 + (1-\lambda)F_2 \in \Psi(\operatorname{hull}(\{\mu_n\}_{n=1}^N))$ for every $\lambda \in [0,1]$.  Therefore, $\Psi(\operatorname{hull}(\{\mu_n\}_{n=1}^N))$ is convex.  

For (ii): Let $\mu_1,\dots,\mu_N\in \ppp[\xxx][1]$ and $\lambda\in \Delta_N$.  For any $f \in \operatorname{Lip}_0(\xxx)$ we compute:
\allowdisplaybreaks
\begin{align}
\nonumber
\Psi\left(\sum_{n=1}^N \lambda_n \mu_n \right)
= & 
\int_{x \in \xxx}f(x) 
\left[\sum_{n=1}^N\lambda_n \mu_n\right](dx)
- f(x_0)
\\
\nonumber
= & 
\int_{x \in \xxx}f(x) 
\left[\sum_{n=1}^N\lambda_n \mu_n\right](dx)
-
\sum_{n=1}^N \lambda_n f(x_0)
\\
\nonumber
= & 
\sum_{n=1}^N\lambda_n \int_{x \in \xxx}f(x) 
\mu_n(dx)
-
\sum_{n=1}^N \lambda_n f(x_0)
\\
\nonumber
= &
\sum_{n=1}^N\lambda_n \left[\int_{x \in \xxx}f(x) 
\mu_n(dx)
-
f(x_0)
\right]
\\
\nonumber
=&  \sum_{n=1}^N \lambda_n \Psi(\mu_n)(f)
.
\end{align}
This concludes the proof.  
\end{proof}
\subsection{Technical Lemmata}
\label{ss_App_Techincal_Lemmas}
The following lemma will be used to reduce the dimension of the image $f(\xxx)$ from a potentially infinite-dimensional object to one determined by finitely many ``well-positioned" probability measures in $f(\xxx)$.  
For a non-empty subset $A$ of a metric space $(X,d)$ and any $\epsilon>0$, we define the $\epsilon$-covering number of $A$ as the smallest positive integer $\mathcal{N}^{cov}_{\epsilon}(A)$ for which there exists $x_1,\dots,x_{\mathcal{N}^{cov}_{\epsilon}(A)}\in A$ satisfying
\[
    \max_{x\in A}\,
    \min_{n=1,\dots,\mathcal{N}^{cov}_{\epsilon}(A)}\,
        d(x,x_n)
    <\epsilon
.
\]
The $\epsilon$-covering number quantifies the complexity of a set.  We have the following estimate on the $\epsilon$-covering number of $f(\xxx)$.  
\begin{lemma}[$\frac{\epsilon}{4}$-Covering Lemma]\label{lem_covering_lemma}
Let $\xxx$ be a compact metric space, $f \in C(\xxx,\ppp[\yyy][1])$ with modulus of continuity $\omega_f$, and $\varphi\in C(\xxx,\rrd)$ be injective with modulus of continuity $\omega_{\varphi}$.  For 
$0<\epsilon\leq 2^{-2}\sup_{t\in [0,\infty)}\omega_{f\circ \varphi^{-1}}(t)$ 
there exist at most:
\begin{equation*}
    N_{\epsilon}^{\star}\eqdef \left\lceil
        \left(
            \frac{
                2^{\frac{5}{2}}
                d\omega_{\varphi}\left(\operatorname{diam}(\xxx)\right)
            }{
                (d+1)^{\frac{1}{2}}
                \omega_{\varphi}^{-1}\circ \omega_f^{-1}(2^{-2}\epsilon)
            }
        \right)^d
\right\rceil
,
\end{equation*}
probability measures $\{\mu_n\}_{n=1}^{N^{\star}_{\epsilon}}$ in $f(\xxx)\subseteq \ppp[\yyy][1]$ such that:
\begin{equation}
\max_{x \in \xxx}\, \min_{n\leq N^{\star}_{\epsilon}}
W_1 \left(
f(x),\mu_n
\right)<2^{-2}\epsilon
.
    \label{eq_lem_covering_lemma_covering_esitimate}
\end{equation}
\end{lemma}
\begin{proof}%
Since $\xxx$ is compact, then the Heine-Borel Theorem \citep[Theorem 45.1]{munkres2014topology} implies that $\operatorname{diam}(\xxx)<\infty$.  Now, since $\varphi$ is uniformly continuous, then we have the estimate:
$$
    \left\|\varphi(x_1)-\varphi(x_2)\right\|
    \leq \omega_{\varphi}(d_{\xxx}(x_1,x_2))
    \leq \omega_{\varphi}(\operatorname{diam}(\xxx))
    ;
$$
where the right-most inequality follows the fact that $\omega_{\varphi}$ is increasing.  Thus, we have the estimate:
\begin{equation}
    \operatorname{diam}(\varphi(\xxx))\leq \omega_{\varphi}(\operatorname{diam}(\xxx))
    \label{eq_lem_covering_lemma_covering_esitimate_PROOF_diameter_bound_A}
    .
\end{equation}
Now, since $\xxx$ is compact then \citep[Theorem 26.5]{munkres2014topology} implies that $\varphi(\xxx)$ is compact.  Hence, we may apply \citep[Jung's Theorem]{jung1910boundedingdiametresinEuclideanSpace} to~\eqref{eq_lem_covering_lemma_covering_esitimate_PROOF_diameter_bound_A} to obtain the inclusion:
\begin{equation*}
\begin{aligned}
        \varphi(\xxx)\subseteq \overline{B_{\rrd}\left(\bar{x},r_{\varphi}\right)}
        \mbox{ and }
        r_{\varphi}\eqdef
        \omega_{\varphi}\left(\operatorname{diam}(\xxx)\right)
        \left(\frac{d}{2(d+1)}\right)^{\frac1{2}}
\end{aligned}
    ;
\end{equation*}
for some $\bar{x}\in \rrd$.

Re-centring by the isometry $x \mapsto x-\bar{x}$ (if needed), we may, without loss of generality, assume that $\bar{x}=0$.  Let us estimate the number of metric balls of radius $\delta>0$ required to cover $\overline{B_{\rrd}(0,r_{\varphi})}$; that is, let us estimate the $\delta$-external covering number of $\overline{B_{\rrd}(0,r_{\varphi})}$.  
For any $\delta>0$, the estimate given on \citep[page 337]{shalev2014understanding} implies that the $\delta$-external covering number of $\overline{B_{\rrd}(0,r_{\varphi})}$, denoted by $N_{\delta}^{\operatorname{ext}}(\overline{B_{\rrd}(0,r_{\varphi})})$ satisfies:
\begin{equation*}
        N_{\delta}^{\operatorname{ext}}(\overline{B_{\rrd}(0,r_{\varphi})})
            \leq
        \left(
            \frac{4r_{\varphi}\sqrt{d}}{\delta}
        \right)^d
        .
\end{equation*}
This means that, for every $\delta>0$, there exist $\{\tilde{x}_n\}_{n=1}^{N_{\delta}}\subset \overline{B_{\rrd}(0,r_{\varphi})}$ such that 
\begin{equation}
    \operatorname{B}_{\rrd}(0,r_{\varphi})\subseteq \bigcup_{n=1}^{\tilde{N}_{\delta}} B_{\rrd}(\tilde{x}_n,\delta)
    \,
        \mbox{ where }
    \,
    \tilde{N}_{\delta}\eqdef 
        \left\lceil
            N_{\delta}^{\operatorname{ext}}(\overline{B_{\rrd}(0,r_{\varphi})})
        \right\rceil
.
    \label{eq_proof_first_covering_estimate}
\end{equation}
For every $n\leq N_{\delta}^{\operatorname{ext}}$, we define a new collection $\{x_{m}\}_{m=1}^{N_{\delta}}$ as follows.  For every $n\leq \tilde{N}_{\delta}^{\star}$ if $\varphi(\xxx)\cap B_{\rrd}(\tilde{x}_n',\delta)\neq \emptyset$ then pick some $x_n\in \varphi(\xxx)\cap B_{\rrd}(\tilde{x}_n',\delta)$.  Set $N_{\delta}\eqdef \#\left\{B(\tilde{x}_n',\delta)\cap \varphi(\xxx)\neq \emptyset: \, n\leq \tilde{N}_{\delta}^{\star}\right\}$.  Note that, for every $1\leq n\leq \tilde{N}_{\delta}$, the set $\operatorname{B}_{\rrd}(x_n,2\delta)$ must contain $\operatorname{B}_{\rrd}(x_n',\delta)$.  Upon relabeling,~\eqref{eq_proof_first_covering_estimate} implies that:
\begin{equation}
    \varphi(\xxx)\subseteq \bigcup_{m=1}^{N_{\delta}} B_{\rrd}(x_m
    ,2\delta)
    \,
        \mbox{ where }
    \,
    N_{\delta}\leq 
    \left\lceil
\left(
            \frac{4r_{\varphi}\sqrt{d}}{\delta}
        \right)^d
\right\rceil
.
    \label{eq_proof_first_covering_estimate_reduced}
\end{equation}
Since $\varphi$ is a continuous bijection onto its image with continuous inverse thereon and since $\varphi(\xxx)$ is compact, then the Heine-Cantor Theorem (\cite[Theorem 27.6]{munkres2014topology}) guarantees that $\varphi^{-1}$ is uniformly continuous on $\varphi(\xxx)$.  Hence, for every $x \in \xxx$, we compute:
\begin{equation}
    \begin{aligned}
\min_{m\leq N_{\delta}}\, d_{\xxx}(x,x_m^{\star}) & = 
\min_{m\leq N_{\delta}}\, d_{\xxx}(\varphi^{-1}\circ\varphi(x),\varphi^{-1}\circ \varphi(x_m^{\star}))\\
& \leq 
\min_{m\leq N_{\delta}}\, 
\omega_{\varphi^{-1}}\left(\left\|\varphi(x)-\varphi(x_m^{\star})\right\|\right)\\
& \leq \omega_{\varphi^{-1}}\left(2 \delta)\right)
.
\end{aligned}
\label{eq_proof_first_covering_estimate_reduced_covering_refinement_further_pushing_to_X}
\end{equation}
Thus, combining~\eqref{eq_proof_first_covering_estimate_reduced} and~\eqref{eq_proof_first_covering_estimate_reduced_covering_refinement_further_pushing_to_X} we find that:
\begin{equation}
    \xxx\subseteq \bigcup_{m=1}^{N_{\delta}} B_{\xxx}
                \left(
                    x_m^{\star}
                        ,
                    \omega_{\varphi^{-1}}\left(2\delta\right)
                \right)
    \,
        \mbox{ where }
    \,
    N_{\delta}\leq 
    \left\lceil
\left(
            \frac{4r_{\varphi}\sqrt{d}}{\delta}
        \right)^d
\right\rceil
.
    \label{eq_proof_first_covering_estimate_reduced_1}
\end{equation}
Therefore, by the uniform continuity of $f$ and by~\eqref{eq_proof_first_covering_estimate_reduced_1}, we have the following estimate for every $x \in \xxx$:
\begin{equation}
    \begin{aligned}
        \min_{m\leq N_{\delta}} 
            W_1 \left(
                f(x)
                    ,
                f(x_m^{\star})
            \right)
    \leq & 
        \min_{1\leq m\leq N_{\delta}} 
        \omega_{f}
        \left(
            d_{\xxx}
                \left(
                        x
                    ,
                        x_m^{\star}
                \right)
        \right)
        \\
        \leq & 
        \omega_{f}
        \left(
            \omega_{\varphi^{-1}}(2 \delta)
        \right)\\
        = & 
        \omega_{f\circ \varphi^{-1}}
        \left(
            2 \delta
        \right)
        .
    \end{aligned}
    \label{eq_proof_first_covering_estimate_in_image}
\end{equation}
Since we want the right-hand side of~\eqref{eq_proof_first_covering_estimate_in_image} to be at-most $2^{-2}\epsilon$ then, by the right-continuity of $\omega_{f\circ \varphi^{-1}}$ and by \citep[Proposition 1: (4) and (8)]{EmbrechtsHofert}, if $0<\epsilon\leq 2^{-2}\sup_{t\in [0,\infty)}\omega_{f\circ \varphi^{-1}}(t)$, then we can set:
\begin{equation}
    \delta 
        \eqdef 
    2^{-1}  \left(
                \omega_{\varphi}^{-1}\circ \omega_{f}^{-1}\left(2^{-2}\epsilon\right)
            \right)
\label{eq_proof_first_covering_estimate_value_of_delta}
.
\end{equation}
Setting $\mu_m\eqdef f(x_m^{\star})$, for $1\leq m\leq 
N_{\epsilon}^{\star}\eqdef 
N_{
2^{-1}  \left(
                \omega_{\varphi}^{-1}\circ \omega_{f}^{-1}\left(2^{-2}\epsilon\right)
            \right)
}^{\star}$, yields the conclusion.  
\end{proof}
The softmax function plays and integral role in the probabilistic transformer architecture.  The next result bounds its Lipschitz constants. 
\begin{lemma}[Lipschitz Coefficient of Softmax Function]
\label{lemma_Lipschitz_softmax}
\hfill\\
Let $N\in \nn_+$, $N\geq 2$.  The Softmax function $\operatorname{Softmax}_N$ is $\frac{\sqrt{N-1}}{N}$-Lipschitz.  
\end{lemma}
\begin{proof}[Proof of Lemma~\ref{lemma_Lipschitz_softmax}]
By the Rademacher-Stepanov Theorem \citep[Theorem 3.1.6]{Federer_GeometricMeasureTheory_1978}, $\operatorname{Lip}(\operatorname{Softmax}_N)  =\max_{x\in \rr^d}\, \|Jf\|_{op}$ where $\|\cdot\|_{op}$ is the operator norm and $Jf$ is $f$'s Jacobian.  Moreover, the operator norm of any matrix is upper-bounded by the Fr\"{o}benius norm $\|\cdot\|_F$ then, we find that
\[
    \begin{aligned}
    \max_{x\in \rrd}\,\|Jf(x)\|_{op}^2
    \leq &
    \max_{x\in \rrd}\,\|Jf(x)\|_F^2
    \\
    = &
    \max_{x\in \rrd}
    \sum_{n,m=1;n\neq m}^N \, 
        \frac{e^{2x_n}}{\Big(\sum_{i=1}^N\, e^{x_i}\Big)^2}
        \frac{e^{2x_m}}{\Big(\sum_{i=1}^N\, e^{x_i}\Big)^2}
    +
    \sum_{n=1}^N \frac{e^{2x_n}}{\Big(\sum_{i=1}^N\, e^{x_i}\Big)^2}
    \Big(1 - 
    \frac{e^{x_n}}{\sum_{i=1}^N\, e^{x_i}}
    \Big)
    \\
        \leq & 
    \max_{x\in \rrd}
    \sum_{n,m=1;n\neq m}^N \, 
        \frac{1}{N^2}
        \frac{1}{N^2}
    +
    \sum_{n=1}^N 
    \frac{1}{N}
    \Big(1 - 
    \frac{1}{N}
    \Big)
    \\
    = & \frac{\sqrt{N-1}}{N}
    .
    \end{aligned}
\]
\end{proof}

Our rates for ReLU networks will rely on the following mild extension of the main result of \cite{ShenYangZhang__OptimalApproxRatesReLUWidthDepth_2022_JPAMath}.  The following extended version of that result allows us to approximate multivariate uniformly continuous functions defined on arbitrary compact subsets of a Euclidean space using deep ReLU networks.  The approximation also makes use of the deep parallelization technique for ReLU networks introduced in \cite{FloR2021} to ensure that the approximating ReLU networks depend on few parameters (as opposed to the larger parallelization of ReLU networks described in \cite{gribonval2019approximation}).  The proof is a modification of the arguments used in \citep[Lemma 4]{KratsiosZamanlooy_2022_ReLUSigmoid} and in \citep[Proposition 3.8]{acciaio2022metric__PartnerPaper}.
\begin{lemma}[Universal Approximation Theorem for Deep ReLU Networks]
\label{lem_UAT_ReLU_ShenYangZhangExtended}
Let $K\subseteq \rr^d$ be non-empty and compact, $f:K\rightarrow \rr^N$ be uniformly continuous with modulus of continuity $\omega_f$.  Then, there exists a $\hat{f}\in \NN[d,N][\operatorname{ReLU}]$ satisfying
\[
    \max_{x\in K}\,
        \|f(x) - \hat{f}(x)\|
    < 
        \epsilon.
\]
Furthermore, $\hat{f}$ has width and depth are respectively given exactly by
\begin{enumerate}
    \item \textbf{Width}: $
            d(N-1)
        +
            3^{d+3}\max\{d,3\}
    $,
    \item \textbf{Depth}: 
    $
    N\Big(
        C_d^{(1)}
    +
        11
        \Big\lceil
                |K|^{d/2} 
                C_d^{(2)}
                \Big(
                    \omega_{f}^{\dagger}\big(
                        C_d^{(3)} 
                        \frac{\epsilon}{N^{1/2}}
                    \big)
                \Big)^{-d/2}
        \Big\rceil
    \Big)
    $,
\end{enumerate}
where, the dimensional constants are given by $C_d^{(1)}\eqdef 19 + 2d$,  $C_d^{(2)}\eqdef \big(\frac{d}{2(d+1)}\big)^{d/4}$ and $C_d^{(3)}\eqdef \frac{d^{-1/2}}{131}$.
\end{lemma}
\begin{proof}[Proof of Lemma~\ref{lem_UAT_ReLU_ShenYangZhangExtended}]
Since $K$ is non-empty and compact.  Set $R\eqdef |K| \big(\frac{d}{2(d+1)}\big)^{1/2}$.  By \citep[Jung's Theorem]{jung1910boundedingdiametresinEuclideanSpace}, there exists an $x_0\in \rr^d$ satisfying
\[
    K 
        \subseteq 
    \overline{B_{\rr^d}(x_0,R)}
        \eqdef 
    \big\{
        x\in \rr^d:\,
        \|x-x_0\|
            \leq 
        R
    \big\}
    .
\]
Therefore, there exists some $b\in \rr^d$ such that the bijective affine map $W_0\eqdef \frac1{R} \cdot (x-x_0) + b$ satisfies
$W_0(K)\subseteq [0,1]^d$.  NB, that $W_0^{-1}$ is Lipschitz with Lipschitz constant $\operatorname{Lip}(W_0^{-1})\eqdef R$; note also that $f\circ W_0^{-1}$ maps the compact subset $W_0(K)$ of $[0,1]^d$ to $\rr^N$.  In particular, $f\circ W_0^{-1}$ is uniformly continuous with modulus of continuity $t\mapsto |R|\omega_f(t)$.  

Since there is no guarantee that $f$ is continuous outside of $K$, much less $f\circ W_0^{-1}$ is uniformly continuous on $[0,1]^d$ with modulus of continuity $t\mapsto |R|\omega_f(t)$ then, since both $\rr^d$ and $\rr^N$ are separable Hilbert spaces we may apply \citep[Theorem 1.12]{BenyaminiLindenstrauss_2000_NonlinearFunctionalAnalysis} to deduce that there is a uniformly function $F:\rr^d\rightarrow \rr^D$ with modulus of continuity $t\mapsto |R|\omega_f(t)$ extending $f\circ W_0^{-1}$ to all of $\rr^d$; i.e.: for every $x\in W_0(K)$ it holds that
\begin{equation}
\label{PROOF__lem_UAT_ReLU_ShenYangZhangExtended___McShaneExtension}
    F(x) 
        = 
    f\circ W_0^{-1}(x)
    .
\end{equation}
Let $\{e_i\}_{i=1}^N$ denote the standard orthonormal basis of $\rr^N$ and, for every $i=1,\dots,N$, define the maps $\bar{F}^{(i)}:\rr^d\rightarrow \rr$ by $\bar{F}^{(i)}\eqdef \langle F(\cdot),e_i\rangle$.  By \citep[Proposition 12.28]{Combettes}, for every $i=1,\dots,N$, the map $\rr^N\ni x\mapsto \langle x, e_i\rangle$ is $1$-Lipschitz; consequentially, $t\mapsto |R|\omega_{\bar{f}}(t)$ is a modulus of continuity for each $\bar{F}^{(i)}$.  

Fix $\tilde{\epsilon}>0$.  For every $i=1,\dots,N$, by \citep[Theorem 1.1]{ShenYangZhang__OptimalApproxRatesReLUWidthDepth_2022_JPAMath} there exist a $\hat{f}^{(i)}\in \NN[d,1][\operatorname{ReLU}]$ satisfying the following uniform estimate
\begin{equation}
\label{PROOF__lem_UAT_ReLU_ShenYangZhangExtended___UnivariateExtendedVersion}
    \max_{x\in [0,1]^d}\,
    \|
        \bar{F}^{(i)}
        (x)
            -
        \hat{f}^{(i)}
        (x)
    \|
        <
    \tilde{\epsilon}
    .
\end{equation}
Furthermore, each $\hat{f}^{(i)}$ has width and depth given by
\begin{enumerate}
    \item \textbf{Width:} $3^{d+3} \max\{d,3\}$,
    \item \textbf{Depth:} $
        18 
    + 
        2d 
    +
        11
        \Big\lceil
                \big(
                    \omega_{\bar{F}^{(i)}}^{\dagger}\big(
                        \frac{\tilde{\epsilon}}{131 \sqrt{d}}
                    \big)
                \big)^{-d/2}
        \Big\rceil
    $
    .
\end{enumerate}
Applying \citep[Proposition 5]{FloR2021}, there exists a ReLU network $\tilde{f}\in \NN[d,N][\operatorname{ReLU}]$ satisfying 
\begin{equation}
\label{PROOF__lem_UAT_ReLU_ShenYangZhangExtended___DeepParallelization}
        \tilde{f}
    =
        \sum_{i=1}^N \hat{f}^{(i)} \cdot e_i
,
\end{equation}
and whose width and depth is given by
\begin{enumerate}
    \item \textbf{Width}: $
            d(N-1)
        +
            3^{d+3}\max\{d,3\}
    $,
    \item \textbf{Depth}: 
    $
    N\Big(
        19 
    + 
        2d 
    +
        11
        \Big\lceil
                \big(
                    \omega_{\bar{F}^{(i)}}^{\dagger}\big(
                        \frac{\tilde{\epsilon}}{131 \sqrt{d}}
                    \big)
                \big)^{-d/2}
        \Big\rceil
    \Big)
    $.
\end{enumerate}
Define $\hat{f}\eqdef \tilde{f}\circ W_0$ and note that $\hat{f}\in \NN[d,D][\operatorname{ReLU}]$ and that $\hat{f}$ has the same depth and width as does $\tilde{f}$, since the composition of affine maps is again an affine map.  Together,~\eqref{PROOF__lem_UAT_ReLU_ShenYangZhangExtended___McShaneExtension},~\eqref{PROOF__lem_UAT_ReLU_ShenYangZhangExtended___UnivariateExtendedVersion} and~\eqref{PROOF__lem_UAT_ReLU_ShenYangZhangExtended___DeepParallelization} imply the following uniform estimate
\allowdisplaybreaks
\begin{align}
\nonumber
        \max_{x\in K}\,
        \|
            f
            (x)
                -
            \hat{f}
            (x)
        \|
    & =
        \max_{x\in W_0(K)}\,
        \|
            f\circ W_0^{-1}
            (x)
                -
            \tilde{f}
            (x)
        \|
    \\
    \nonumber
    & =
        \max_{x\in W_0(K)}\,
        \|
            F
            (x)
                -
            \tilde{f}
            (x)
        \|
    \\
    \nonumber
    & =
        \max_{x\in W_0(K)}\,
        \|
            \sum_{i=1}^N 
                \langle F(x) - \hat{f}^{(i)}(x)
                    ,
                e_i
                \rangle
        \|
    \\
    \nonumber
    & =
        \max_{x\in W_0(K)}\,
        \Big(
            \sum_{i=1}^N 
            \|
                    \langle F(x) - \hat{f}^{(i)}(x)
                        ,
                    e_i
                    \rangle
            \|
        \Big)^{1/2}
    \\
    \nonumber
    & \leq 
        \max_{x\in [0,1]^d}\,
        \Big(
            \sum_{i=1}^N 
            \|
                    \langle F(x) - \hat{f}^{(i)}(x)
                        ,
                    e_i
                    \rangle
            \|
        \Big)^{1/2}
    \\
    \nonumber
    & \leq
        N^{{1/2}}
        \tilde{\epsilon}
    .
\end{align}
Setting $\tilde{\epsilon}
    \eqdef 
\frac{\epsilon}{N^{1/2}}
$ and using the identity 
$
\omega{\bar{f}}(t) = |R|\omega_f(t) = |K| \big(\frac{d}{2(d+1)}\big)^{1/2}\omega_f(t)
$, for all $t>0$, yields the conclusion.  
\end{proof}

\subsection{Proof of the Main Theorem}
\label{a_Proofs__MainTheorem}
We are now in place to prove the paper's main results.  Each of our main results is a consequence of the following generalization of Theorem~\ref{theorem_MAIN__GeneralCase}; which acts as the central technical result in this paper.  

We require some notation.  
For every $Q\in \nn_+$, let $\mathcal{P}_{\operatorname{fin}:Q}(\yyy)$ denote the set of all probability measures in $\mathcal{P}_1(\yyy)$ supported on at-most $Q$ points.  Define the interior (in the sense of manifold with boundary) of the $N$-simplex by 
\[
\overset{\circ}{\Delta}_N
\eqdef 
\{
w \in (0,1)^N :\, \sum_{i=1}^N \,w_n = 1
\}
.
\]
The following is a lemma refines Theorem~\ref{theorem_MAIN__GeneralCase}.  
\begin{lemma}[{Refined Version of Theorem~\ref{theorem_MAIN__GeneralCase}}]
\label{MAIN_LEMMA_theorem_UAT_qualitative}
Let $\xxx$ be a compact metric space, $\yyy$ be a path connected polish metric space, $f:\xxx\rightarrow \mathcal{P}_1(\yyy)$ be uniformly continuous with modulus of continuity $\omega_f$, and $\varphi:\xxx\rightarrow \rr^d$ satisfies condition~\ref{cond_feature}.  
Let $\mathcal{N}:(0,\infty)\rightarrow \nn$ be such that $N^{\operatorname{cov}}_{f(\xxx)}(r)\leq \mathcal{N}(r)$ for all $r>0$.
If $\sigma$ satisfies condition~\ref{condi_KL} then, for every ``quantization error'' $\epsilon_Q>0$ and 
every ``approximation error'' $0<\epsilon_A \leq 2^{-2}\sup_{t\in [0,\infty)}\omega_{f\circ \varphi^{-1}}(t)$, there exists a probabilistic transformer $\hat{T}:\xxx\rightarrow \mathcal{P}_1(\yyy)$ satisfying:
\begin{enumerate}
    \item[(i)] \textbf{Universal Approximation:}
        $
        \max_{x\in \xxx}\,
            \mathcal{W}_1\left(
                \hat{T}(x)
                    ,
                f(x)
            \right)
                \leq \epsilon_A + \epsilon_Q
            ,
        $
    \item[(ii)] \textbf{$\epsilon_Q$-Optimal Discretization:} 
    \begin{enumerate}
        \item[(a)] $
        N
            =
        \mathcal{N}^{\operatorname{cov}}_{f(\xxx)}(\epsilon_Q/3)
        \leq 
                \min\left\{
                \mathcal{N}\big(
                        \frac{\epsilon_Q}{3}
                    \big)
            ,
                \left\lceil
                        \left(
                            \frac{
                                2^{\frac{5}{2}}
                                d\omega_{\varphi}\left(
                                \left|
                                    \xxx
                                \right|
                                \right)
                            }{
                                (d+1)^{\frac{1}{2}}
                                \omega_{\varphi}^{-1}\circ \omega_f^{-1}(\frac{4}{3}\epsilon_Q)
                            }
                        \right)^d
                \right\rceil
            \right\}
        ,
        $
        \item  There exists $\mu_1,\dots,\mu_N\subseteq f(\xxx)$ such that $
        \max_{\nu\in f(\xxx)}\,
        \min_{n=1,\dots,N}\,
        	\mathcal{W}_1\big(
        	\nu
        		,
        	\mu_n
        	\big)
        	<
        	\epsilon_Q/2.
        $
        \item[(c)] $Q\eqdef 
    \min\{\tilde{Q} \in \nn_+:\,
    (\exists \hat{\mu}_{n}\in \mathcal{P}_{\operatorname{fin}:\tilde{Q}}(\yyy))\,
    \max_{n\leq N}\,
    \mathcal{W}_1\left(\mu_n,\hat{\mu}_n\right) 
        \leq 
    \frac{\epsilon_Q}{2}
    \}
    ,
    $
    \end{enumerate}
    \item[(iii)] \textbf{$\epsilon_A/2$-Optimal Metric Projection:} For every $x \in \xxx$ satisfies
    \[
    \mathcal{W}_1\Big(\hat{T}(x),f(x)\Big)
        \leq 
            \frac{\epsilon_A}{2}
        +
            \min_{n\leq N}
            \,
            \mathcal{W}_1
            \Big(
                \hat{\mu}_n
            ,
                f(x)
            \Big)
    .
    \]
\end{enumerate}
The width of $\hat{T}$ and the number of point masses is recorded in Table~\ref{tab_general_approximation_rates}, on a case-by-case basis depending on $\sigma$. 

Furthermore, if the following ``quantization rate'' is finite
\begin{equation}
\label{eq_MAIN_LEMMA_theorem_UAT_qualitative___optimalquantizationrate}
    \resizebox{0.90\hsize}{!}{$
        r^{\star}
          \eqdef 
        \min\left\{
            r>0:\,
        \forall \mu\in f(\xxx) 
        \,
            \exists C_{\mu},Q_{\mu}>0
        \,
        \forall Q\in \nn_+ 
            \mbox{ if } 
                Q\geq Q_{\mu} 
            \mbox{ then }
            \inf_{\tilde{\mu}\in \mathcal{P}_{\operatorname{fin}:Q}(\yyy)}
            \,
            \mathcal{W}_1\left(
                \mu
                ,
                \sum_{q=1}^Q w_q\mu_n
            \right)
                < 
            C_{\mu}Q^{-r}
        \right\}
    .
    $}
\end{equation}
then $Q\in \big(
2\max_{n=1,\dots,N}\, C_{\mu_n}
\big)^r
\epsilon_Q^{-r}$.
\end{lemma}
\begin{proof}[{Proof of Lemma~\ref{MAIN_LEMMA_theorem_UAT_qualitative}}]
Let $f:\xxx\rightarrow \mathcal{P}_1(\yyy)$ be uniformly continuous with modulus of continuity $\omega_f$.  
Fix a \textit{``quantization error''}
$\epsilon_Q>0$ and an approximation error $\epsilon_A \leq 2^{-2}\sup_{t\in [0,\infty)}\omega_{f\circ \varphi^{-1}}(t)
=
2^{-2} \, \omega_f\Big(\big(
C_{\varphi:L} \, t
\big)^{1/\alpha}\Big)
$.
\hfill\\
\textbf{Step 1 - Covering and Quantization}
\hfill\\
Since $f$ is continuous and $\xxx$ is compact then $f(\xxx)$ is compact.  Furthermore, since $\omega_f$ is a modulus of continuity for $f$ then the diameter of $\varphi(\xxx)$ is bounded above by
\[
\left|\varphi(\xxx)\right| 
    \leq 
\omega_{\varphi}\big(
    \big|\xxx\big|
\big)
    \leq 
C_{\varphi:U}|\xxx|^{\alpha}
.
\]
By Lemma~\ref{lem_covering_lemma} and the assumption that $\mathcal{N}$ upper-bounds $N^{\operatorname{cov}}_{f(\xxx)}(r)$, there is an $N\in \nn_+$ satisfying 
\[
    N
\leq 
    \min\left\{
        \mathcal{N}\big(
                \frac{\epsilon_Q}{3}
            \big)
    ,
        \left\lceil
                \left(
                    \frac{
                        2^{\frac{5}{2}}
                        d\omega_{\varphi}\left(
                        \left|
                            \xxx
                        \right|
                        \right)
                    }{
                        (d+1)^{\frac{1}{2}}
                        \omega_{\varphi}^{-1}\circ \omega_f^{-1}(\frac{4}{3}\epsilon_Q)
                    }
                \right)^d
        \right\rceil
    \right\}
\]
and there are probability measures $\{\mu_n\}_{n=1}^{N}$ in $f(\xxx)\subseteq \ppp[\yyy][1]$ satisfying the covering estimate
\begin{equation}
\label{PROOF___theorem_UAT_qualitative}
        \max_{x \in \xxx}\, \min_{n\leq N}
        W_1 \left(
        f(x),\mu_n
        \right)
            <
        \frac{\epsilon_Q}{3}
    .
\end{equation}

\textbf{COMMENT:} \textit{We are faced with two cases.  Either the condition in~\eqref{eq_MAIN_LEMMA_theorem_UAT_qualitative___optimalquantizationrate} holds or it does not.  Once either case is addressed separately the proofs will again be the same afterwards.  }
\begin{enumerate}
    \item \textbf{Case A -~\eqref{eq_MAIN_LEMMA_theorem_UAT_qualitative___optimalquantizationrate} does not hold:}  
    Since $\cup_{Q\in \nn_+}\, \mathcal{P}_{fin:Q}(\yyy)$ is dense in $\ppp[\yyy][1]$ (e.g.\ see the proof of \citep[Theorem 6.18]{VillaniOptTrans}) then, for every $n\leq N$ there exist some $Q_n'\in \nn_+$ and finitely supported probability measures $\sum_{q=1}^{Q_n'}\, w_q^{(n)}\, d_{y_{n,q}'}$ satisfying
    \begin{equation}
    \label{eq_quantization_step_A__Case_A}
        \max_{n=1,\dots,N}\,
            \mathcal{W}_1\big(
                \mu_n
            ,
                \sum_{q=1}^{Q_n'}\, w_q^{(n)}\, d_{y_{n,q}'}
            \big) 
                < 
            \frac{\epsilon_Q}{2}
    .
    \end{equation}
    Let $Q^{\star}\eqdef \max_{n\leq N}\, Q_{\mu_n}'$. 
    \item \textbf{Case B -~\eqref{eq_MAIN_LEMMA_theorem_UAT_qualitative___optimalquantizationrate} holds:}  
    By~\eqref{eq_MAIN_LEMMA_theorem_UAT_qualitative___optimalquantizationrate} we have that $r^{\star}$ is finite.  Therefore, there are $C_{\mu_n},Q_{\mu_n}>0$ such that: for every $Q\in \nn_+$ greater than $Q_{\mu_n}$ the following holds
    \begin{equation}
    \label{eq_quantization_step_A__Case_B}
        \max_{n\leq N}\,
        \inf_{\nu_n\in \mathcal{P}_{\operatorname{fin}:Q}(\yyy)}
        \,
        \mathcal{W}_1\left(
        \mu_n
            ,
        \nu_n
        \right)
            \leq 
        \frac{
            C_{\mu_n}
        }{
            Q^{r^{\star}}
        }
            \leq
        \frac{\epsilon_Q}{2}
        .
    \end{equation}
    Let $C^{\star}\eqdef \max_{n\leq N}\, C_{\mu_n}$ and $Q^{\star}\eqdef \max_{n\leq N}\, Q_{\mu_n}$. 
\end{enumerate}
In either case~\eqref{eq_quantization_step_A__Case_A} \textit{(resp.~\eqref{eq_quantization_step_A__Case_B} if the condition of~\eqref{eq_MAIN_LEMMA_theorem_UAT_qualitative___optimalquantizationrate} is met)} imply that for every $n\leq N$ and every $Q\in \nn_+$ with $Q\geq Q^{\star}$ there exists $\sum_{q=1}^{Q_n'}\,w_{q}^{(n)} \delta_{y_{n,q}'}\in \mathcal{P}_{\operatorname{fin}:Q}(\yyy)$ with $1\leq Q_n'\leq Q$ satisfying
\begin{equation*}
    \mathcal{W}_1\Big(
        \mu_n
            ,
        \sum_{q=1}^{Q_n'}\,w_q^{(n)} \delta_{y_{n,q}'}
    \Big)
        \leq 
    \frac{
        C_{\mu_n}
    }{
        Q^{r^{\star}}
    }
        \leq 
    \epsilon_Q
    .
\end{equation*}
Since $\yyy$ was assumed to be a path-connected space then so is $\mathcal{P}_1(\yyy)$.  Therefore, for every $\delta>0$, we may choose distinct $y_{1}^{(1)}$, $\dots$, $y_{Q}^{(N)}$ in $\yyy$ such that
$
\mathcal{W}_1\Big(
        \sum_{q=1}^{Q_n'}\,w_{q}^{(n)} \delta_{y_{q}^{(n)}}
            ,
        \sum_{q=1}^{Q_n'}\,w_{q}^{(n)} \delta_{y_{n,q}'}
    \Big) 
    \leq 
    \delta.
$
Thus, for every $n\leq N$, if $Q_n'<Q$ then for each $Q_n'<q\leq Q$ we define $w_{q}^{(n)}=0$ and we observe that $(w_{q}^{(n)})_{q=1}^Q\in \Delta_Q$.  Therefore, it holds that
\[
\mathcal{W}_1\Big(
        \sum_{q=1}^{Q_n}\,w_{q}^{(n)} \delta_{y_{q}^{(n)}}
            ,
        \sum_{q=1}^{Q_n'}\,w_{q}^{(n)} \delta_{y_{n,q}'}
    \Big) 
=
\mathcal{W}_1\Big(
        \sum_{q=1}^{Q_n'}\,w_{q}^{(n)} \delta_{y_{q}^{(n)}}
            ,
        \sum_{q=1}^{Q_n'}\,w_{q}^{(n)} \delta_{y_{n,q}'}
    \Big) 
    \leq 
    \delta.
\]
Set $\delta \eqdef \frac{
        C_{\mu_n}
    }{
        Q^{r^{\star}}
    }
    .
$
Therefore, for every $Q\in \nn_+$ with $Q\geq  
\max
\big\{
Q^{\star} 
    ,
\,(\frac{3}{2}C)^{1/r^{\star}}
\epsilon_Q^{-1/r^{\star}}
\big\}$ 
there exist $\hat{\mu}_1$, $\dots$, $\hat{\mu}_N\in \mathcal{P}_{\operatorname{fin}:Q}(\yyy)$ in general position each of whose support has \textit{exactly} $Q$ atoms such that, and such that
\begin{equation}
\label{PROOF___theorem_UAT_qualitative____eq_quantization_step_C}
    \max_{n\leq N}
    \,
        \mathcal{W}_1\Big(
        \mu_n
            ,
        \hat{\mu}_n
    \Big)
        \leq 
    \frac{2\epsilon_Q}{3}
    .
\end{equation}
Combining the quantization estimate in~\eqref{PROOF___theorem_UAT_qualitative____eq_quantization_step_C} and the covering estimate on the left-hand side of~\eqref{PROOF___theorem_UAT_qualitative} implies that
\begin{equation}
\label{PROOF___theorem_UAT_qualitative____STEP_CoveringAndQuantization}
    \max_{x \in \xxx}\, \min_{n\leq N}
        \mathcal{W}_1
        \left(
            f(x)
                ,
            \hat{\mu}_n
        \right)
            \leq 
        \epsilon_Q
    .
\end{equation}
\textbf{Comment:} \textit{In the next step, we will approximately metrically project on the convex hull of measures $\{\hat{\mu}_n\}_{n=1}^N$.}
\hfill\\
\textbf{Step 2 - Approximate Metric Projection}
\hfill\\
For simplicity, denote $\bar{Y}
\eqdef 
\Phi(\operatorname{hull}(\hat{\mu}_n)_{n=1}^N)$.  Consider an \textit{''approximation error''} $\epsilon_A>0$.

By construction, the family $\left\{\hat{\mu}_n\eqdef \sum_{q=1}^{Q_n}\,w_{q}^{(n)} \delta_{y_{q}^{(n)}}\right\}_{n=1}^N$ is in general position.  Therefore, the map $\psi:\Delta_N\ni w\mapsto \sum_{n=1}^N\,w_n \hat{\mu}_n$ is injective.  Since $\Delta_N$ is compact then $\psi$ is a homeomorphism onto its image.  
Thus, the set $\Phi\circ \psi(\Delta_N)$ is a convex subset of the finite-dimensional linear subset of $\text{\AE}(\yyy,y_0)$ spanned by the set $\{\Phi(\hat{\mu}_n)\}_{n=1}^N$.  
Let $\epsilon_A>0$, we may apply \citep[Corollary 1.2]{ModulusOfContinuityAdditiveApproxProjection1994Albrekht} to conclude that there is some $c>0$ (depending only on $\bar{Y}$ and \textit{independant of} $\epsilon_A$) such that there is a uniformly-continuous map $\Pi_{\epsilon_A}:\text{\AE}(\yyy,y_0)\rightarrow \bar{Y}$ satisfying the following $\epsilon_A/2$-approximate metric projection property
\begin{equation}
\label{PROOF___theorem_UAT_qualitative_A}
    \sup_{F \in \text{\AE}(\yyy,y_0)}\,
        \left\|
            F - \Pi_{\epsilon_A}(F)
        \right\|_{\text{\AE}(\yyy,y_0)}
        -
        \inf_{\tilde{F} \in \bar{Y}}
        \,
            \left\|
            F - \tilde{F}
        \right\|_{\text{\AE}(\yyy,y_0)}
        \leq \epsilon_A/2
        .
\end{equation}
Furthermore, the following is a modulus of continuity for $\Pi_{\epsilon_A}$
\begin{equation}
\label{PROOF___theorem_UAT_qualitative____modulus_of_projection}
        \omega_{\Pi_{\epsilon_A}} (t) 
            \eqdef 
        c\,\Big(1 - \frac{2}{\epsilon_A}\Big) t
        ,
\end{equation}
where $c$ is a constant depending only on $\Phi(\operatorname{hull}\{\hat{\mu}_n\}_{n=1}^N)\eqdef 
\{
\pp\in \mathcal{P}_1(\yyy):\, (\exists w \in \Delta_N)\, 
\pp = \sum_{n=1}^N \, w_n \hat{\mu}_n
\}
$.

\textbf{Comment:} 
\textit{``We will pull-back this approximate metric projection of $\text{\AE}(\yyy,y_0)$ on $\Phi(\operatorname{hull}(\hat{\mu}_{n=1}^N))$ to an approximate metric projection of $\mathcal{P}_1(\yyy)$ on $\operatorname{hull}(\hat{\mu}_{n=1}^N)$''.}

Since $\Pi_{\epsilon_A}(\text{\AE}(\yyy,y_0))\subseteq \Phi(\mathcal{P}_1(\yyy))$ then, the map $\Pi^{\epsilon_A}\eqdef \Phi^{-1}\circ \Pi_{\epsilon_A}\circ \Phi$ is well-defined.  
By~\eqref{PROOF___theorem_UAT_qualitative____modulus_of_projection} and the fact that $\Phi$ (and therefore $\Phi^{-1}$ also) is an isometry then $\Pi^{\epsilon_A}$ is uniformly continuous and $\omega_{\Pi_{\epsilon_A}}$ is also a modulus of continuity for the map $\Pi^{\epsilon_A}$.  Furthermore,~\eqref{PROOF___theorem_UAT_qualitative_A} implies that, for every $\mu \in \mathcal{P}_1(\yyy)$
\[
\begin{aligned}
    \mathcal{W}_1\Big(
        \mu
    ,
        \Pi^{\epsilon_A}(\mu)
    \Big)
        = &
    \Big\|
        \Phi(\mu)
    -
        \Phi(\Pi^{\epsilon_A}(\mu))
    \Big\|_{\text{\AE}(\yyy,y_0)}
    \\
    = & 
    \Big\|
        \Phi(\mu)
    -
        \Phi(\Phi^{-1}\circ \Pi_{\epsilon_A}\circ \Phi(\mu))
    \Big\|_{\text{\AE}(\yyy,y_0)}
    \\
    = & 
    \Big\|
        \Phi(\mu)
    -
         \Pi_{\epsilon_A}\big( \Phi(\mu)\big)
    \Big\|_{\text{\AE}(\yyy,y_0)}
    \\
    \leq  &
        \inf_{\tilde{F} \in \bar{Y}}
        \,
            \left\|
            \Phi(\mu) - \tilde{F}
        \right\|_{\text{\AE}(\yyy,y_0)}
        +\epsilon_A/2
    \\
    = & 
    \inf_{\nu\in \operatorname{hull}(\{\hat{\mu}\}_{n=1}^N}
        \,
    \mathcal{W}_1\Big(
        \mu
        ,
        \nu
        \Big)
        +\epsilon_A/2
.
\end{aligned}
\]
Thus,~\eqref{PROOF___theorem_UAT_qualitative_A} implies that: for every $x\in f(\xxx)$ it holds that
\begin{equation}
\label{PROOF___theorem_UAT_qualitative__in_Wasserstein_Space}
    \mathcal{W}_1\Big(
        \mu
    ,
        \Pi^{\epsilon_A}(\mu)
    \Big)
-
    \inf_{\nu\in \operatorname{hull}\big(\{\hat{\mu}_n\}_{n=1}^N}\big)
        \,
    \mathcal{W}_1\Big(
        \mu
        ,
        \nu
    \Big)    
\leq 
    \frac{\epsilon_A}{2}
.
\end{equation}

\hfill\\
\textbf{Step 3 - Translating the Approximation Problem to the $N$-Simplex $\Delta_N$}
\hfill\\
Define the barycenter $\bar{\Delta}_N$ of the $N$-simplex by $\bar{\Delta}_N\eqdef (1/N,\dots,1/N) \in \Delta_N$ and define the map
\[
h:[0,1]\times \Delta_N 
\ni (t,w)
\mapsto 
t(w - \bar{\Delta}_N) + \bar{\Delta}_N
\in \Delta_N
.
\]
Define the $1$-Lipschitz map $H\eqdef h(t^{\star},\cdot)$ where $[0,1]\ni t^{\star}
\eqdef 
    \max\{ t\in [0,1]:\, \|H(t,\Delta_N)-\Delta_N\| 
        \leq 
    \epsilon_A/2
    \}$
and where $\bar{\Delta}_N = \big(\frac{1}{N},\dots,\frac{1}{N}\big)\in \Delta_N$. In particular, $\omega_H^{-1}(t)=t$.  Furthermore, since $N>1$ then, $t^{\star}$ is given by
\[
    t^{\star}
=
    1 - \frac{\epsilon_A}{2 \|\partial \Delta_N - \bar{\Delta}_N\|}
=   
    1 - \frac{\epsilon_A}{
    2 
    \|\bar{\Delta}_N - (1,0,\dots)\|
    }
=
    \frac{
        (N - 1) - \sqrt{N}\epsilon_A
    }{
        2(N-1)
    }
.
\]
Define the affine isometry $W:\mathbb{R}^{N-1}\ni z \mapsto (z_1,\dots,z_{N-1},1)\in \mathbb{R}^N$ and define the map
\[
\rho: \mathbb{R}^{N-1}\ni z \mapsto \operatorname{Softmax}_N \circ W(z) \in \overset{\circ}{\Delta}_N
.
\]
By Lemma~\ref{lemma_Lipschitz_softmax} $\operatorname{Softmax}_N$ is $\frac{\sqrt{N-1}}{N}$-Lipschitz and therefore $\rho$ is $\frac{\sqrt{N-1}}{N}$-Lipschitz and since $W$ is an isometry. Moreover, it has the following continuous right-inverse
\[
R: \overset{\circ}{\Delta}_N\ni y \mapsto  (\ln(y_n) - \ln(y_N) + 1)_{n=1}^{N-1}  \in \rr^{N-1}.  
\]
Furthermore, by the Rademacher-Stepanov Theorem (see \citep[Theorem 3.1.6]{Federer_GeometricMeasureTheory_1978}), $R$ is $\max_{u\in H_{t^{\star}}(\Delta_N)}\, \|\nabla R(u)\|=\frac{1}{t^{\star}}$-Lipschitz on $H_{t^{\star}}(\Delta_N)$.  Therefore,
\[
    \omega_R^{-1}(t)
        =
    \frac{
        (N - 1) - \sqrt{N}\epsilon_A
    }{
        2(N-1)
    }\,
        t
    .
\]
Next, the bi-H\"{o}lder regularity of the feature map $\varphi:\xxx\rightarrow \rr^d$ given by Assumption~\ref{cond_feature}, implies that $\varphi^{-1}$ has modulus of continuity $\omega_{\varphi^{-1}}(t)=\frac1{C_{\varphi:L}^{1/\alpha}} t^{1/\alpha}$.  
Consequentially,
\[
    \omega_{\varphi^{-1}}^{-1}(t)
        =
    C_{\varphi:L} \, 
    t^{\alpha}
    .
\]
Moreover, by Lemma~\ref{lem_Lipschitz_constant_parameterization} $\psi^{-1}$ is well-defined and has modulus of continuity $\omega_{\psi^{-1}}=Ct$ for some constant $C>0$ depending only on $\{\hat{\mu}_n\}_{n=1}^N$.  
Consider the map
\[
\bar{f}
    \eqdef 
R\circ H\circ \psi^{-1}\circ \Pi_{\epsilon_A} \circ \Phi \circ f \circ \varphi^{-1}
: \varphi(\xxx)\rightarrow \rr^{N-1}
.
\]
Putting it all together, and using the fact that $0<\epsilon_A<2$, we have that $\bar{f}$ is uniformly continuous and its modulus of continuity $\omega_{\bar{f}}$ satisfying: for every $t>0$
\[
\begin{aligned}
    \omega_{\bar{f}}^{-1}(t)
        \eqdef &
    C_{\varphi:L}\Big(
        \omega_f^{-1}\Big(
            \frac{\epsilon_A}{c(\epsilon_A-2)}
            \frac{1}{C}
            \frac{
                (N - 1) - \sqrt{N}\epsilon_A
            }{
                2(N-1)
            }
            \,
            t
        \Big)
    \Big)^{\alpha}
.
\end{aligned}
\]
In particular, since the inverse of a monotone increasing function (such as $\omega_f$) is monotone increasing then, for any $0<\epsilon_A\leq 1$ we deduce the following bounds on $\omega_f^{-1}$: for every $t>0$
\begin{equation}
\label{PROOF___theorem_UAT_qualitative____BoundsOnModulus}
        C_{\varphi:L}\Big(
        \omega_f^{-1}\Big(
            \tilde{C}
            \frac{\epsilon_A}{(\epsilon_A-2)}
            \frac{
                N - \sqrt{N}- 1
            }{
                N
            }
            \,
            t
            \Big)
        \Big)^{\alpha}
    \leq 
        \omega_{\bar{f}}^{-1}(t) 
    \leq 
        C_{\varphi:L}\Big(
        \omega_f^{-1}\Big(
            \tilde{C}
            \frac{\epsilon_A}{(\epsilon_A-2)}
            \,
            t
        \Big)
    \Big)^{\alpha}
;
\end{equation}
where $\tilde{C}\eqdef \frac{1}{2Cc}$; consequentially, $\tilde{C}$ depends only on $\{\hat{\mu}_n\}_{n=1}^N$.
If $\sigma=\operatorname{ReLU}$ then Lemma~\ref{lem_UAT_ReLU_ShenYangZhangExtended} applies, or if $\sigma\in C^{\infty}(\rr)$ or $\sigma \in C(\rr)$ is polynomial of degree at-least $2$ then, \citep[Theorem 10]{paponkratsios2021quantitative} applies when furthermore $\sigma \in C(\rr)$ simply satisfies the Kidger-Lyons condition (Assumption~\ref{condi_KL}) then \citep[Theorem 3.2]{kidger2019universal} applies.  
In either case, depending on $\sigma$, for any $\epsilon_A'>0$, there exists a deep feedforward network $\hat{f}^{(1)}:\rr^d\rightarrow \overset{\circ}{\Delta}_N$ satisfying the following uniform estimate on $\varphi(\xxx)$
\begin{equation}
\label{PROOF___theorem_UAT_qualitative____UAT_application_A}
    \max_{u\in \varphi(\xxx)}\,
    \Big\|
    \hat{f}^{(1)}(u) - \bar{f}(u)
    \Big\|_2
        < 
    \epsilon_A' 
    ;
\end{equation}
furthermore, using the lower-bound on $\omega_{\bar{f}}^{-1}$ in~\eqref{PROOF___theorem_UAT_qualitative____BoundsOnModulus}, $\hat{f}$ has the following width and depths:
\begin{enumerate}
\item \textbf{Case 1 - $\sigma = \operatorname{ReLU}$:} 
\begin{enumerate}
    \item[(a)] \textbf{Width}: $
            d(N-1)
        +
            3^{d+3}\max\{d,3\}
    $,
    \item[(b)] \textbf{Depth}: 
        $
        N\Big(
            C_d^{(1)}
        +
            11
            \Big\lceil
                    \tilde{C}_d^{(2)} 
                    |\xxx|^{\alpha d/2} 
                     \Big(
                    \omega_f^{-1}\Big(
                        \tilde{C}_d^{(3)}  
                        \,
                        \frac{\epsilon_A^2}{(\epsilon_A-2)}
                        \,
                        \frac{
                            N - \sqrt{N}- 1
                        }{
                            N^{3/2}
                        }
                        \Big)
                    \Big)^{-\alpha d/2}
            \Big\rceil
        \Big)
        $,
    \end{enumerate}
    where $\tilde{C}_d^{(2)} \eqdef  C_d^{(d)} C_{\varphi:L}^{-\alpha d /2} C_{\varphi:U}$, $\tilde{C}_d^{(3)}\eqdef \tilde{C} C_d^{(3)}$, 
    $C_d^{(1)}$, $C_d^{(2)}$, and $C_d^{(3)}$ are as in Lemma~\ref{lem_UAT_ReLU_ShenYangZhangExtended} and thus they only depend on the embedding dimension $d$ and on $\varphi$'s regularity.  
\item \textbf{Case 2 - $\sigma \in C^{\infty}(\rr)$ and satisfies \ref{condi_KL}:} $\hat{f}$ has width $2+N+d$ and depth at-most
    \[
    \mathscr{O}\left(
        \frac{
            C_{\varphi,d}\,
            N 
	        \big(
	            |\xxx|^{2d/\alpha}
	        \big)
        }{
            \omega_f^{-1}\Big(
                C_{d,\hat{\mu}_{\cdot}}
                \frac{
                    \epsilon_A
                }{
                    |2-\epsilon_A|
                }
                \frac{
                    N - \sqrt{N}- 1
                }{
                    N^{1+2d}
                }
                \,
                    \big(
                		\epsilon_A'
                	\big)^{2d}
                \Big)
            \Big)^{\alpha}
        }
    \right)
    ,
    \]
where, $
C_{d,\hat{\mu}_{\cdot}}
    \eqdef 
(C_2(1+\frac1{4} d))^{2d}\tilde{C}$, $
C_{\varphi,d}
    \eqdef 
\frac{2^{2d}C_{\varphi:U}^{2d}}{C^{2d}C_{\varphi:L}}
            $ where $C>0$ is an absolute constant.
\item \textbf{Case 3 - $\sigma\in C(\rr)$ is a polynomial of degree at-least $2$:} 
    \[
    \mathscr{O}\left(
        \frac{
            C_{\varphi,d}\,
            N 
	        \big(
	            |\xxx|^{(4d+2)/\alpha}
	        \big)
        }{
            \omega_f^{-1}\Big(
                C_{d,\hat{\mu}_{\cdot}}
                \frac{
                    \epsilon_A
                }{
                    |2-\epsilon_A|
                }
                \frac{
                    N - \sqrt{N}- 1
                }{
                    N^{3+4d}
                }
                \,
                    \big(
                		\epsilon_A'
                	\big)^{4d+2}
                \Big)
            \Big)^{\alpha}
        }
    \right)
    ,
    \]
    where, $
C_{d,\hat{\mu}_{\cdot}}
    \eqdef 
(C_2(1+\frac1{4} d))^{4d+2}\tilde{C}$, $
C_{\varphi,d}
    \eqdef 
\frac{2^{4d+2}C_{\varphi:U}^{4d+2}}{C^{4d+2}C_{\varphi:L}}$ 
where $C>0$ is an absolute constant,
\item \textbf{Case 4 - $\sigma\in C(\rr)$ and satisfies \ref{condi_KL}:} 
    $\hat{f}$ has width $2+N+d$.
\end{enumerate}
Since $\varphi$ is a homeomorphism then it is a bijection between $\varphi(\xxx)$ and $\xxx$ and therefore~\eqref{PROOF___theorem_UAT_qualitative____UAT_application_A} implies 
\allowdisplaybreaks
\begin{equation}
\label{PROOF___theorem_UAT_qualitative____UAT_application}
\resizebox{0.92\hsize}{!}{$
\begin{aligned}
\sup_{x\in \xxx}\,
    \left\|
        \rho\circ \hat{f} \circ \varphi(x) 
            - 
        \psi^{-1}\circ \Pi_{\epsilon_A}\circ \Phi \circ f(x)
    \right\|
        = &
    \sup_{x\in \xxx}\,
    \left\|
        \rho\circ \hat{f} \circ \varphi(x) 
            - 
        \rho\circ R \circ \psi^{-1}\circ \Pi_{\epsilon_A}\circ \Phi \circ f\circ \varphi^{-1}\circ \varphi(x)
    \right\|
    \\
    \leq &
        \sup_{x\in \xxx}\,
    \operatorname{Lip}(\rho)
    \left\|
        \hat{f} \circ \varphi(x) 
            - 
        R \circ \psi^{-1}\circ \Pi_{\epsilon_A}\circ \Phi \circ f\circ \varphi^{-1}\circ \varphi(x)
    \right\|
    \\
    \leq  &
        \sup_{u\in \varphi(\xxx)}\,
    \frac{\sqrt{N-1}}{N}
    \left\|
        \hat{f}(u)
            - 
         R \circ \psi^{-1}\circ \Pi_{\epsilon_A}\circ \Phi \circ f\circ \varphi^{-1}(u)
    \right\|
    \\
    \leq &
    \frac{\sqrt{N-1}}{N}
        \sup_{u\in \varphi(\xxx)}\,
    \left\|
        \hat{f}(u)
            - 
         R \circ H\circ \psi^{-1}\circ \Pi_{\epsilon_A}\circ \Phi \circ f\circ \varphi^{-1}(u)
    \right\|
    \\
    & +
    \frac{\sqrt{N-1}}{N}
    \sup_{u\in \varphi(\xxx)}\,
    \left\|
        R \circ H\circ \psi^{-1}\circ \Pi_{\epsilon_A}\circ \Phi \circ f\circ \varphi^{-1}(u)
            - 
         R \circ \psi^{-1}\circ \Pi_{\epsilon_A}\circ \Phi \circ f\circ \varphi^{-1}(u)
    \right\|
    \\
    &
    =
    \frac{\sqrt{N-1}}{N}
    \sup_{u\in \varphi(\xxx)}\,
    \left\|
        \hat{f}(u)
            - 
         \bar{f}(u)
    \right\|
    \\
    & +
    \frac{\sqrt{N-1}}{N}
    \sup_{u\in \varphi(\xxx)}\,
    \left\|
        R \circ H\circ \psi^{-1}\circ \Pi_{\epsilon_A}\circ \Phi \circ f\circ \varphi^{-1}(u)
            - 
         R \circ \psi^{-1}\circ \Pi_{\epsilon_A}\circ \Phi \circ f\circ \varphi^{-1}(u)
    \right\|
    \\
    & \leq 
    \frac{\sqrt{N-1}}{N}
    \epsilon_A' 
    + 
    \frac{\sqrt{N-1}}{N}
    \operatorname{Lip}(R)
    \sup_{z \in \Delta_N}\,
    \left\|
        h_{t^{\star}}(z) - z
    \right\|
    \\ 
    & \leq 
    \frac{\sqrt{N-1}}{N}\left(
            \epsilon_A' + \frac{\epsilon_A}{2}
        \right)
    \\
    & = 
    \epsilon_A
    .
\end{aligned}
$}
\end{equation}
Since the composition of affine maps is again affine, then $\hat{f}\eqdef W\circ \hat{f}^{(1)}:\rr^d \rightarrow \rr^N$ is a feedforward network.  
Set $\epsilon_A'\eqdef \frac{N}{2\sqrt{N-1}}\epsilon_A$.  Thus, the following is a probabilistic transformer
\[
    \hat{T}
\eqdef 
    \sum_{n=1}^N 
    \,
    (\operatorname{Softmax}_N \circ \hat{f}\circ \varphi(\cdot))_n
    \,
    \hat{\mu}_n
=
    \Phi^{-1}\circ \psi \circ \rho \circ \hat{f}\circ \varphi
,
\]
\textbf{Comment:} \textit{``We now verify that $\hat{T}$ implements an $\epsilon_A/2$-optimal metric projection of $f(\xxx)$ onto the hull of the set of empirical measures $\{\hat{\mu}_n\}_{n=1}^N$.''}
\hfill\\
\textbf{Step 4 - $\epsilon$-Optimal Metric Projection}
\hfill\\
Therefore,~\eqref{PROOF___theorem_UAT_qualitative____UAT_application} implies that
\begin{equation}
\label{PROOF___theorem_UAT_qualitative____UAT_application_B}
\resizebox{0.92\hsize}{!}{$
\begin{aligned}
    \sup_{x\in \xxx}\, 
        \mathcal{W}_1
        \Big(
            \hat{T}(x)
        ,
            \Pi^{\epsilon_A}\circ f(x)
        \Big)
    = &
        \sup_{x\in \xxx}\, 
        \mathcal{W}_1
        \Big(
            \hat{T}(x)
        ,
            \Phi^{-1}\circ \Pi_{\epsilon_A}\circ \Phi\circ f(x)
        \Big)
    \\
    = &
        \sup_{x\in \xxx}\, 
        \Big\|
            \psi\circ \rho\circ \hat{f}\circ \varphi(x)
        -
            \Phi\circ \Phi^{-1}\circ \Pi_{\epsilon_A}\circ \Phi\circ f(x)
        \Big\|_{\text{\AE}}
    \\
        = &
        \sup_{x\in \xxx}\, 
        \Big\|
            \psi\circ \rho\circ \hat{f}\circ \varphi(x)
        -
            \psi\circ \psi^{-1}\circ \Pi_{\epsilon_A}\circ \Phi\circ f(x)
        \Big\|_{\text{\AE}}
    \\
    \leq 
    &
        \sup_{x\in \xxx}\, 
        \operatorname{Lip}(\psi)
        \Big\|
            \rho\circ \hat{f}\circ \varphi(x)
        -
            \psi^{-1}\circ \Pi_{\epsilon_A}\circ \Phi\circ f(x)
        \Big\|_{2}
    \\
    \leq &
    \epsilon_A
    .
\end{aligned}
$}
\end{equation}
\textbf{Comment:} \textit{It remains to show that the probabilistic transformer $\hat{T}$ approximates $f$ on $\xxx$ with uniform error $\epsilon_A+\epsilon_Q$.}
\hfill\\
\textbf{Step 5- Verifying The Approximation of $f$ on $\xxx$}
\hfill\\
Combining the estimates in~\eqref{PROOF___theorem_UAT_qualitative____STEP_CoveringAndQuantization},~\eqref{PROOF___theorem_UAT_qualitative__in_Wasserstein_Space}, and in~\eqref{PROOF___theorem_UAT_qualitative____UAT_application_B}: for every $x\in \xxx$ we have that
\begin{align}
\label{PROOF___theorem_UAT_qualitative____CONCLUSION_A}
\mathcal{W}_1
        \Big(
            \hat{T}(x)
        ,
            f(x)
        \Big)
    & \leq 
        \mathcal{W}_1
        \Big(
            \hat{T}(x)
        ,
            \Pi^{\epsilon_A}\circ f(x)
        \Big)
    +
        \mathcal{W}_1
        \Big(
            \Pi^{\epsilon_A}\circ f(x)
        ,
            f(x)
        \Big)
    \\
    \nonumber
    & \leq 
        \mathcal{W}_1
        \Big(
            \hat{T}(x)
        ,
            \Pi^{\epsilon_A}\circ f(x)
        \Big)
    +
        \inf_{\nu \in \operatorname{hull}(\{\hat{\mu}_n\}_{n=1}^N)}
        \,
        \mathcal{W}_1
        \Big(
            \nu
        ,
            f(x)
        \Big)
    + 
        \epsilon_A/2
    \\
    \nonumber
    & \leq 
        \mathcal{W}_1
        \Big(
            \hat{T}(x)
        ,
            \Pi^{\epsilon_A}\circ f(x)
        \Big)
    +
        \min_{n\leq N}
        \,
        \mathcal{W}_1
        \Big(
            \hat{\mu}_n
        ,
            f(x)
        \Big)
    + 
        \epsilon_A/2
    \\
    \label{PROOF___theorem_UAT_qualitative____CONCLUSION_B}
    & \leq 
        \frac{\sqrt{N-1}}{N}\Big(
                \epsilon_A' + \frac{\epsilon_A}{2}
            \Big)
    +
        \min_{n\leq N}
        \,
        \mathcal{W}_1
        \Big(
            \hat{\mu}_n
        ,
            f(x)
        \Big)
    \\
    \label{PROOF___theorem_UAT_qualitative____CONCLUSION_C}
    & \leq 
        \epsilon_A
    +
        \epsilon_Q
    .
\end{align}
Together~\eqref{PROOF___theorem_UAT_qualitative____CONCLUSION_A} and~\eqref{PROOF___theorem_UAT_qualitative____CONCLUSION_B} yield the first conclusion and together~\eqref{PROOF___theorem_UAT_qualitative____CONCLUSION_A} and~\eqref{PROOF___theorem_UAT_qualitative____CONCLUSION_C} yield the second.  
\end{proof}


\begin{proof}[Proof of Theorem~\ref{theorem_unit_ball}]
By \citep[Theorem 2.2.11 and Example 2.2.1]{panaretos2020invitation} $\mathcal{N}\big(
            \frac{\epsilon_Q}{3}
        \big)$ is of the order of $\mathcal{O}\Big(
        \exp\big(
            -\log(\epsilon/3) 
            C_D 
            (\epsilon/3)^{-d}
        \big)
    \Big)$.    Moreover, by \citep[Theorem 3]{chevallier2018uniform} $Q=\Big\lceil
     \big(\epsilon_Q\frac{D-1}{4D}\big)^{-D}
    \Big\rceil$.  Applying Lemma~\ref{MAIN_LEMMA_theorem_UAT_qualitative} yields the conclusion.  
\end{proof}
\begin{proof}[Proof of Theorem~\ref{theorem_pmoments}]
By \citep[Theorem 2.2.11 and Example 2.2.3]{panaretos2020invitation} $\mathcal{N}\big(
            \frac{\epsilon_Q}{3}
        \big)$ is of the order of 
    \[
        \mathcal{O}\Big(
        \exp\big(
            -\log(\epsilon/3) 
            C_D 
            (\epsilon/3)^{-d}
        \big)
        \Big).
    \]
    Moreover, by \citep[Corollary 1]{chevallier2018uniform} if $\epsilon_Q = 
    \big\lceil
        \epsilon_Q^{-1/D} \, (\frac{2(D-1)}{D})^D
    \big\rceil
    $ then we may make $\frac{\epsilon_Q}{2}$ is $o\big(
        \epsilon_Q^{-D + D/p}
        \big(\frac{2^{1-1/p}(D-1)}{D}\big)^D
    \big)$.
    Applying Lemma~\ref{MAIN_LEMMA_theorem_UAT_qualitative} and defining $C_{D,p} \eqdef \big(\frac{2^{1-1/p}(D-1)}{D}\big)^D$ 
    yields the conclusion.  
\end{proof}
\begin{proof}[Proof of Theorem~\ref{theorem_MAIN__AhlforsRegularCase}]
Suppose that $f(\xxx)\subseteq \mathcal{AP}^{C,q}(\rrD)$ then \citep[Proposition 5.1]{KloecknerQuantizationAhlforsRegular2012} guarantees that $r^{\star} = \frac{1}{q}$; where $r^{\star}$ is defined in~\eqref{eq_MAIN_LEMMA_theorem_UAT_qualitative___optimalquantizationrate}.  
A fortiori, 
\[
    \sup_{\pp \in \mathcal{AP}^{C,q}(\rr^D)}
    \min_{
    \underset{\#\operatorname{supp}(\hat{\pp})\leq Q}{
        \hat{\pp} \in \mathcal{P}_1(\rr^D)
    }
    }
    \,
    \mathcal{W}_1\big(
        \pp
    ,
        \hat{\pp}
    \big)
        \leq 
    \big(\frac{5}{C^{1/q}}\big)
        N^{-1/q}
    .
\]
Thus, the conclusion follows from Lemma~\ref{MAIN_LEMMA_theorem_UAT_qualitative}.  
\end{proof}
\begin{proof}[{Proof of Corollary~\ref{cor_lowerboundeddensity}}]
Let $x\in \xxx$, $r>0$, and $u\in \rr^D$ be such that $B(x,r)\subseteq [0,1]^d\cap H$.  The assumptions on $f$ imply that
\[
\begin{aligned}
        \mu(B(u,r))
    = &
        \int_{y\in B(u,r)}\,
            1
            f(dy)
    \\
    = &
        \int_{y\in B(u,r)}\,
            \frac{df(x)}{d\mu}(y)
            \,
            \mu(dy)
    \\
    \geq &
        \int_{y\in B(u,r)}\,
            C\,
            \mu(dy)
    \\
    = &
        C\mu(B(u,r))
    \\ 
    = &
        C
        \frac{\pi^{q/2}}{\Gamma(\frac{q}{2}+1)}
        r^q
    .
\end{aligned}
\]
\end{proof}
%
\begin{proof}[{Proof of Theorem~\ref{theorem_MAIN__LipschitzManifoldCase}}]
As shown on \citep[page 612]{Navas_2013_CanonicalBarycenterMap} the map from $\mathcal{P}_1(\rr^D)$ to $\rr^D$ defined by $\pp\mapsto \mathbb{E}_{Y\sim \pp}[Y]$ is a $1$-Lipschitz left-inverse of the map $x\mapsto \delta_x$ (here $\mathbb{E}_{T\sim \pp}[Y]$ denotes the Bochner integral on $\rr^D$).  Consequentially,~\eqref{eq_biLipshitz_condition} implies that
\begin{equation}
\label{PROOF_genstatmanifold_eq_biLipshitz_condition__lowerbound}
        \frac1{C}\,
        \|\theta_1-\theta_2\|
    \leq 
        C\frac1{C}\,
        \big\|
            \mathbb{E}_{Y_1\sim \pp_{\theta_1}}[Y_1]
                -
            \mathbb{E}_{Y_2\sim \pp_{\theta_2}}[Y_2]
        \big\|
    \leq 
        \mathcal{W}_1\big(
            \pp_{\theta_1}
                ,
            \pp_{\theta_2}
        \big)
    .
\end{equation}
Conversely, fix $\theta_1,\theta_2\in \Theta$, let $\operatorname{DIAG}:\rr^{D}\rightarrow \rr^D\times \rr^D$, and define the coupling $\pi^{\star}\eqdef \big((p_{\theta_1},p_{\theta_2})\circ \operatorname{DIAG}\big)_{\#} \pp$ between $\pp_{\theta_1}$ and $\pp_{\theta_2}$.  By definition of the Wasserstein $1$ distance, we have that
\begin{equation}
\label{PROOF_genstatmanifold_eq_biLipshitz_condition__upperbound}
\begin{aligned}
        \mathcal{W}_1\big(
            \pp_{\theta_1}
                ,
            \pp_{\theta_2}
        \big)
    = &
        \inf_{\pi \in \operatorname{Cpl}(\pp_{\theta_1},\pp_{\theta_2}}\,
        \int\, \|y_1-y_2\| \, \pi(d(y_1,y_2))
    \\
    \leq &
        \int\, \|y_1-y_2\| \, \pi^{\star}(d(y_1,y_2))
    \\
    = &
        \int\, \|p_{\theta_1}(u)-p_{\theta_2}(u)\| \, \pp(du)
    \\
        \leq &
        L\, 
        \|\theta_1-\theta_2\|\int 1 \pp(du)
    \\
       \leq & 
       \max\{L,C\}\, 
       \|\theta_1-\theta_2\| 1
    .
\end{aligned}
\end{equation}
Thus, the map from $[0,1]^{P}$ to $\mathcal{P}_1(\rr^D)$ given by $\theta \mapsto \pp_{\theta}$ is $\max\{L,C\}$-bi-Lipschitz.  Therefore, \citep[Lemmata 9.2 and 9.3]{Robinson_2009} imply that $\mathcal{M}$ is $(2^p\max\{L,C\}^{2p},p)$-homogeneous (in the sense of \citep[Definition 9.1]{Robinson_2009}); meaning that, the compact set $\mathcal{M}$ can be covered by at-most 
$
2^p\max\{L,C\}^{2p} \, 
\big(
    \frac{|\mathcal{M}|}{\epsilon_Q}
\big)^p,
$
of radius $\epsilon_Q>0$.  Note, by~\eqref{PROOF_genstatmanifold_eq_biLipshitz_condition__upperbound} the map $\theta \mapsto \pp_{\theta}$ is $L$-Lipschitz and $|[0,1]^p|=\sqrt{p}$; thus, $|\mathcal{M}|=L\sqrt{p}$.  Therefore, $\mathcal{M}$ can be covered by at-most
\[
\big\lceil
    \epsilon_Q^{-p}
    \,
    \big(
        2^p\max\{L,C\}^{2p} L^p\sqrt{p}^{p}
    \big)
\big\rceil
\]
many $1$-Wasserstein balls of radius $\epsilon_Q>0$.  

For the last claim about $\theta \mapsto \pp_{\theta}$, note that by~\eqref{PROOF_genstatmanifold_eq_biLipshitz_condition__lowerbound} the map $\theta \mapsto \pp_{\theta}$ must be injective.  Therefore, the inverse map $\pp_{\theta}\mapsto \theta$ is a well-defined map from $\mathcal{M}$ to $[0,1]^p$.  Multiplying across by $C$ in~\eqref{PROOF_genstatmanifold_eq_biLipshitz_condition__lowerbound} we conclude that $\pp_{\theta}\mapsto \theta$ is also $1$-Lipschitz.  
\end{proof}
\begin{proof}[Proof of Theorem~\ref{theorem_finite_Assouad_Dimension}]
Since $\mmm$ has Assouad dimension $d_{\mathcal{M}}>0$ and since $C_{\mathcal{M}}$ is such that~\eqref{eq_homogenity} holds then 
\[
    \mathcal{N}^{\operatorname{cov}}_{f(\xxx)}(r)
        \leq 
    C_{\mathcal{M}} r^{-d_{\mathcal{M}}} |f(\xxx)|^{d_{\mathcal{M}}}.
\]
Since $f$ is uniformly continuous with modulus of continuity $\omega_f$ then $|f(\xxx)|\leq \omega_f(|\xxx|)$.  The result now follows from Lemma~\ref{MAIN_LEMMA_theorem_UAT_qualitative}.  
\end{proof}

\begin{proof}[{Proof of Corollary~\ref{cor_mean_prediction}}]
By \citep[page 612]{Navas_2013_CanonicalBarycenterMap} the map $\beta:\mathcal{P}_1(\rr^d)\rightarrow \rr^d$ sending any $\mathbb{Q} \in \mathcal{P}_1(\rr^d)$ to the Bochner integral $\mathbb{E}_{X\sim \mathbb{Q}}[X]$ is $1$-Lipschitz.  Since Theorem~\ref{theorem_MAIN__GeneralCase} showed that
\begin{equation}
	\label{PROOF_cor_mean_prediction}
	    \max_{x\in \xxx}\,
	\mathcal{W}_1(f(x),\hat{T}(x))<\epsilon_Q+\epsilon_A
,
\end{equation}
and since the contracting barycenter map $\beta$ $1$-Lipschitz then~\eqref{PROOF_cor_mean_prediction} implies that
\[
    \max_{x\in \xxx}\,
    \Big\|
        \mathbb{E}_{Y^x\sim f(x)}\big[
            Y^x
        \big]
            -
        \mathbb{E}_{\hat{Y}^x\sim \hat{\mu}_n}\big[
            \hat{Y}^x
        \big]
    \Big\|
        \leq 
    \max_{x\in \xxx}\,
    1 
    \mathcal{W}_1(f(x),\hat{T}(x))<\epsilon_Q+\epsilon_A.
\]
\end{proof}

\begin{proof}[{Proof of Corollary~\ref{cor_MC_Lipschitz_test_functions}}]
If $g$ is constant, then
$
\Big|
            \mathbb{E}_{Y\sim \hat{T}(x)}\big[
            g(Y)
            \big]
        -
            \mathbb{E}_{\tilde{Y}\sim f(x)}\big[
            g(\tilde{Y})
            \big]
        \Big|
    =
        0.
$
Therefore, let us assume that $g$ is non-constant.  Now, since $g:\yyy\rightarrow \rr$ is Lipschitz and non-constant then $\tilde{g}\eqdef \frac1{\operatorname{Lip}(g)}\, g$ is $1$-Lipschitz.  Thus, for every $x\in \xxx$, the Kantorovich-Rubinstein duality and Theorem~\ref{theorem_MAIN__GeneralCase} imply that
\begin{equation}
\label{PROOF_eq__cor_MC_Lipschitz_test_functions___KRDualityPlusMainTheorem}
       \Big|
            \mathbb{E}_{Y\sim \hat{T}(x)}\big[
                \tilde{g}(Y)
            \big]
        -
            \mathbb{E}_{\tilde{Y}\sim f(x)}\big[
            \tilde{g}(\tilde{Y})
            \big]
        \Big|
    \leq   
        \mathcal{W}_1(\hat{T}(x),f(x))
    <
        \epsilon_Q + \epsilon_A
    .
\end{equation}
Plugging in the definition of $\tilde{g}$ and using the linearity of integration in~\eqref{PROOF_eq__cor_MC_Lipschitz_test_functions___KRDualityPlusMainTheorem} yields
\[
       \Big|
            \mathbb{E}_{Y\sim \hat{T}(x)}\big[
                g(Y)
            \big]
        -
            \mathbb{E}_{\tilde{Y}\sim f(x)}\big[
            g(\tilde{Y})
            \big]
        \Big|
    <
        \operatorname{Lip}(g)\big(
                \epsilon_Q + \epsilon_A
        \big)
    .
\]
Expanding the definition of each $\hat{T}(x)$ and relabeling yields the conclusion.
\end{proof}

\begin{proof}[Proof of Corollary~\ref{cor_UAT_Qualitative}]
Since $\rr^D$ is complete and separable and since $\yyy\subseteq \rr^D$ is closed, then $\yyy$ is also closed and separable.  Thus,  by \citep[Theorem 6.18]{VillaniOptTrans} the Wasserstein space $\mathcal{P}_1(\rr^D)$ is also complete and separable.  Since both $\mathcal{P}_1(\yyy)$ and $\xxx$ are separable and metrizable then \citep[Theorem 13.6]{klenke2013probability} implies that $\pp$ is a Radon measure on $(\xxx,\bbb(\xxx))$.  Since $\xxx$ is locally-compact, separable, and has a second-countable topology, and since $f:\xxx\rightarrow \mathcal{P}_1(\yyy)$ is Borel-measurable thus, we may apply Lusin's Theorem (as formulated in \citep[Exercise 13.1.3]{klenke2013probability}) to conclude that there is a compact subset $K_{\epsilon}\subseteq \xxx$ satisfying
\begin{equation}
\label{PROOF__cor_UAT_Qualitative_Lusin}
    \pp\left(K_{\epsilon}\right)
        \geq 
    1-\epsilon
\end{equation}
     and 
    $f|_{K_{\epsilon}}:K_{\epsilon}\rightarrow \mathcal{P}_1(\yyy)$ is continuous.  
    Since every $K_{\epsilon}$ is compact and $f|_{K_{\epsilon}}:K_{\epsilon}\rightarrow \mathcal{P}_1(\rr^D)$ is continuous then by \citep[Theorem 27.6]{munkres2014topology} it is uniformly continuous.  Thus, Theorem~\ref{theorem_MAIN__GeneralCase} (i) implies that there is a probabilistic transformer satisfying 
    \begin{equation}
    \label{PROOF__cor_UAT_Qualitative_Theorem1Application}
    \max_{x\in K_{\epsilon}}\,
        \mathcal{W}_1\Big(
            f(x)
        ,
            \hat{T}(x)
        \Big)   < \epsilon;
    \end{equation}
    where we have set $\epsilon_Q\eqdef \epsilon_A\eqdef \epsilon/2$ in the notation of Theorem~\ref{cor_UAT_Qualitative}.  Together~\eqref{PROOF__cor_UAT_Qualitative_Lusin} and~\eqref{PROOF__cor_UAT_Qualitative_Theorem1Application} imply that
    \[
    \pp\Big(
        \mathcal{W}_1\Big(
            f(x)
        ,
            \hat{T}(x)
        \Big)   < \epsilon
    \Big)
        \geq 1-\epsilon
    .
    \]
\end{proof}

\subsection{{Proof of Results from Section~\ref{sss_Functorial_efficiency}}}\label{s_A_proofs_of_localization_theorems}
\begin{lemma}[Uniform Dual Bound]
\label{lem_tree_bound}
Fix a valency parameter $V\in \nn_+$ with $V\geq 2$, a rate $r>0$, a growth parameter $S>0$, an activation function $\sigma \in C(\rr)$, a feature map $\varphi:\xxx\rightarrow \rr^d$ satisfying condition~\ref{cond_feature}, and a dimension $D\in \nn_+$.  
Any function $f\in \xxx\rightarrow \mathcal{P}_1(\rr^D)$ belonging to the tree-class $\mathcal{T}_{T:V}^{\varphi,r,\sigma,S}(\rr^d,\mathcal{P}_1(\rr^D))$
satisfies the following ``uniform dual boundedness property''
\[
    \underset{h:\rr^D\rightarrow \rr,\, h(x_v)=0,\, \operatorname{Lip}(h)\leq 1}{
            \sup
        }\,
            \int_{x\in \rr^D}\,
            h(u)\,
            [f_v(x)](du)
        \leq 
    S.
\]
Equivalently, in terms of the Arens-Eells' space norm, $\big\|f\big\|_{\text{\AE}(\rr^D),0}\leq S$.
\end{lemma}
\begin{proof}[Proof of Lemma~\ref{lem_tree_bound}]
We argue by induction on $T$.  For the base case, if $T=0$ then $f\in\mathcal{T}_{0:V}^{\varphi,r,\sigma,S}(\rr^d,\mathcal{P}_1(\rr^D))$.  Whence, $f$ is of the form $f=\delta_{g\circ \varphi}$ for some $g\in \mathcal{A}^{\sigma,r,S}(\rr^d,\rr^D)$.  Therefore, 
\[\resizebox{0.95\hsize}{!}{$
\begin{aligned}
\nonumber
    \underset{h:\rr^D\rightarrow \rr,\, h(x_v)=0,\, \operatorname{Lip}(h)\leq 1}{
            \sup
        }\,
            \int_{u\in \rr^D}\,
                h(u)\,
            [f(x)](du)
    = &
    \underset{h:\rr^D\rightarrow \rr,\, h(x_v)=0,\, \operatorname{Lip}(h)\leq 1}{
            \sup
        }\,
            \int_{u\in \rr^D}\,
                 h(u)-h(0)\,
            \delta_{g\circ \varphi(x)}(du)
    \\
    \nonumber
    \leq &
    \underset{h:\rr^D\rightarrow \rr,\, h(x_v)=0,\, \operatorname{Lip}(h)\leq 1}{
            \sup
        }\,
            \int_{u\in \rr^D}\,
                |h(u)-h(0)|\,
            \delta_{g\circ \varphi(x)}(du)
    \\
    \nonumber
    \leq &
        \int_{u\in \rr^D}\,
            1 \|u\|\,
        \delta_{g\circ \varphi(x)}(du)
    \\
    \nonumber
    = &
        g\circ \varphi(x)
    \\
    \nonumber
    \leq &
        |g(\varphi(x))-0|
    \\
    \nonumber
    \leq & 
        S.
\end{aligned}
$}
\]
This establishes the base case.  For the (strong) inductive step let $T\in \nn_+$ and assume that for every $t\in \{0,\dots,T-1\}$ and every $\tilde{f}\in \mathcal{T}_{T:V}^{\varphi,r,\sigma,S}(\rr^d,\mathcal{P}_1(\rr^D))$
\begin{equation}
\label{PROOF__lem_tree_bound__InductiveStep}
    \underset{h:\rr^D\rightarrow \rr,\, h(x_v)=0,\, \operatorname{Lip}(h)\leq 1}{
            \sup
        }\,
            \int_{u\in \rr^D}\,
            h(u)\,
            [\tilde{f}(x)](du)
        \leq 
    S.
\end{equation}
Using,~\eqref{PROOF__lem_tree_bound__InductiveStep} and the recursive representation of any $f\in \mathcal{T}_{T:V}^{\varphi,r,\sigma,S}(\rrd,\mathcal{P}_1(\rr^D))$ as $f
        \eqdef 
    \sum_{v=1}^V\,
        [\operatorname{Softmax}_V\circ w\circ \varphi(\cdot)]_v
        f_v$ where $
f_1,\dots,f_V
    \in
\mathcal{T}_{T-1:V}^{\varphi,r,\sigma,S}(\rr^d,\mathcal{P}_1(\rr^D))
$ and $w\in \mathcal{A}^{r,\sigma,S}(\rr^d,\rr^V)$, we deduce the follow upper-bound; wherein, for notational simplicity we set $
\bar{w}(x)\eqdef \operatorname{Softmax}_V\circ w\circ \varphi(x)
$
\[
\resizebox{0.95\hsize}{!}{$
\begin{aligned}
    \nonumber
        \underset{h:\rr^D\rightarrow \rr,\, h(x_v)=0,\, \operatorname{Lip}(h)\leq 1}{
                \sup
            }\,
        \int_{u\in \rr^D}\,
                h(u)
            [f(x)](du)
    = &
        \underset{h:\rr^D\rightarrow \rr,\, h(x_v)=0,\, \operatorname{Lip}(h)\leq 1}{
                \sup
            }\,
        \int_{u\in \rr^D}\,
                h(u)-h(0)
            \big[
                \sum_{v=1}^V
                \bar{w}(x)_v
                f_v
            \big](du)
    \\
    \nonumber
    = &
        \underset{h:\rr^D\rightarrow \rr,\, h(x_v)=0,\, \operatorname{Lip}(h)\leq 1}{
                \sup
            }\,
        \int_{u\in \rr^D}\,
                h(u)-h(0)
            \big[
                \sum_{v=1}^V
                \bar{w}_v(x)
                f_v(x)
            \big](du)
    \\
    \nonumber
    \leq &
        \sum_{v=1}^V
            \bar{w}_v(x)
        \,
        \underset{h:\rr^D\rightarrow \rr,\, h(x_v)=0,\, \operatorname{Lip}(h)\leq 1}{
                \sup
            }\,
        \int_{u\in \rr^D}\,    
                \big(
                    h(u)-h(0)
                \big)
                f_v(x)(du)
    \\
    \nonumber
    \leq &
        \sum_{v=1}^V
            \bar{w}_v(x)
        \,
        S
    \\
    \label{PROOF__lem_tree_bound__InductiveStep___SimplicialSum}
    = & S;
\end{aligned}
$}
\]
where~\eqref{PROOF__lem_tree_bound__InductiveStep___SimplicialSum} held since $\bar{w}_v(x)\in \Delta_V$ whence, for every $x\in \rr^d$ we have that $\sum_{v=1}^V\, \bar{w}_v(x) = 1$.  This completes the induction.  
\end{proof}
\begin{proof}[{Proof of Theorem~\ref{theorem_efficient_approximation}}]
Let $f\in \mathcal{T}_{T:V}^{\varphi,r,\sigma,S}$.  We will argue by induction.  
\hfill\\
\textbf{Base Case: Height $T=0$}
\hfill\\
If $T=0$ then, by definition, there must exist a $g\in \mathcal{A}^{\sigma,r,S}(\rr^d,\rr^D)$ such that
\begin{equation*}
    f 
        = 
    \delta_{g\circ \varphi}
.
\end{equation*}
By condition~\ref{cond_feature} and \citep[Jung's Theorem]{jung1910boundedingdiametresinEuclideanSpace} there is an $x \in \rr^d$ for which $\varphi(\xxx)\subseteq \overline{B_{\rr^d}(x,R)}$ where 
\begin{equation*}
    R
        \eqdef 
    C_{\varphi:U} |\xxx|^{\alpha} \Big(\frac{d}{2(d+1)}\Big)^{1/2}   
.
\end{equation*}
Since $g\in \mathcal{A}^{\sigma,r,S}(\rr^d,\rr^D)$ then,
for every ${\frac{\epsilon}{V^T}}>0$ there must exist a DNN approximation class $\hat{f} \in \NN[d,D][\sigma]$ whose depth and width are both at-most $C^{(0)} R^{r}({\frac{\epsilon}{V^T}})^{-r}$; for some constant $C^{(0)}>0$ independent of $R$, $x$, $r$, and of $\epsilon$.  Therefore, the number of parameters determining $\hat{f}$ must be at-most
\begin{equation}
\label{PROOF_theorem_efficient_approximation__parametercount___0}
\begin{aligned}
\operatorname{Par}(\hat{f})
        \eqdef &
    \Big\lceil
    C^{(0)} R^{r}({V^{-T}\epsilon})^{-r}
    \,
    C^{(0)} R^{r}({V^{-T}\epsilon})^{-r}
    \left(
    C^{(0)} R^{r}({V^{-T}\epsilon})^{-r}
    + 1
    \right)
    \Big\rceil
        \\
        \leq &
    2
    \big\lceil
        \tilde{C}^{(0)} R^{2r}({V^{-T}\epsilon})^{-2r}
    \big\rceil
        \\
    = &
    \Big\lceil
        \tilde{C}^{(0)} 
        C_{\varphi:U}^{2r} 
        |\xxx|^{\alpha 2r} 
        \Big(\frac{d}{2(d+1)}\Big)^{r}
        ({V^{-T}\epsilon})^{-2r}
    \Big\rceil
;
\end{aligned}
\end{equation}
where $(C^{(0)})^2\eqdef \tilde{C}^{(0)}$.  
Furthermore, $\hat{f}$ satisfies the uniform estimate
\begin{equation}
\label{PROOF_theorem_efficient_approximation__estimateEuclidean}
    \max_{x\in \overline{B_{\rr^d}(x,R)}}\,
    \big\|
        g(x)
            -
        \hat{f}(x)
    \big\|
        <
    {V^{-T}\epsilon}
.
\end{equation}
Since the map sending any $x\in \rr^D$ to $\delta_x\in \mathcal{P}_1(\rr^D)$ is an isometry, since $\varphi$ is a bijection from $\xxx$ onto $\varphi(\xxx)$, and since $\varphi(\xxx)\subseteq \overline{B_{\rr^d}(x,R)}$ then, the estimate in~\eqref{PROOF_theorem_efficient_approximation__estimateEuclidean} implies that
\allowdisplaybreaks
\begin{align*}
        \max_{x\in \xxx}\,
            \mathcal{W}_1\big(
            f(x)
                ,
            \delta_{\hat{f}\circ \varphi(x)}
        \big)
    = &
        \max_{x\in \xxx}\,
            \mathcal{W}_1\big(
            \delta_{g\circ \varphi(x)}
                ,
            \delta_{\hat{f}\circ \varphi(x)}
        \big)
    \\
    = & 
        \max_{u\in \varphi(\xxx)}\,
            \mathcal{W}_1\big(
            \delta_{g(u)}
                ,
            \delta_{\hat{f}(u)}
        \big)
    \\
    = &
    \max_{u\in \varphi(\xxx)}\,
        \big\|
            g(u)
                -
            \hat{f}(u)
        \big\|
    \\
    \leq &
        \max_{u\in \overline{B_{\rr^d}(x,R)}}\,
        \big\|
            g(u)
                -
            \hat{f}(u)
        \big\|
    \\
    < & {V^{-T}\epsilon}
.
\end{align*}
Thus, $\hat{T}\eqdef \hat{T}^{(0)} \eqdef \delta_{\hat{f}\circ \varphi}$ satisfies (i) and~\eqref{PROOF_theorem_efficient_approximation__parametercount___0} implies that (ii) holds.  
\hfill\\
\textbf{Inductive Step: Height $T\in \nn_+$}
\hfill\\
Since $f\in \mathcal{T}_{T:V}^{\varphi,r,\sigma,S}$ then by definition it has representation
\begin{equation}
\label{PROOF_theorem_efficient_approximation___representation}
    f
        \eqdef 
    \sum_{v=1}^V\,
        [\operatorname{Softmax}_V\circ w\circ \varphi(\cdot)]_v
        f_v
    ,
\end{equation}
where $
f_1,\dots,f_V
    \in
\mathcal{T}_{T-1:V}^{\varphi,r,\sigma,S}(\rr^d,\mathcal{P}_1(\rr^D))
$ and $w\in \mathcal{A}^{r,\sigma,S}(\rr^d,\rr^V)$.  
By strong-induction, for every $t\in \{0,\dots,T-1\}$, there exist $\hat{T}_1,\dots,\hat{T}_V\in \NN[\varphi:\yyy:s,V][\sigma:\star]$ such that
\begin{equation}
\label{PROOF_theorem_efficient_approximation___InductionHypothesis}
    \max_{x\in \overline{B_{\rr^d}(x,R)}}\,
        \mathcal{W}_1\big(
            f_v(x)
                ,
            \hat{T}_v(x)
        \big)
    <
        \epsilon^{(t)}
            \eqdef 
        \frac{
            \epsilon
        }{
            S 
            V^{t}
        }
\end{equation}
and such that, for $v=1,\dots,V$, $\hat{T}_v$ is determined by at-most 
\allowdisplaybreaks
\begin{equation*}
\begin{aligned}
    \max_{v=1,\dots,V}\,
    \operatorname{Par}(\hat{T}_v)
        \leq 
    C^{(t)}
        \epsilon^{-2r}
        V^{t(1-2r)}
        \left(
            \frac{
                    C_{\varphi:U}^{2} 
                    |\xxx|^{2\alpha } 
                    d
                }{
                    2(d+1)
                }
        \right)^r
    ,
\end{aligned}
\end{equation*}
for some constants $C^{(t)}>0$ independent of $\epsilon$, $d$, $D$, $|\xxx|$, $\alpha$, and of $r$.  
Now, since $w\in \mathcal{A}^{r,\sigma,S}(\rr^d,\rr^V)$ then, there exists a $\hat{w}^{(T)}\in \NN[d,V][\sigma]$ satisfying
\begin{equation}
\label{PROOF_theorem_efficient_approximation___wapproximation}
    \max_{u \in \varphi(\xxx)}\,
    \big\|
        w(x)
            -
        \hat{w}^{(T)}(x)
    \big\|
        <
    \frac{\epsilon}{
    S
    V^{1/2}
    }
,
\end{equation}
and $\hat{w}^{(T)}$ has representation depth and width bounded-above by
$
\Big\lceil
        C_{(T)}
        C_{\varphi:U}^r
        |\xxx|^{\alpha r}
        \Big(
            \frac{
                d
            }{
                2(d+1)
            }
        \Big)^{r/2}
        V^{-r}
        \epsilon^{-r}
    \Big\rceil
.
$  
Therefore, the number of parameters determining $\hat{w}^{(T)}$ is at-most 
\begin{equation}
\label{PROOF_theorem_efficient_approximation___wapproximation_A}
    2
    \Big\lceil
        C^{(T)}
        C_{\varphi:U}^r
        |\xxx|^{\alpha r}
        \Big(
            \frac{
                d
            }{
                2(d+1)
            }
        \Big)^{r/2}
        S^{-r}
        V^{-r/2}
        \epsilon^{-r}
    \Big\rceil^2
.
\end{equation}
In particular,~\eqref{PROOF_theorem_efficient_approximation___wapproximation_A} implies the following estimate
\begin{equation}
\label{PROOF_theorem_efficient_approximation___wapproximation_2}
\resizebox{0.95\hsize}{!}{$
\begin{aligned}
    \max_{x\in \xxx}\,
        \|
            \operatorname{Softmax}_V\circ w\circ \varphi(x)
                -
            \operatorname{Softmax}_V\circ \hat{w}^{(T)}\circ \varphi(x)
        \|_1
    = &
        \max_{u\in \varphi(\xxx)}\,
        \|
            \operatorname{Softmax}_V\circ w(u)
                -
            \operatorname{Softmax}_V\circ \hat{w}^{(T)}(u)
        \|_1
    \\
    \leq &
        \max_{u\in \overline{B_{\rr^d}(x,R)}}\,
        V^{1/2}
        \|
            \operatorname{Softmax}_V\circ w(u)
                -
            \operatorname{Softmax}_V\circ \hat{w}^{(T)}(u)
        \|_2
    \\
    \leq &
        \max_{u\in \overline{B_{\rr^d}(x,R)}}\,
        V^{1/2}
        \frac{(V-1)^{1/2}}{V}
        \|
            w(u)
                -
            \hat{w}^{(T)}(u)
        \|_2
    \\
    \leq &
        V^{1/2}
        \frac{(V-1)^{1/2}}{V}
        \frac{\epsilon}{
            S
            V^{1/2}
        }
    \\
    = & 
        \frac{1}{S}
        \frac{(V-1)^{1/2}}{V}
        \epsilon
    \\
    \leq &
        \frac{1}{S}
        \epsilon .
\end{aligned}
$}
\end{equation}
Together~\eqref{PROOF_theorem_efficient_approximation___representation},~\eqref{PROOF_theorem_efficient_approximation___InductionHypothesis}, and~\eqref{PROOF_theorem_efficient_approximation___wapproximation_2} imply that there are $\hat{T}_1,\dots,\hat{T}_V\in \NN[\varphi:\yyy:s,V][\sigma:\star]$ and a $\hat{w}^{(T)}\in \NN[d,N][\sigma]$ such that the tall probabilistic transformer $
\hat{T}\eqdef \sum_{v=1}^V\,
[\operatorname{Softmax}_{V}\circ \hat{w}^{(T)}\circ \varphi(x)]_v
\hat{T}_v(x)
$ satisfies the following
\[
\resizebox{1\hsize}{!}{$
\begin{aligned}
    \max_{x\in \xxx}\,
        \mathcal{W}_1\big(
                f(x)
            ,
                \hat{T}(x)
        \big)
    = &
        \max_{x\in \xxx}\,
            \mathcal{W}_1\big(
                    \sum_{v=1}^V\,
                        [\operatorname{Softmax}_V\circ w\circ \varphi(x)]_v
                        f_v(x)
                ,
                    \sum_{v=1}^V\,
                        [\operatorname{Softmax}_{V}\circ \hat{w}^{(T)}\circ \varphi(x)]_v
                        \hat{T}_v(x)
            \big)
    \\
    \leq &
        \max_{x\in \xxx}\,
            \mathcal{W}_1\big(
                    \sum_{v=1}^V\,
                        [\operatorname{Softmax}_V\circ w\circ \varphi(x)]_v
                        f_v(x)
                ,
                    \sum_{v=1}^V\,
                        [\operatorname{Softmax}_{V}\circ \hat{w}^{(T)}\circ \varphi(x)]_v
                        f_v(x)
            \big)
    \\
    & +
        \max_{x\in \xxx}\,
            \mathcal{W}_1\big(
                    \sum_{v=1}^V\,
                        [\operatorname{Softmax}_{V}\circ \hat{w}^{(T)}\circ \varphi(x)]_v
                        f_v(x)
                ,
                    \sum_{v=1}^V\,
                        [\operatorname{Softmax}_{V}\circ \hat{w}^{(T)}\circ \varphi(x)]_v
                        \hat{T}_v(x)
            \big)
    \\
    \leq &
        \max_{x\in \xxx}\,
            \big\|
                    \big(
                            \sum_{v=1}^V\,
                            [\operatorname{Softmax}_V\circ w\circ \varphi(x)]_v
                            f_v(x)
                        -
                            \delta_{0}
                    \big)
                -
                    \big(
                        \sum_{v=1}^V\,
                        [\operatorname{Softmax}_{V}\circ \hat{w}^{(T)}\circ \varphi(x)]_v
                        f_v(x)
                        -
                            \delta_{0}
                    \big)
            \big\|_{\text{\AE}(\rr^D,x_v)}
    \\
    & +
        \max_{x\in \xxx}\,
            \big\|
                    \big(
                        \sum_{v=1}^V\,
                        [\operatorname{Softmax}_{V}\circ \hat{w}^{(T)}\circ \varphi(x)]_v
                        f_v(x)
                            -
                        \delta_{0}
                    \big)
                -
                    \big(
                        \sum_{v=1}^V\,
                        [\operatorname{Softmax}_{V}\circ \hat{w}^{(T)}\circ \varphi(x)]_v
                        \hat{T}_v(x)
                            -
                        \delta_{0}
                    \big)
            \big\|_{\text{\AE}(\rr^D,x_v)}
    \\
    = &
        \max_{x\in \xxx}\,
            \big\|
                \sum_{v=1}^V\,
                f_v(x)
                \big(
                    [\operatorname{Softmax}_V\circ w\circ \varphi(x)]_v
                        -
                    [\operatorname{Softmax}_{V}\circ \hat{w}^{(T)}\circ \varphi(x)]_v
                \big)
            \big\|_{\text{\AE}(\rr^D,x_v)}
    \\
    & +
        \max_{x\in \xxx}\,
            \big\|
                \sum_{v=1}^V\,
                [\operatorname{Softmax}_{V}\circ \hat{w}^{(T)}\circ \varphi(x)]_v
                    \big(
                        f_v(x)
                            -
                        \hat{T}_v(x)
                    \big)
            \big\|_{\text{\AE}(\rr^D,x_v)}
    \\
    \leq &
        \max_{x\in \xxx}\,
        \sum_{v=1}^V\,
            \big|
                [\operatorname{Softmax}_V\circ w\circ \varphi(x)]_v
                    -
                [\operatorname{Softmax}_{V}\circ \hat{w}^{(T)}\circ \varphi(x)]_v
            \big|
            \big\|
                f_v(x)
            \big\|_{\text{\AE}(\rr^D,x_v)}
    \\
    & +
        \max_{x\in \xxx}\,
        \sum_{v=1}^V\,
            [\operatorname{Softmax}_{V}\circ \hat{w}^{(T)}\circ \varphi(x)]_v
                \big\|
                    \big(
                        f_v(x)
                            -
                        \hat{T}_v(x)
                    \big)
                \big\|_{\text{\AE}(\rr^D,x_v)}
    \\
    = &
        \max_{x\in \xxx}\,
        \sum_{v=1}^V\,
            \big\|
                f_v(x)
            \big\|_{\text{\AE}(\rr^D,x_v)}
            \big\|
                \operatorname{Softmax}_V\circ w\circ \varphi(x)
                    -
                \operatorname{Softmax}_{V}\circ \hat{w}^{(T)}\circ \varphi(x)
            \big\|_1
    \\
    & +
        \max_{x\in \xxx}\,
        \sum_{v=1}^V\,
            [\operatorname{Softmax}_{V}\circ \hat{w}^{(T)}\circ \varphi(x)]_v
                \big\|
                    \big(
                        (f_v(x)-\delta_{0})
                            -
                        (\hat{T}_v(x) - \delta_{0})
                    \big)
                \big\|_{\text{\AE}(\rr^D,x_v)}
    \\
    \leq &
        \max_{x\in \xxx}\,
        \left(
            \max_{v=1,\dots,V}
            \underset{g:\rr^D\rightarrow \rr,\, h(x_v)=0,\, \operatorname{Lip}(g)\leq 1}{\sup}\,
                \int_{x\in \rr^D}\,
                g(u)\,
                [f_v(x)](du)
        \right)
            \big\|
                \operatorname{Softmax}_V\circ w\circ \varphi(x)
                    -
                \operatorname{Softmax}_{V}\circ \hat{w}^{(T)}\circ \varphi(x)
            \big\|_1
    \\
    & +
        \max_{x\in \xxx}\,
        \sum_{v=1}^V\,
            [\operatorname{Softmax}_{V}\circ \hat{w}^{(T)}\circ \varphi(x)]_v
            \mathcal{W}_1
                \big(
                    f_v(x)
                ,
                    \hat{T}_v(x)
                \big)
   \\
   \leq &
        \frac1{S}
        \epsilon
            \left(
                \max_{x\in \xxx}\,
                \max_{v=1,\dots,V}
                    \underset{g:\rr^D\rightarrow \rr,\, h(x_v)=0,\, \operatorname{Lip}(g)\leq 1}{\sup}\,
                        \int_{x\in \rr^D}\,
                        g(u)\,
                        [f_v(x)](du)
            \right)
        +
        \max_{x\in \xxx}\,
        V
        \max_{v=1,\dots,V}\,
            \mathcal{W}_1
                \big(
                    f_v(x)
                ,
                    \hat{T}_v(x)
                \big)
   \\
   \leq &
        \frac{1}{S}
        \epsilon
            \left(
                \max_{x\in \xxx}\,
                \max_{v=1,\dots,V}
                    \underset{g:\rr^D\rightarrow \rr,\, h(x_v)=0,\, \operatorname{Lip}(g)\leq 1}{\sup}\,
                        \int_{x\in \rr^D}\,
                        g(u)\,
                        [f_v(x)](du)
            \right)
        +
            \max_{x\in \xxx}\,
            V 
            \epsilon^{(T-1)}
    \\
    \leq &
        \frac{1}{S}
        \epsilon
            S
        +
            \max_{x\in \xxx}\,
            V 
            \epsilon^{(T-1)}
    \\
    = &
            \epsilon
        +
            \max_{x\in \xxx}\,
            V 
            \epsilon^{(T-1)}
    ;
\end{aligned}
$}
\]
where, the last inequality held by the ``uniform dual bound'' in Lemma~\ref{lem_tree_bound}. 

Finally, by recursion and by~\eqref{PROOF_theorem_efficient_approximation__parametercount___0} and~\eqref{PROOF_theorem_efficient_approximation___wapproximation}, the number of parameters determining $\hat{T}$ is at-most
\allowdisplaybreaks
\begin{align*}
    \nonumber
    \operatorname{Par}(\hat{T})
    = &
    \operatorname{\hat{w}^{(T)}}
        +
    \sum_{v=1}^V\, 
        \operatorname{Par}(\hat{T}_v)
    \\
    \nonumber
    \leq &
        V^{T} \, 
        \Big\lceil
            \tilde{C}^{(0)} 
            C_{\varphi:U}^{2r} 
            |\xxx|^{\alpha 2r} 
            \Big(\frac{d}{2(d+1)}\Big)^{r}
            ({V^{-T}\epsilon})^{-2r}
        \Big\rceil
    + 
        \sum_{t=1}^T \,
            V^{T-t}\,\operatorname{Par}(\hat{w}^{(t)})
    \\
    \nonumber
    \leq &
        \sum_{t=0}^T \,
            V^{T-t}
            \Big\lceil
                C^{(t)} \,
                C_{\varphi:U}^{2r} 
                |\xxx|^{\alpha 2r} 
                \Big(\frac{d}{2(d+1)}\Big)^{r}
                ({V^{-T}\epsilon})^{-2r}
            \Big\rceil
    \\
    \nonumber
    \leq & 
        \Big(
            \max_{t=0,\dots,T}\, C^{(t)}\,
                C_{\varphi:U}^{2r} 
                |\xxx|^{\alpha 2r} 
                \Big(\frac{d}{2(d+1)}\Big)^{r}
                \epsilon)^{-2r}
        \Big)
        \Big(
            V^{T(1-2r)}
            \sum_{t=0}^T \,
            V^{-t}
        \Big)
    \\
    \nonumber
    = &
        \Big(\max_{t=0,\dots,T}\, C^{(t)}\,\Big)
        \Big(
                C_{\varphi:U}^{2r} 
                |\xxx|^{\alpha 2r} 
                \Big(\frac{d}{2(d+1)}\Big)^{r}
                \epsilon)^{-2r}
        \Big)
        \Big(
            V^{T(1-2r)}
            \frac{1 - V^{-T-1}}{1-V}
        \Big)
    \\
    \leq &
        \Big(\max_{t=0,\dots,T}\, C^{(t)}\,\Big)
        V^{T(1-2r)}
        \Big(
                C_{\varphi:U}^{2r} 
                |\xxx|^{\alpha 2r} 
                \Big(\frac{d}{2(d+1)}\Big)^{r}
                \epsilon)^{-2r}
        \Big)
\end{align*}
\end{proof}
\begin{proof}[{Proof of Corollary~\ref{cor_localization}}]
Set $\epsilon\eqdef \epsilon_Q/2\eqdef \epsilon_A/2$ where $\epsilon_A\in (0,1)$ and set $\sigma = \operatorname{ReLU}$.  
Let $\delta\eqdef \omega_f^{-1}(\epsilon/3)$ and define $
        \mathbb{X}^{\delta}
    \eqdef 
        \cup_{x\in \mathbb{X}}\, 
        \{u\in \xxx:\, 
        d(u,x) <\delta\}.
$ 
Then, $\mathcal{N}_{f(\mathbb{X}^{\delta})}(\epsilon_Q/3)=\#\mathbb{X}$.  
For every $R>0$, $x\in \xxx$, define $
    \mathbb{X}^{R,\delta}_{x} 
        \eqdef 
    \{u\in \xxx:\, d(u,x)\leq R\} \cap \mathbb{X}^{\delta}
$. Lemma~\ref{MAIN_LEMMA_theorem_UAT_qualitative} implies that that there exists a $\hat{T}\in \NN[\varphi,D][\sigma:\star]$ satisfying 
\[
    \sup_{x\in \mathbb{X}^{R,\delta}_x}\,
        \mathcal{W}_1\big(
            f(x)
        ,   
            \hat{T}(x)
        \big)
    <
        \epsilon
    ,
\]
where $N = \# \{u\in \xxx:\, d(u,x)\leq R\}$, Width $d(N-1) + 3^{d+3}\max\{d,3\}$, and of depth $\operatorname{Depth}(\hat{T})$ bounded above by
\begin{equation}
\label{PROOF__cor_localization}
        \operatorname{Depth}(\hat{T})
    \leq 
        N\left(
                C_d^{(1)}
            +
                11
                \left\lceil
                        {C}_d^{(2)} 
                        R^{\alpha d /2}
                        \,
                        \Big(
                            \omega_f^{-1}\Big(
                                {C}_d^{(3)} 
                                \,
                                \frac{\epsilon^2}{
                                |2-\epsilon|}
                                \,
                                \frac{
                                    N - \sqrt{N}- 1
                                }{
                                    N^{3/2}
                                }
                                \Big)
                            \Big)^{-\alpha d/2}
                \right\rceil
            \right)
.
\end{equation}
We may without loss of generality assume that $N\geq 10$; which we do.  
Since $\omega_f(t)=B t^{\beta}$ is strictly increasing and since $\frac{N^{3/2}}{N - N^{1/2}-1}<10$ then
\begin{equation}
\label{PROOF__cor_localization___improved_bound}
        \operatorname{Depth}(\hat{T})
    \leq 
        N\left(
                C_d^{(1)}
            +
                11
                \left\lceil
                    {C}_d^{(2)} 
                    R^{\alpha d /2}
                    \,
                    \Big(
                        B
                        (
                            10 \,
                            {C}_d^{(3)}
                        )^{\alpha d/(\beta 2)}
                            2
                        \,
                            \epsilon^{-2}
                        \Big)^{\alpha d/(\beta 2)}
                \right\rceil
            \right)
.
\end{equation}
Set $R
<
    \sup\{r>0:\, 
        \# \mathbb{X} \cap B_{\xxx}(x,r) \leq \epsilon^{-r/2}
    \}
$.  
Then, $N\leq \epsilon^{-r/2}$; therefore,~\eqref{PROOF__cor_localization___improved_bound} implies that $\operatorname{Depth}(\hat{T})\in \mathscr{O}\big(
    \epsilon^{-r}
\big)$ if $
R^{\alpha d /2}
\,
    \epsilon^{\alpha d/(\beta)}
\leq 
    \epsilon^{-r/2}.
$
Therefore, setting
\[
        R
    \eqdef 
        \min\left\{
            \epsilon^{2/\beta - 2r/ad}
        ,
            \sup\{r>0:\, 
                \# \mathbb{X} \cap B_{\xxx}(x,r) \leq \epsilon^{-r/2}
            \}
        \right\}
\]
implies that $\operatorname{Depth}(\hat{T})\in \mathscr{O}\big(
    \epsilon^{-r}
\big)$.  Finally, the assumption that each $f(x)$ is lower $(C,q)$-Ahlfors regular implies that Theorem~\ref{theorem_MAIN__AhlforsRegularCase} applies; whence we may take $Q\in \mathscr{O}\big(
\epsilon^{-r}
\big)$.  Setting $C_f\eqdef (3B)^{\frac{-1}{\beta}}$ completes the proof.  
\end{proof}

\section{Proofs from the Applications Section}\label{s_A_Motivational_Results}
This appendix contains proofs and formalizations of the paper's motivational examples in the paper's introductory section.  

\subsection{Proof Pertaining to Problem 1: Generic Regular Conditional Distributions}
\label{a_Proofs__applications}
\begin{proof}[Proof of Corollary~\ref{cor_Universal_Regular_Conditional_Distributions_RCP}]
Since $\yyy$ and $\xxx$ are Polish, then \citep[Theorem 13.6]{klenke2013probability} implies that $X_{\#}\pp$ is a Borel probability measure on $(\yyy,\bbb(\yyy))$.  Note also that $\xxx$ was assumed to be locally compact.  Applying \citep[Theorem 6.3]{Kallenberg2002Foundations}, also we observe that $\pp_{\cdot}:\xxx \ni x\mapsto \pp_x^Y\in \ppp[\yyy][1]$ is Borel-measurable.   Therefore, we may apply Corollary~\ref{cor_UAT_Qualitative} yields the conclusion to obtain the conclusion.  
\end{proof}

\begin{proof}[{Proof of Corollary~\ref{cor_Universal_Disintegration}}]
Let $L$ be the uniform Lipschitz constant of $f$. 
Since $y\mapsto f(y,x)$ is Lipschitz, for every $x\in \xxx$, then it is continuous in its first argument.  Since $f$ is uniformly bounded then, there must be some $m>0$ and some distinguished point $(y^{\star},x^{\star})\in \yyy\times \xxx$ for which:
\begin{equation*}
    \sup_{(y,x)\in \yyy\times \xxx} |f(y,x)|
=
\sup_{y \in \yyy} \sup_{x \in \xxx}|f(y,x)|
\leq 
m|f(y^{\star},x^{\star})|
<\infty
.
\end{equation*}
Hence, $M\eqdef \max\{2,L,\sup_{(y,x)\in \yyy\times \xxx}f(y,x)\}<\infty$.  
Therefore, the function $
\tilde{f}(y,x)\eqdef M^{-1}f(y,x)
,
$ is measurable, $2^{-1}$-Lipschitz in its first argument, and uniformly bounded in both arguments by $2^{-1}$.  
In particular:
\begin{equation}
    \sup_{(y,x) \in \yyy\times \xxx}
    |\tilde{f}(y,x)|
        +
    \|\tilde{f}(\cdot,x)\|_{\operatorname{Lip}_0}\leq 1
        \label{eq_proof_cor_universal_disintegration_0}
    ,
\end{equation}
Applying Corollary~\ref{cor_Universal_Regular_Conditional_Distributions_RCP} and Kantorovich-Rubinstein duality (see \citep[Remark 6.5]{VillaniOptTrans}) we obtain the following bound, for every $x \in \xxx$:
\allowdisplaybreaks
\begin{align}
\nonumber
\epsilon 
& >
W_1\left(
\pp(Y \in \cdot|X=x)
    ,
\sum_{n=1}^{N(\epsilon)}
    \left(
        \operatorname{Softmax}\circ \hat{f}\circ \varphi(x)
    \right)\hat{\mu}_n
\right)
\\
\nonumber
& =  \sup_{\sup_{y \in \yyy} \|g(x)\| + \operatorname{Lip}(g)\leq 1}\,
\int_{y \in \yyy} g(y) \pp\left(Y \in dx\mid X=x\right)
-
\int_{y \in \yyy} g(y) 
    \left(
    \sum_{n=1}^{N(\epsilon)}
    \left(
        \operatorname{Softmax}\circ \hat{f}\circ \varphi(x)
    \right)\hat{\mu}_n\right)(dy)
\\ 
& \geq 
\left|
    \int_{y \in \yyy} \tilde{f}(y,x) \pp\left(Y \in dx\mid X=x\right)
    -
    \int_{y \in \yyy} \tilde{f}(y,x)
        \left(
        \sum_{n=1}^{N(\epsilon)}
        \left(
            \operatorname{Softmax}\circ \hat{f}\circ \varphi(x)
        \right)\hat{\mu}_n\right)(dy)
\right|
;
    \label{eq_proof_cor_universal_disintegration_1}
\end{align}
where we have used~\eqref{eq_proof_cor_universal_disintegration_0} to obtain the last inequality in~\eqref{eq_proof_cor_universal_disintegration_1}.  Under our assumptions, we may apply the Disintegration Theorem (as formulated in \citep[Theorem 6.4]{Kallenberg2002Foundations}) to identify:
\begin{equation}
    \ee\left[\tilde{f}(Y,X)
\mid X = x
\right]
=
\int_{y \in \yyy} \tilde{f}(y,x) 
d\pp\left(Y \in \mid X =x \right)(y)
\label{eq_proof_cor_universal_disintegration_2_applyication_of_disintegration_Theorem}
,
\end{equation}
on a set of full $\pp$-measure.  Plugging~\eqref{eq_proof_cor_universal_disintegration_2_applyication_of_disintegration_Theorem} in to~\eqref{eq_proof_cor_universal_disintegration_1} we obtain the bound:
\allowdisplaybreaks
\begin{align}
\nonumber
\epsilon 
& >
\left|
    \int_{y \in \yyy} \tilde{f}(y,x) d\pp\left(Y \in \mid X=x\right)(d)
    -
    \int_{y \in \yyy} \tilde{f}(y,x)
        \left(
        \sum_{n=1}^{N(\epsilon)}
        \left(
            \operatorname{Softmax}\circ \hat{f}\circ \varphi(x)
        \right)\hat{\mu}_n\right)(dy)
\right|
\\
\nonumber
& 
= 
\left|
    \ee\left[
        \tilde{f}(Y,X)
        \mid X = x
    \right]
    -
    \int_{y \in \yyy} \tilde{f}(y,x)
        \left(
        \sum_{n=1}^{N(\epsilon)}
        \left(
            \operatorname{Softmax}\circ \hat{f}\circ \varphi(x)
        \right)\hat{\mu}_n\right)(dy)
\right|
\\
& 
=
    M^{-1}
\left|
    \ee\left[
        f(Y,X)
        \mid X = x
    \right]
    -
    \int_{y \in \yyy} 
        f(y,x)
        \left(
        \sum_{n=1}^{N(\epsilon)}
        \left(
            \operatorname{Softmax}\circ \hat{f}\circ \varphi(x)
        \right)\hat{\mu}_n\right)(dy)
\right|
.
    \label{eq_proof_cor_universal_disintegration_B}
\end{align}
Multiplying both sides of~\eqref{eq_proof_cor_universal_disintegration_B} by $M$ yields the result.  
\end{proof}

\subsection{Proof Pertaining to Problem 3: ``A Generic Expression of Epistemic Uncertainty''}\label{a_PROOF_ss_Problem_2}

\subsection{{Proof from Section~\ref{ss_Problem_2}}}\label{ss_A_proofs_ss_Problem_2}

\begin{proof}[{Proof of Proposition~\ref{prop_MOTIVATION_ex_MOTIVATION_CNTFUNCTIONS_2_Randomized_Parameters}}]
The condition: there exists some $y_0\in \yyy$ and some $1<q<\infty$ such that, for all $x\in \xxx$ it holds that $\sup_{\theta \in \Theta} \ee\left[d_{\yyy}(\hat{f}_{\theta}(x),y_0)^q\right]<\infty$ is the definition of $\pp\left(\hat{f}_{\theta}(x)\in \cdot\right)\in \ppp[\yyy][q]$.  Thus, $\xxx\ni x \mapsto \pp\left(\hat{f}_{\theta}(x)\in \cdot \right) \in \ppp[\yyy][q]$.  Next, fix any two $x_1,x_2 \in \xxx$ and compute:
$$
\begin{aligned}
W_1\left(
\mathbb{P}(\hat{f}_{\theta}(x_1)\in \cdot),\mathbb{P}(\hat{f}_{\theta}(x_2)\in \cdot)
\right)
 =&
\inf_{Y_1\sim \hat{f}_{\theta}(x_1),Y_2\sim \hat{f}_{\theta}(x_2)}\,\mathbb{E}\left[d_{\yyy}(Y_1,Y_2)\right]
\\
\leq &
\mathbb{E}\left[d_{\yyy}(\hat{f}_{\theta}(x_1),\hat{f}_{\theta}(x_2))\right]
\\
\leq &
\mathbb{E}\left[
\omega_{\hat{f}}\left(d_{\xxx}(g(\cdot),g(\cdot)) + d_{\Theta}(x_1,x_2)\right)\right]
\\ 
= & \mathbb{E}\left[0+\omega_{\hat{f}}\left(d_{\xxx}(x_1,x_2)\right)\right]\\
= & \omega\left(d_{\xxx}(x_1,x_2)\right).
\end{aligned}
$$
Hence, $x\mapsto \pp\left(\hat{f}_{\theta}(x)(\vartheta)\in \cdot\right)$ is uniformly continuous with modulus of continuity $\omega_f$.  
\end{proof}

\begin{proof}[{Proof of Corollary~\ref{cor_learnability}}]
By Proposition~\ref{prop_MOTIVATION_ex_MOTIVATION_CNTFUNCTIONS_2_Randomized_Parameters} $f:x\mapsto \pp\left(\hat{f}_{\vartheta}(x)\in \cdot\right) \in \ppp[\yyy][q]$ belongs to $C(\xxx,\ppp[\yyy][q])$.  Therefore, the result follows from Theorem~\ref{theorem_MAIN__GeneralCase}.  
\end{proof}
\section{Heuristic Training Procedure Used to Train the Probabilistic Transformer}
\label{a_NumericalScheme}
This section records the numerical scheme used to train our probabilistic transformer model.  This is also the same numerical scheme used in the partner paper \cite{AB_2022__PartnerPaper}.  
The procedure is a heuristic mimicking the proof technique for our main result, namely, Lemma~\ref{MAIN_LEMMA_theorem_UAT_qualitative}.  

Assume that the total number of training data points $\#\xx$ is at-least $2$.  
We highlight that, in Algorithm~\ref{Algo_train_Model} below, one must pick an integer $N$ which is strictly less than the total number of training.  This encourages the heuristic to learn how to interpolate from the training data rather than trivially assign $\{x_n\}_{n=1}^N \eqdef \xx$ in the algorithm's first step.  

\begin{algorithm}[H]
\label{Algo_train_Model}
\caption{Procedure used to train the probabilistic transformer in our experiments}
\SetAlgoLined
\KwInput{$1\leq N< \#\xx;\,N\in \nn_+$}
\KwOutput{Trained probabilistic transformer $\hat{T}$}
$x_1,\dots,x_N \leftarrow 
\underset{\tilde{x}_1,\dots,\tilde{x}_N \in \xx}{\operatorname{argmin}}
\sum_{n=1}^N \sum_{x \in \xx} \|\tilde{x}_n-x\|$
\tcp*{Get $\hat{\mu}_n$}
\For{$n\leq N$}{
$\hat{\mu}_n \eqdef \mu_{x_n}$
}
\For{$x \in \xx$,$n\leq N$
}{
\uIf{$n\in \operatorname{argmin}_{\tilde{n}\leq N} \|x-x_{\tilde{n}}\|$
\tcp*{Partition $\xx$ using $\hat{\mu}_n$}
}{
$L_{n,x}\eqdef 1$
}
\Else{
$L_{n,x}\eqdef 0$
}
}
$\hat{f}\in \operatorname{argmin}_{\tilde{f}\in \NN[\varphi,N]}
\frac{-1}{\#\xx}\sum_{x \in \xx} \sum_{n=1}^N
L_{n,x}\log\left(
\operatorname{Softmax}_N\circ \tilde{f}\circ \varphi(x)_n
\right).
$\tcp*{Train}
\textbf{Return:}$\,
\hat{F}\eqdef \sum_{n=1}^N \left(\operatorname{Softmax}_N\circ \hat{f}\circ \varphi(\cdot)\right)_n\hat{\mu}_n$. \tcp*{Return trained model.}
\end{algorithm}
\begin{remark}[Computational Tractability of Algorithm~\ref{Algo_train_Model}]\label{training_advantage}
An appealing feature of Algorithm~\ref{Algo_train_Model} is that it intentionally avoids passing any (stochastic) gradient updates through the Wasserstein distance by decoupling the identification of the $\{\mu_n\}_{n=1}^N$ from the training of the deep classifier $\operatorname{Softmax}_N\circ \hat{f}\circ \varphi$ when training any model in $\NN[\varphi,\ddd][\sigma:\star]$.  This decoupling trick is especially important from a computational standpoint since evaluating Wasserstein distances has super-cubic complexity (see \cite{pele2009fast} and \cite{cuturi2013sinkhorn}).  
The logic behind the two-step procedure of Algorithm~\ref{Algo_train_Model} is to emulate the method of proof behind Lemma~\ref{MAIN_LEMMA_theorem_UAT_qualitative}.  First, we heuristically try to cover $f(\xxx)$ by $N$ probability measures $\mu_1,\dots,\mu_N$ and the proof of Lemma~\ref{MAIN_LEMMA_theorem_UAT_qualitative}.  Then, each $x \in \xx$ is labelled with a binary vector with a non-zero entry indicating which $\hat{\mu}_n$ it is closest to. 
The key observation here in this step is that by uniform continuity, the unknown measure $f(x)$ is near $\hat{\mu}_n$, only if $x$ is near $x_n$.  
Thus, the continuity of $f$ allows us to fully reduce the classification problem in $\ppp[\rrD][1]$ to a classical classification problem in $\rrD$.  
Thus, in the last step, a classical Euclidean classifier may be trained, and we again avoid working with outputs in $\ppp[\rrD][1]$.
\end{remark}
\end{appendix}

\bibliography{References}
\end{document}